\pdfoutput=1
\documentclass[12pt]{article}

\usepackage[square,sort,comma,numbers]{natbib}

\usepackage[utf8]{inputenc} % allow utf-8 input
\usepackage[T1]{fontenc}    % use 8-bit T1 fonts
\usepackage{hyperref}       % hyperlinks
\usepackage{url}            % simple URL typesetting
\usepackage{booktabs}       % professional-quality tables
\usepackage{amsfonts}       % blackboard math symbols
\usepackage{nicefrac}       % compact symbols for 1/2, etc.
\usepackage{microtype}      % microtypography
\usepackage{authblk}
\usepackage{fullpage}

\title{Predicting What You Already Know Helps:\\ Provable Self-Supervised Learning}

\author[$\thanks{Princeton University. Email: \url{jasonlee@princeton.edu}}$]{Jason D. Lee}
\author[$\thanks{Princeton University. Email: \url{qilei@princeton.edu} }$]{\ Qi Lei}
\author[$\thanks{Princeton University. Email: \url{nsaunshi@cs.princeton.edu} }$]{\ Nikunj Saunshi}
\author[$\thanks{University of Texas at Austin. Email:  \url{jzhuo@utexas.edu}}$]{\ Jiacheng Zhuo}

\affil[ ]{}%{$^\dagger$ UT Austin \ \

\date{\today}
\usepackage{amsmath,amssymb,amsthm}
\usepackage{subfigure}
\usepackage{wrapfig}
\newcommand{\Rr}{\mathbb{R}}
\newcommand{\E}{\mathbb{E}}
\newcommand{\Eqref}[1]{Equation~(\ref{#1})}
\newcommand{\brho}{{\bar{\rho}}}
\newcommand{\linspan}{\text{sp}}

\usepackage{amsmath,amssymb}
\usepackage{enumitem}
\usepackage{comment} 
\usepackage{hyperref}
\usepackage{natbib}
\usepackage{bm}
\usepackage{graphicx} 
\usepackage{thmtools}
\usepackage{thm-restate}

\usepackage{times}

\usepackage{amssymb}
\usepackage{amsmath}
\usepackage{amsfonts}

\usepackage[utf8]{inputenc} % allow utf-8 input
\usepackage[T1]{fontenc}    % use 8-bit T1 fonts
\usepackage{url}            % simple URL typesetting
\usepackage{booktabs}       % professional-quality tables
\usepackage{amsfonts}       % blackboard math symbols
\usepackage{nicefrac}       % compact symbols for 1/2, etc.
\usepackage{microtype}      % microtypography
  \usepackage{mathtools}

\usepackage{tikz}
\usepackage{pgfplots}
\usetikzlibrary{pgfplots.groupplots}

\usepackage{caption}
\usepackage[bottom,hang,flushmargin]{footmisc}

\newtheorem{theorem}{Theorem}[section]

\newtheorem{claim}[theorem]{Claim}
\newtheorem{assumption}{Assumption}[section]

\newtheorem{lemma}[theorem]{Lemma}
\newtheorem{corollary}[theorem]{Corollary}
\newtheorem{proposition}[theorem]{Proposition}
\newtheorem{definition}[theorem]{Definition}

\newtheorem{remark}{Remark}[section]
\newtheorem{example}{Example}[section]

\newcommand{\cov}{\mathrm{cov}}

\newcommand{\rank}{\mathrm{rank}}

\newcommand{\ER}{\mathrm{ER}}

\newcommand{\id}{\mathrm{id}}

\def\R{\mathbb{R}}

\def\cA{\mathcal{A}}
\def\cB{\mathcal{B}}
\def\cC{\mathcal{C}}

\def\cE{\mathcal{E}}
\def\cF{\mathcal{F}}

\def\cH{\mathcal{H}}

\def\cS{\mathcal{S}}

\def\cL{\mathcal{L}}

\def\cN{\mathcal{N}}
\def\cO{\mathcal{O}}

\def\cS{\mathcal{S}}
\def\cT{\mathcal{T}}

\def\cZ{\mathcal{Z}}
\def\cX{\mathcal{X}}
\def\cY{\mathcal{Y}}

\def\bP{\mathbf{P}}

\def\vw{\mathbf{w}}

\newcommand{\prob}{\mathbb{P}}

\newcommand{\apx}{\text{apx}}

\newtheorem{fact}{Fact}[section]

\def\approxcorrect{\cmark\kern-1.4ex\raisebox{.30ex}{$\xmark$}}

\newcommand{\idxn}[1][]{\ifthenelse{\equal{#1}{}}{\mathbb{INDQ}_n}{\mathbb{INDQ}_{#1}}}

\newcommand{\beq}{\begin{equation}}

\newcommand{\eeq}{\end{equation}}

\newcommand{\A}{{\mtx{A}}}

\newcommand{\B}{{{\mtx{B}}}}

\newcommand{\tCI}{\text{CI}}
\newcommand{\tHS}{\text{HS}}

\newcommand{\op}{{\text{op}}}

\newcommand{\bmu}{{\boldsymbol{{\mu}}}}

\newcommand{\bx}{{\bm{x}}}
\newcommand{\bX}{{\bm{X}}}

\newcommand{\bw}{{\bm{w}}}
\newcommand{\bL}{{\bm{L}}}
\newcommand{\bE}{{\bm{E}}}
\newcommand{\bW}{{\bm{W}}}
\newcommand{\bI}{{\bm{I}}}
\newcommand{\bM}{{\bm{M}}}
\newcommand{\bC}{{\bm{C}}}
\newcommand{\bV}{{\bm{V}}}
\newcommand{\bN}{{\bm{N}}}

\newcommand{\ba}{{\bm{a}}}
\newcommand{\bu}{{\bm{u}}}
\newcommand{\bv}{{\bm{v}}}
\newcommand{\bz}{{\bm{z}}}
\newcommand{\bZ}{{\bm{Z}}}
\newcommand{\be}{{\bm{e}}}
\newcommand{\by}{{\bm{y}}}
\newcommand{\bY}{{\bm{Y}}}

\newcommand{\bSigma}{{\bm{\Sigma} }}

\newcommand{\bb}{\bm{b}}
\newcommand{\bB}{\bm{B}}
\newcommand{\bA}{\bm{A}}
\newcommand{\bU}{\bm{U}}
\newcommand{\bS}{\bm{S}}

\definecolor{emmanuel}{RGB}{255,127,0}

\newcommand{\Var}{\textrm{Var}}
\newcommand{\mtx}[1]{\bm{#1}}

\newcommand{\X}{{\mtx{X}}}

\newcommand{\longrnote}[1]{ &  &  \\   &  &  &  &  & \notag \text{\footnotesize\llap{#1}}}

\newcommand{\Trace}{{\mathrm{Tr}}}

\DeclareMathOperator*{\argmax}{arg\,max}
\DeclareMathOperator*{\argmin}{arg\,min}

\newcommand{\pre}{\text{pre}}
\newcommand{\HS}{\text{HS}}
\newcommand{\down}{\text{down}}

\newcommand{\jnote}[1]{\textcolor{purple}{[J: #1]}}

\usepackage{hyperref}
\hypersetup{colorlinks=true,citecolor=blue,linkcolor=blue}
\usepackage[parfill]{parskip}

\ifdefined\usebigfont

\usepackage{times}
\usepackage[fontsize=13pt]{scrextend}
\usepackage[left=1.56in,right=1.56in,top=1.71in,bottom=1.77in]{geometry}
\usepackage{enumitem}
\setlist[itemize]{leftmargin=*}
\setlist[enumerate]{leftmargin=*}
\else
\usepackage{times}
\fi
\begin{document}

\maketitle

% !TEX root = colt2021-SSL.tex

\begin{abstract}

\iftrue
Self-supervised representation learning solves auxiliary prediction tasks (known as pretext tasks) without requiring labeled data to learn useful semantic representations. These pretext tasks are created solely using the input features, such as predicting a missing image patch, recovering the color channels of an image from context, or predicting missing words in text; yet predicting this \textit{known} information helps in learning representations effective for downstream prediction tasks. \\
We posit a mechanism exploiting the statistical connections between certain {\em reconstruction-based} pretext tasks that guarantee to learn a good representation. 
Formally, we quantify how the approximate independence between the components of the pretext task (conditional on the label and latent variables) allows us to learn representations that can solve the downstream task by just training a linear layer on top of the learned representation. We prove the linear layer yields small approximation error even for complex ground truth function class and will drastically reduce labeled sample complexity. Next, we show a simple modification of our method leads to nonlinear CCA, analogous to the popular SimSiam algorithm, and show similar guarantees for nonlinear CCA.

%We also establish its connection to other SSL algorithms that enforce similarity between two augmented views, or namely contrastive learning.
\fi 

%Formally, we quantify how approximate independence between the pretext tasks and features (conditional on the label) allows for drastically reduced sample complexity when learning a linear layer on top of the representation from self-supervised learning.
%\jnote{add some equations in abstract to get competent reviewers}

%Pretext tasks are typically created by removing or corrupting part of the input (i.e. predicting an image patch or the next word from the context). I

%Self-supervised learning can learn to teach itself supervised tasks from pseudo-labels. This paper explains its reasoning behind: why predicting the information we already know helps downstream tasks. For a class of reconstruction-based pretext tasks like context-encoder, jigsaw and image colorization, we pose a mechanism to explain the phenomenon through conditional independence between input and pretext task given the downstream labels. We theoretically exploit different level of conditional independence when downstream tasks can be well approximated by a linear layer on top of the representation learned through self-supervised learning. We exhaustively analyze the sample complexity improvement through the procedure.
\end{abstract}

% !TEX root = colt2021-SSL.tex

\section{Introduction} 
%importance
\label{sec:intro}
Self-supervised learning revitalizes machine learning models in computer vision, NLP, and control problems (see reference therein \citep{jing2020self,kolesnikov2019revisiting,devlin2018bert,wang2015unsupervised,jang2018grasp2vec}). Training a model with auxiliary tasks based only on input features reduces the extensive costs of data collection and semantic annotations for downstream tasks. 
It is also known to improve the adversarial robustness of models~\citep{hendrycks2019using,carmon2019unlabeled,chen2020adversarial}.
%e.g. DeepMind's self-supervised network, CPC, surpassed AlexNet's performance on ImageNet.
%how it works 
Self-supervised learning creates pseudo labels solely based on input features, and solves auxiliary prediction tasks (or pretext tasks) in a supervised manner.  
However, the underlying principles of self-supervised learning are mysterious since it is a-priori unclear why predicting what we already know should help. We thus raise the following question:
\begin{center} {\em What conceptual connection between pretext and downstream tasks ensures good representations? What is a good way to quantify this?}
\end{center}
As a thought experiment, consider a simple downstream task of classifying desert, forest, and sea images.
A meaningful pretext task is to predict the background color of images (known as image colorization \citep{zhang2016colorful}).
Denote $X_1,X_2,Y$ to be the input image, color channel, and the downstream label respectively.
Given knowledge of the label $Y$, one can possibly predict the background $X_2$ without knowing much about $X_1$.
In other words, $X_2$ is approximately independent of $X_1$ conditional on the label $Y$. 
Consider another task of inpainting \citep{pathak2016context} the front of a building ($X_2$) from the rest ($X_1$).
While knowing the label ``building'' ($Y$) is not sufficient for successful inpainting, adding additional latent variables $Z$ such as architectural style, location, window positions, etc. will ensure that variation in $X_2$ given $Y, Z$ is small.
We can mathematically interpret this as $X_1$ being approximate conditionally independent of $X_2$ given $Y,Z$. 

The main insight that we exploit in this work is that with approximate conditional independence (as in the above examples), a method that predicts $X_2$ from $X_1$ will inadvertently implicitly encode and learn to predict $Y$ (and $Z$) from $X_1$ as an intermediate step, and then predict $X_2$ from $Y$\footnote{This is formally demonstrated in the proof sketch of Lemma \ref{lemma:discrete_case_CI}.}. 
Building upon this insight, we make the following contributions.
%Even without labeled data, the information of $Y$ is hidden in the prediction for $X_2$.

%However, since the perplexity in the representation, as a deterministic function, is much smaller than the original image, we will expect a much fewer sample required to learn downstream tasks from it \citep{kirsch2020unpacking}. \jnote{kind of unclear.}

\paragraph{Contributions.}  
The goal of this paper, as in statistical learning theory, is to investigate the {\em statistical connections} between the random variables of input features (in this paper $(X_1,X_2)$) and downstream labels $Y$, and show how specific connections can guarantee a successful learning procedure.
For self-supervised learning (SSL), success is measured using the following 2 notions, 1) expressivity, i.e. does the learned representation from SSL have the ability to express the ground truth prediction function for labels $Y$, and 2) sample complexity, i.e. can it do so with way fewer labeled samples than what would be required without SSL. 

In this work, we establish theoretical analysis for self-supervised learning fulfilling these goals.

\begin{itemize} 
\item We provide generalization guarantees for a class of self-supervised algorithms under a statistical assumption of  {\em approximate conditional independence (ACI)}. Specifically, we show
\begin{itemize}
	\item {\em small representation error:} the learned representation can almost linearly separate downstream targets, and
	\item {\em small estimation error:} learning the predictor for downstream tasks only require very few number of samples. 
\end{itemize} 

\item Our analysis focused on {\em reconstruction-based} SSL methods (\cite{zhang2016colorful, pathak2016context,devlin2018bert,grill2020bootstrap}) is presented in sections \ref{sec:CI} and \ref{sec:beyondCI}.
In Section~\ref{sec:topic_model}, we instantiate the bound from the analysis in the topic modeling framework, a standard generative model for text \citep{papadimitriou2000latent,hofmann1999probabilistic}, where $X_1$ and $X_2$ are chosen to be two halves of a text document.
Although data can be sampled from a potentially infinite mixtures of $k$ underlying topics, an appropriate ACI assumption can be shown that leads to a downstream sample complexity of $\cO(k)$.

\item We also build the connection and extend the analysis to a variant of the SimSiam \citep{chen2021exploring} method, a non-linear canonical correlation analysis (CCA) method for self-supervised learning in Section \ref{sec:learn_joint_distribution}.
Further connecting this to alternating conditional expectation (ACE) algorithm \citep{breiman1985estimating}, we show how this problem is related to decomposing the conditional distribution $X_2 \mid X_1$.

\item We quantify our notion of ACI by a certain partial covariance matrix (Definition~\ref{def-approx-CI}) and our risk bound scales linear with it. We verify this and other aspects of our main generalization bound (Theorem \ref{thm:main_result_approximate_CI}) using simulation experiments in Section~\ref{sec:experiments}.
We also find that pretext task experimentally helps when CI is approximately enforced in text domain. We further demonstrate on a real-world image dataset that a pretext task-based linear model performs at least as well as many baselines.
\end{itemize}

% !TEX root = colt2021-SSL.tex
\subsection{Related work}
\label{related-work}

\paragraph{Self-supervised learning (SSL) methods in practice:}  There has been a flurry of self-supervised methods lately. 
One class of methods reconstruct images from corrupted or incomplete versions of it, like denoising auto-encoders \citep{vincent2008extracting}, image inpainting \citep{pathak2016context}, and split-brain autoencoder \citep{zhang2017split}.
Pretext tasks are also created using visual common sense, including predicting rotation angle \citep{gidaris2018unsupervised}, relative patch position \citep{doersch2015unsupervised}, recovering color channels \citep{zhang2016colorful}, solving jigsaw puzzle games \citep{noroozi2016unsupervised}, and discriminating images created from distortion \citep{dosovitskiy2015discriminative}. We refer to the above procedures as reconstruction-based SSL. 
Another popular paradigm is contrastive learning \citep{chen2020simple,chen2020big}. The idea is to learn representations that bring similar data points closer while pushing randomly selected points further away \citep{wang2015unsupervised,logeswaran2018efficient,arora2019theoretical} or to maximize a contrastive-based mutual information lower bound between different views \citep{hjelm2018learning,oord2018representation,tian2019contrastive}.
A popular approach for text domain is based on language modeling where models like BERT and GPT create auxiliary tasks for next word predictions \citep{devlin2018bert,radford2018improving}. The natural ordering or topology of data is also exploited in video-based \citep{wei2018learning,misra2016shuffle,fernando2017self}, graph-based \citep{yang2020self,hu2019strategies} or map-based \citep{zhang2019path} SSL. For instance, the pretext task is to determine the correct temporal order for video frames as in \citep{misra2016shuffle}.

\paragraph{Theory for SSL:} 
While we theoretically study reconstruction-based SSL, prior work has different flavors of theoretical results for different kinds of SSL methods.
Most relevant are the guarantees for representation learning using SSL methods on downstream tasks that just learn a linear classifier on top of the learned representations.
\cite{arora2019theoretical} shows guarantees for representations from a contrastive learning objective: $L^{cont}_1(\psi) = \mathbb{E}_{(X_1,X_2), X'_2}[\log(1 + e^{-\psi(X_1)^\top \psi(X_2) + \psi(X_1)^\top \psi(X'_2)})]$.
Under a class conditional independence assumption, i.e. $X_1 \perp X_2 \mid Y$, they show that representation $\psi$ that does well on contrastive objective, i.e. $L^{cont}_1(\psi)\le\epsilon$, will have $\mathcal{O}(\epsilon)$ linear classification loss on the average binary task involving pairs of classes $(y_1,y_2)$.
However, their analysis cannot handle the general case of approximate conditional independence.
Recently, Tosh {\textit {et al.}} \citep{tosh2020contrastive} show that contrastive learning representations can {\em linearly} recover continuous functions of the underlying topic posterior under a topic modeling assumption for text.
While their assumption bears similarity to ours, the assumption of independent sampling of words is strong and does not generalizable to other domains like images.
Most relevant is a concurrent work \citep{tosh2020contrastive_1} that shows guarantees for a contrastive learning objective that looks like $L^{cont}_{2}(\psi,\eta) = \mathbb{E}_{(X_1,X_2), X'_2}\left[\log(1 + e^{-\psi(X_1)^\top \eta(X_2)}) + \log(1 + e^{\psi(X_1)^\top \eta(X'_2)})\right]$, with a multi-view redundancy assumptions that is very similar to our ACI assumption.
We take a closer look at their assumption in Section~\ref{apx:tosh}.
All the above objectives are different from the simple reconstruction-based objective we consider: $L(\psi) = \mathbb{E}_{(X_1,X_2)}\left[\|X_2 - \psi(X_1)\|^2\right]$.
Saunshi {\em et al.} \cite{saunshi2020mathematical} show guarantees for representations learned using language modeling on sentence classification tasks.
Some more recent work \citep{tsai2020demystifying,mitrovic2020representation,tian2020understanding,wang2020understanding} provide theoretical understanding on SSL respectively based on causality, mutual information, gradient-descent dynamics, and alignment/uniformity of representations, without explicit risk bounds for downstream tasks.
There is a mutual information maximization view of contrastive learning, but \cite{tschannen2019mutual} points out issues with it.
Previous attempts to explain negative sampling \citep{mikolov2013distributed} based methods use the theory of noise contrastive estimation \citep{gutmann2010noise,ma2018noise} to show asymptotic guarantees, without explicit connections to downstream tasks.
CI is also used in sufficient dimension reduction \citep{fukumizu2009kernel,fukumizu2004dimensionality}, while CI and redundancy assumptions on multiple views \citep{kakade2007multi,ando2007two} are used to analyze a canonical-correlation based dimension reduction algorithm and also for self-supervised learning algorithms like co-training \citep{blum1998combining}. Finally, \cite{alain2014regularized,vincent2011connection} provide a theoretical analysis for denoising auto-encoder.

%CURL, Tosh topic posterior, noise contrastive estimation (Hyvarynin, Collins), co-training, contrastive geometry on sphere.
%
% \citep{arora2019theoretical,tosh2020contrastive} 

%
%\revise{Add overview of results.}

% !TEX root = SSL_ICLR.tex
\subsection{Overview of results:}
Section~\ref{sec:prelim} introduces notation, setup, and the self-supervised learning procedure considered in this work.
In Section~\ref{sec:CI}, we analyze downstream sample complexity under exact CI and unlimited labeled data to highlight the key ideas. % enables our pretext task solution to help with downstream task performance. After a warm up with gaussian variables, we analyze the case of general random variables with two choices of representation function class, namely, an arbitrarily power one \revise{better term?} and linear representation.
Section~\ref{sec:beyondCI} presents our main result with relaxed conditions: under ACI with latent variables, and assuming finite samples in both pretext and downstream tasks, for various function classes, and both regression and classification tasks. Section~\ref{sec:topic_model} demonstrates our results with an example in the setting of topic modeling.  %then relaxes exact conditional independence to an approximate one \revise{based on HS norm? Qi please rephrase this accordingly} and also unifies the analysis for the settings described in the previous section. The analysis also looks at the case where the pretext task is solved with only finite samples.
 In Section~\ref{sec:learn_joint_distribution} we extend our results to self-supervised tasks that enforce two views of data to have similar representations, or namely SimSiam~\cite{chen2021exploring}. 
  Experiments verifying our theoretical findings are in Section~\ref{sec:experiments}.
 Proofs of most results are in the Appendix.
%	We first describe how oracle access to the optimal language model can solve downstream classification tasks.
%	Following this we describe how one can make this argument robust, i.e. how having access to just an $\epsilon$ approximate (in terms of cross entropy loss) to this language model can affect downstream task performance.
%	\item \Secref{sec:embeddings} goes on to analyze the more practical case how learning a softmax-parameterized language model can learn embeddings that do well on downstream linear classification tasks.
%	This setting introduces additional challenges and \revise{we show a more moderate error propagation than the analysis in the previous section, by exploiting some nice properties about the word embeddings learned by language models.}
%%	\item In \Secref{sec:exps} we present some experimental results, mostly those that verify some of the assumptions we make in our theoretical analysis
%\end{itemize}

% !TEX root = colt2021-SSL.tex

\section{Preliminary}\label{sec:prelim}
\subsection{Notation}
We use lower case symbols ($x$) to denote scalar quantities, bold lower case symbols ($\bx$) for vector values, capital letters ($X$) for random variables, and capital and bold letters $\bX$ for matrices. $P_X$ denotes the probability law of random variable $X$, and the space of square-integrable functions with probability $P$ is denoted by $L^2(P)$.
We use standard $\cO$ notation to hide universal factors and $\tilde \cO$ to hide log factors. $\|\cdot\|$ stands for $\ell_2$-norm for vectors or Frobenius norm for matrices. 
%\jz{Do we want to introduce 'best predictor' as $\E[|Y-f(X)|^2]$ and connect it with $\E[Y|X]$, so that the reader know why we are interested in $\E[Y|X]$?}

\textbf{Linear conditional expectation.} $\E^L[Y|X]$ denotes the prediction of $Y$ with linear regression:
\begin{align*}
\E^L[Y|X=\bx] := \bW^*\bx + \bb^*, \ \   & \text{ where }\bW^*,\bb^* :=  \arg\min_{\bW,\bb} \E [\|Y-\bW X-\bb\|^2].
\end{align*}
In other words, $\E^L[Y|X]$ denotes the best linear predictor of $Y$ given $X$. We also note that $\E[Y|X]\equiv \argmin_f\E[\|Y-f(X)\|^2]$ is the best predictor of $Y$ given $X$.  %Under some circumstances such as when $Y$ is binary, we note that $\E^L[X|Y] = \E[X|Y]$, is the conditional expectation, i.e., the best predictor of the value $X$ as an arbitrary function of $Y$.  

\textbf{(Partial) covariance matrix.} For random variables $X,Y$, we denote $\bSigma_{XY}$ to be covariance matrix of $X$ and $Y$. For simplicity in most cases, we assume $\E[X]=0$ and $\E[Y]=0$; thus we do not distinguish $\E[XY]$ and $\bSigma_{XY}$. The partial covariance matrix between $X$ and $Y$ given $Z$ is:
\begin{align}
\bSigma_{XY| Z}  := & \cov \{X-\E^L[X|Z], Y-\E^L[Y|Z] \} \equiv  \bSigma_{XY}-\bSigma_{XZ}\bSigma_{ZZ}^{-1}\bSigma_{ZY}\label{eqn:partial_cov},
\end{align}
which captures the correlation between $X$ and $Y$ setting aside the effect of $Z$. 

%Notice $ \bSigma_{XX\cdot Y}^{-1} = (\bSigma^{-1})_{XX} $, i.e., $(\bSigma^{-1})_{XX} = \E[(X-\E^L[X|Y])(X-\E^L[X|Y])^\top ]^{-1} $.  
%Conditional independence $X\bot Y|Z$ ensures $\bSigma_{XY| Z}=0$. 

\begin{comment} 
\paragraph{Reproducing Kernel Hilbert Space (RKHS)} We consider the RKHS $\cH_1,\cH_2$ for random variables $X_1\in \R^{d_1}$ and $X_2\in\R^{d_2}$ and the corresponding feature map $\phi_1:\R^{d_1}\mapsto \cF_1$, $\phi_2:\R^{d_2}\mapsto \cF_2$ that satisfies $\langle \phi_1(\bx_1),\phi_1(\bx_1')\rangle_{\cF_1} = k_1(\bx_1,\bx_1')$, and $\langle \phi_2(\bx_2),\phi_2(\bx_2')\rangle_{\cF_2} = k_2(\bx_2,\bx_2') $ for any $\bx_1,\bx_1'\in \R^{d_1}$ and $\bx_2,\bx_2'\in \R^{d_2}$. Here $k_i(\bx_i,\bx_i')$ is the kernel operator associated with the RKHS $\cH_i$ such that $ f(\bx_i) = \langle f,k(\bx_i,\cdot)\rangle_{\cH},\forall f \in \cH_i, \bx_i\in \R^{d_i}$. 
\end{comment} 

\textbf{Sub-gaussian random vectors.}
	$X\in \R^d$
	is $\rho^2$-sub-gaussian if for every fixed unit vector $\bv\in\R^d$, the variable $\bv^\top X$ is $\rho^2$-sub-gaussian, i.e., $\E[e^{s\cdot\bv^\top(X-\E[X])}]\le e^{s^2\rho^2/2}$ ($\forall s\in\R$).
	%Its $\psi_2$-norm $\|X\|_{\psi_2}=\sigma$ if
	%$$ \E[ \exp((\bv^\top X)^2/\sigma^2)] \leq 2, \forall \bv \in \SS^{d-1},  $$
	%where $\SS^{d-1} = \{\bx \in \R^d:
	%\|\bx\| = 1\}$ is the $d -1$ unit sphere. %Throughout the paper we assume the sub-gaussian property on the residual term in the prediction for the probabilistic tail bounds. 

\subsection{Setup and methodology}\label{subsec:setup}
We denote by $X_1$ the input variable, $X_2$ the target random variable for the pretext task, and $Y$ the label for the downstream task, with $X_1\in \cX_1 \subset \R^{d_1}, X_2\in \cX_2\subset \R^{d_2}$ and $Y\in \cY \subset \R^k$.
If $\cY$ is finite with $|\cY|=k$, we assume $\cY\subset\R^k$ is the one-hot encoding of the labels.
$P_{X_1X_2Y}$ denotes the joint distribution over $\cX_1\times \cX_2\times \cY$. $P_{X_1Y},P_{X_1}$ denote the corresponding marginal distributions.  
%\revise{Underlying distribution over $\cX_1\times\cX_2\times\cY$ that's common to pretext and downstream tasks with some common notation}.
Our proposed self-supervised learning aims to fulfill the following two steps:

\emph{Step 1 (pretext task): } Learn a representation $\psi(\bx_1)$ close to $\psi^*\coloneqq \argmin_{g\in \cH}\E\|X_2-g(X_1)\|^2, $ where $\cH$ can vary for different settings that we will specify and discuss later.

\emph{Step 2 (downstream task): } Perform linear regression on $Y$ with $\psi(X_1)$, i.e.
	$f(\bx_1) :=  (\bW^*)^\top\psi(\bx_1), $ where $\bW^*\leftarrow \argmin_{\bW}\E_{X_1,Y} [\|Y-\bW^\top\psi(X_1)\|^2] $. Namely we learn $f(\cdot) = \E^L[Y|\psi(\cdot)]$. 

We study this simplified version in the main text, where in practice, the SSL procedure may utilize an encoder-decoder structure, while the downstream task uses both $X_1$ and $X_2$ to predict $Y$. We incorporate these extensions in Appendix \ref{sec:auto-encoder} and \ref{sec:y_given_x1_x2}. 
	
%We analyze this simplified version. In practice like denoising auto-encoder, the input is only on $X_1$

%However, we can also bridge the gap to practice usage. First, for downstream task, one could use an encoder function instead of all the predictor; 2) use only X1 to predict Y will be close to using both X1 and X2 to predict Y

% we don't care about the sample complexity for pretext task.  

With finite samples, performance of a learned representation $\psi$ on the downstream task depends on the following quantities that capture expressivity and sample complexity respectively:

%\revise{$e_{\apx}(\psi)$? Explicitly mention $g(X_1)=W\psi(X_1)$?} 
\textbf{Approximation error} indicates whether $Y$ is {\em linearly separable} by the learned representation $\psi$, thus measuring expressivity. We measure this by comparing $\bW\psi(X_1)$ to the optimal predictor $f^*:=\E[Y|X_1=\bx_1]$. %for a learned representation $\psi$ by learning a linear function on top of it for the downstream task.
Denote $e_{\apx}(\psi) = \min_{\bW}\E[\|f^*(X_1) - \bW\psi(X_1)  \|^2]$. 
This gives a measure of how well $\psi$ can linearly predict $Y$ when given infinite samples for the task.

\textbf{Estimation error} measure sample complexity of $\psi$ on the downstream task and assume access to $n_2$ i.i.d. samples $(\bx_1^{(1)}, \by^{(1)}), \cdots,(\bx_1^{(n_2)}, \by^{(n_2)}) $ drawn from $P_{X_1Y}$. 
We express the $n_2$ samples collectively as $\bX_1^{\down} \in \R^{n_2\times d_1}$, $\bY\in \R^{n_2\times k}$ and overload notation to say $\psi(\bX_1^{\down})=\left[\psi(\bx_1^{(1)})|\psi(\bx_1^{(2)})\cdots |\psi(\bx_1^{(n_2)})\right]^\top \in \R^{n_2\times d_2}$.  %, applied row-wise on each sample.
We perform linear regression on the learned representation $\psi$ and measure excess risk, that incorporates both approximation and estimation errors.
\begin{align*}
	\hat \bW \leftarrow \argmin_{\bW} \frac{1}{2n_2}\|\bY-\psi(\bX_1)\bW\|_F^2; ~~\ER_{\psi}(\hat \bW) \coloneqq \frac{1}{2}\E \|f^*(X_1) - \hat\bW^\top \psi(X_1) \|_2^2. \label{eqn:W_hat}
\end{align*}
%\begin{align*}
%	\hat \bW \leftarrow& \argmin_{\bW} \frac{1}{2n_2}\|\bY-\psi(\bX_1)\bW\|_F^2\\
%	\ER_{\psi^*}(\hat \bW) &\coloneqq \E \|f^*(X_1) - \hat\bW^\top \psi^*(X_1) \|_2^2.
%\end{align*}

%\input{CI} 

% !TEX root = colt2021-SSL.tex

\section{Guaranteed recovery with conditional independence}\label{sec:CI}

%\revise{Section and subsection titles should only have first letter capitalized according to instructions.}
%We are interested in variable $(x_1,x_2,\by)\in\Rr^{d_1+d_2+k}$, where $x_1\in\R^{d_1}, x_2\in\Rr^{d_2}$ and $y\in\Rr^k$. For simplicity, we assume $X_1,X_2$ and $Y$ to be zero mean random variables.  % and assume that $(x_1,x_2,\by)$ is jointly gaussian with mean 0.
%Let $\bSigma\in\Rr^{(d_1+d_2+k) \times (d_1+d_2+k)}$ denote the covariance of the joint distribution $(X_1,X_2,Y)$. Let $\bSigma_{X_1X_2}$ be the covariance matrix between $X_1$ and $X_2$ and vice versa for all pairs of random variables among them.

%We first consider the setting where we only use linear functions for predictions. Note when $X_1,X_2,Y$ are jointly Gaussian, their conditional expectation is simply linear function.

In this section, we focus on the case where the input $X_1$ and pretext target $X_2$ are conditionally independent (CI) given the downstream label $Y$. While this is a strong assumption that is rarely satisfied in practice, it helps us understand the role of CI with clean results and builds up to our main results with ACI with latent variables in Section \ref{sec:beyondCI}.
As a warm-up, we show how CI helps when $(X_1,X_2,Y)$ are jointly Gaussian to give us a flavor for the results to follow in Appendix \ref{sec:joint_gaussian}.
We then analyze it for general random variables under two settings: (a) when the function class used for $\psi$ is universal, (b) when $\psi$ is restricted to be a linear function of given features.
%The two cases will be eventually unified in Section~\ref{sec:beyondCI}.
For now we assume access to a large amount of unlabeled data so as to learn the optimal $\psi^*$ perfectly and this will be relaxed later in Section~\ref{sec:beyondCI}.
The general recipe for the results is as follows:

1. Find a closed-form expression for the optimal solution $\psi^*$ for the pretext task.\\
2. Use conditional independence to show that optimal $f^*$ is linear in $\psi^*$, i.e., $e_{\apx}(\psi^*)$ is small.\\
3. Exploit the low rank structure of $\psi^*$ to show small estimation error on downstream tasks.
	%\item Argue that sub-optimal $\psi$ learned with finite samples in pretext task yields good performance

%Note that the analysis will assume that we could perfectly learn the optimal $\psi^*$ %while ignoring that we only have finite samples during self-supervised learning, 
%since we could have arbitrary amount of unlabeled data.
%We will analyze the case of learning a sub-optimal $\psi$ later.
%\revise{Discussion about general recipe, i.e. find an expression for the optimal $\psi^*$ for pretext task, use conditional independence to upper bound $e_{\apx}(\psi^*)$, analyze sample complexity on downstream task. As a warm up, we instantiate this in the case of gaussian data to highlight the important steps in the analysis and to quantify the benefits of this self-supervised learning. We then show how conditional independence is useful in the more general case.}

\paragraph{Data assumption.} Suppose $Y=f^*(X_1)+N$, where $f^*=\E[Y|X_1]$ and $\E[N]=0$. We assume $N$ is $\sigma^2$-subgaussian. %\jz{what is $\psi_2$ here?} 
%of $\sigma$.
%For downstream task we have access to $n_2$ i.i.d. samples $(\bx_1^{(1)}, \by^{(1)}), \cdots,(\bx_1^{(n_2)}, \by^{(n_2)}) $ drawn from joint distribution with density $P_{X_1Y}$. For convenience, we express these $n_2$ samples collectively as an input matrix $\bX_1^{\down} \in \R^{n_2\times d_1}$ and an output matrix $\bY\in \R^{n_2\times k}$.
%We overload the representation function $\psi(\bX_1^{\down})=[\psi(\bx_1^{(1)})|\psi(\bx_1^{(2)})\cdots |\psi(\bx_1^{(n_2)})]^\top \in \R^{n_2\times d_2}$ that is applied row-wisely on each sample.    
For simplicity, we assume non-degeneracy: $\bSigma_{X_iX_i}$, $\bSigma_{YY}$ are full rank. 

\begin{assumption}
	\label{assump:independence}
	Let $X_1\in \R^{d_1},X_2\in \R^{d_2}$ be random variables from some unknown distribution. Let label $Y\in \cY$ be a discrete random variable with $k=|\cY|<d_2$. We assume conditional independence: $X_1\bot X_2|Y$. 
	%\revise{Use different notation than $c$? $c$ looks more like a constant and might be confusing. Can use $k$ instead since $\R^k$ does not seemed to be used anywhere.}
\end{assumption}
Here $Y$ can be interpreted as the multi-class labels where $k$ is the number of classes. For regression problems, one can think about $Y$ as the discretized values of continuous labels. We do not specify the dimension for $Y$ since $Y$ could be arbitrarily encoded but the results only depend on $k$ and the variance of $Y$ (conditional on the input $X_1$).

\subsection{Universal function class.}
Suppose we learn the optimal $\psi^*$ among all measurable functions
%\footnote{In the next section, we include the error of not learning $\psi^*$ exactly.}.
The optimal function $\psi^*$ in this case is naturally given by conditional expectation: $\psi^*(\bx_1) = \E[X_2|X_1=\bx_1]$.
%\end{claim}
We show that CI implies that $\psi^*$ is good for downstream tasks, which is not apriori clear.
\begin{lemma}[Approximation error]
	\label{lemma:discrete_case_CI}
	If random variables $X_1,X_2,Y$ satisfy Assumption \ref{assump:independence}, and $\bA \in \R^{\cY\times d_2} $ with $\bA_{y,:} := \E[X_2|Y=\by]$ has rank $k=|\cY|$. %\jz{mention what is $\bA$ here?}
	Then $f^*\equiv \bW^* \psi^*$, i.e., $e_\apx(\psi^*) = 0$. 
\end{lemma}	
This tells us that although $f^*$ could be nonlinear in $\bx_1$, it is guaranteed to be linear in $\psi^*(\bx_1)$.
\begin{proof}[Proof Sketch of Lemma \ref{lemma:discrete_case_CI}]
Lemma is proved by law of total expectation:
\begin{align}
\psi^*(\cdot):= \E[X_2|X_1 ] = & \E[\E[X_2|X_1,Y]|X_1] = \E[\E[X_2|Y]|X_1]\tag{uses CI}\\
\notag 
= & \sum_{y}  P(Y=y|X_1)\E[X_2|Y=y] =:  f(X_1)^\top \bA,
\end{align}
where $f(x_1)_y = P(Y=y|X_1=x_1)$, and $\bA \in \R^{\cY\times d_2} $ satisfies $\bA_{y,:} = \E[X_2|Y=y]$. %Here $\Delta_d$ denotes simplex of dimension $d$, which represents the discrete probability density over support of size $d$. 
One could see that through predicting $X_2$, due to the CI assumption, $\psi^*$ has implicitly encoded the information of $Y|X_1$. Finally due to the fact that matrix $\bA$ is full rank, we get that $f^*$ is linear in $\psi^*$ as well. 
%We leave the full proof to the Appendix. 
\end{proof} 
We see that besides CI, another important property is $\E[X_2|Y]$ being rank $k$. This means $X_2$ is correlated with every instance of $Y$, and thus captures information of every prediction class. This is naturally a necessary assumption for $X_2$ to be a reasonable pretext task for predicting $Y$.
Note that this assumption does not trivialize the problem and that even though $\psi$ is designed to predict $X_2$, it can still be a better representation than $X_2$ for downstream tasks. 
% as there could be a lot of noise $\Var[X_2|Y=y]$ that would make predicting $Y$ directly from $X_2$ almost impossible. Our result shows that through CI, $\psi^*$ can 
%We note that even though $\psi$ is designed to predict $X_2$, it can still be a better representation than $X_2$ for downstream tasks.
Note that $Y$ does not have to be linear in $X_2$ but is proven to be linear in $\psi$, since $\psi$ learns to ignore some information in $X_2$ that is irrelevant to $Y$.
We provide this simple example for better understanding:
\begin{example}
	\label{example:mixture_gaussian} 
	Let $Y\in \{-1,1\}$ be binary labels, and $X_1,X_2$ be $2-$mixture Gaussian random variables with $X_1\sim \cN(Y\bmu_1, \mathbf{I}), X_2\sim \cN(Y\bmu_2,\mathbf{I})$.
In this example, $X_1\bot X_2|Y$. Although $\E[Y|X_2]$ and $\E[Y|X_1]$ are not linear, $\E[Y|\psi]$ is linear:
$\psi(\bx_1)=P(Y=1|X_1=\bx_1) \bmu_{2} - P(Y=-1|X_1=\bx_1) \bmu_{2}$ and $f^*(\bx_1)=P(Y=1|X_1=\bx_1)-P(Y=-1|X_1=\bx_1) \equiv \bmu_2^T\psi(\bx_1) /\|\bmu_2\|^2$.	
\end{example}

%
%\begin{remark}
%	This lemma suggests that with self-supervised representation learning procedure $\psi^*(X_1)=\E[X_2|X_1]$, the downstream task prediction of $Y$ is a linear function of $\psi^*$. And the linear map is at most of rank $|\cY|$.
%	\revise{Is the rank covered in this lemma?}
%	%For feature map with universal property, we will have	$\E[Y|X_1] = \E^L [Y|\phi_1(X_1)] = (w^*)^\top \psi(X_1)$. Then our conclusion will be we could learn $\E[Y|X_1]$ by using linear regression to learn $Y$ from $\psi(X_1)$. 
%\end{remark}

Given that $\psi^*$ is good for downstream, we now care about the sample complexity.
We will need to assume that the representation has some nice concentration properties.
We make an assumption about the whitened data $\psi^*(X_1)$ to ignore scaling factors.
%Assume $\E[\psi^*(X_1)]=0$ and write $\bSigma_{\psi} = \E[\psi^*(X_1)\psi^*(X_1)^\top]$. Let random variable $U=\bSigma_{\psi}^{-1/2}\phi(X_1)$. Then it satisfies $\E[U]=0$ and $\var{U}=\bI$.  
\begin{assumption}
	\label{assumption:subgaussian_psi}
	We assume the whitened feature variable $U:=\bSigma_{\psi}^{-1/2}\psi(X_1)$ is a $\rho^2$-subgaussian random variable, where $\bSigma_{\psi} = \E[\psi(X_1)\psi(X_1)^\top]$. %\footnote{A random vector $\bx$ is called $\rho^2$-subgaussian if for any fixed unit vector $\bv$ of the same dimension, the random variable $\bv^\top \bx$ is $\rho^2$-subgaussian, i.e., $\E[e^{s\cdot\bv^\top(\bx-\E[\bx])}]\le e^{s^2\rho^2/2}$ ($\forall s\in\R$). \jz{Seems removable}}  
%Residual term $N=Y-\E[Y|X_1]$ satisfies $\E[\|N\|^2|X_1]\leq \sigma^2.$ % $k\sigma^2$-subgaussian.
\end{assumption}
We note that all bounded random variables satisfy sub-gaussian property. %For multi-classification problems $N$ is naturally bounded 

\begin{theorem}[General conditional independence]
	\label{thm:CI_nonlinear_sample_complexity}
	Fix a failure probability $\delta\in (0,1)$, under the same assumption as Lemma \ref{lemma:discrete_case_CI} and Assumption \ref{assumption:subgaussian_psi} for $\psi^*$, if additionally $n_2\gg \rho^4(k+\log(1/\delta))$, then the excess risk of the learned predictor $\bx_1\rightarrow \hat\bW^\top \psi^*(\bx_1)$ on the downstream task satsifies
\begin{center}
$\ER_{\psi^*}[\hat \bW] \leq \tilde \cO\left(\frac{k}{n_2}\sigma^2\right)$\footnote{We will use $\tilde O$ to hide log factor $\log(k/\delta)$ or $\log(d_2/\delta)$.}
\end{center}
%	\begin{equation*}
%	\ER_{\psi^*}[\hat \bW] \leq \tilde \cO\left(\frac{k}{n_2}\sigma^2\right)\footnote{We will use $\tilde O$ to hide log factor $\log(k/\delta)$ or $\log(d_2/\delta)$.}.
%	\end{equation*}
\end{theorem}
%\qi{for binary $Y$, does this $f^*(X_1)+\epsilon$ assumption make sense?... shall we change this result from regression to classification tasks? }

%
%\revise{Move this remark somewhere else.}
\begin{remark}
	This analysis assumes we could perfectly learn $\psi^*=\E[X_2|X_1]$ disregarding the number of samples in the SSL phase (unlabeled data is cheap to obtain). Here by sample complexity we refer to the labeled data $(X_1,Y)$. We defer the effect of imprecise representation $\psi$ to Section \ref{sec:beyondCI}.
\end{remark}

%\begin{remark}
%	To understand the benefits of using $\psi(X_1)$ instead of $X_1$ to predict $Y$, we note the following two points. On one hand, it is much easier to learn a linear model instead of an unknown model in $\cF_1$, and we have $\E^L[Y|\psi(X_1)]=\E[Y|X_1]$. On the other hand, it is easy to show $\var{Y|X_1}\geq \var{Y|\psi(X_1)}$.
%\end{remark}

%\revise{Add a subsection for $\psi$ being linear function of finite dimensional features? It is easy to present and understand. Section 4 will bring everything together.}
\subsection{Function class induced by feature maps.}
Given feature map $\phi_1:\cX_1\rightarrow\R^{D_1}$, we consider the function class $\cH_1=\{\psi:\cX_1\rightarrow\R^{d_2} | \exists \bB\in\R^{d_2\times D_1}, \psi(\bx_1) = \bB\phi_1(\bx_1)\}$.
%It is easy to find the optimal $\psi^*$ in this class $\cH$. 
\begin{claim}[Closed form solution]
	The optimal function in $\cH$ is $\psi^*(\bx_1) = \bSigma_{X_2\phi_1}\bSigma_{\phi_1\phi_1}^{-1}\phi_1(\bx_1)$, where $\bSigma_{X_2\phi_1} := \bSigma_{X_2\phi_1(X_1)}$ and $\bSigma_{\phi_1\phi_1} := \bSigma_{\phi_1(X_1)\phi_1(X_1)}$.
\end{claim}
%\begin{align*}
%	\bA^* = \arg\min_{\bA\in\R^{d_1\times d_2}} \E_{X_1,X_2} [\|X_2 - \bA\phi_1(X_1)\|^2]
%\end{align*}
%This is very similar to the Gaussian case, but the optimal function is linear function in $\phi_1(\bx_1)$ instead of $\bx_1$.
We again show the benefit of CI, but only comparing the performance of $\psi^*$ to the original features $\phi_1$.
Since $\psi^*$ is linear in $\phi_1$, it cannot have smaller approximation error than $\phi_1$.
However CI will ensure that $\psi^*$ has the same approximation error as $\phi_1$ and enjoys better sample complexity.
\begin{lemma}[Approximation error]
	\label{lemma:linear_case_CI}
	If Assumption \ref{assump:independence} is satisfied, and if the matrix $\bA \in \R^{\cY\times d_2} $ with $\bA_{y,:} := \E[X_2|Y=\by]$ is of rank $k=|\cY|$.
	Then $e_\apx(\psi^*) = e_{\apx}(\phi_1)$.
\end{lemma}
%We will need an additional 
We additionally need an assumption on the residual $a(\bx_1):=\E[Y|X_1=\bx_1]- \E^L[Y|\phi_1(\bx_1)]$. 
\begin{assumption}
	\label{assumption:bounded_error_finite_dim} (Bounded approx. error; Condition 3 in \cite{hsu2012random}))
We have almost surely $$\|\bSigma_{\phi_1\phi_1}^{-1/2}\phi_1(X_1)a(X_1)^\top\|_F\leq b_0\sqrt{k}$$
\end{assumption}

\begin{theorem}(CI with approximation error)
	\label{thm:CI_linear_sample_complexity}
%	\revise{Need to fix this statement, define subgaussian assumption in this case, and additional Sham assumption.}
	Fix a failure probability $\delta\in (0,1)$, under the same assumption as Lemma \ref{lemma:linear_case_CI}, Assumption \ref{assumption:subgaussian_psi} for $\psi^*$ and Assumption \ref{assumption:bounded_error_finite_dim}, if $n_2\gg \rho^4(k+\log(1/\delta))$, then the excess risk of the learned predictor $\bx_1\rightarrow \hat\bW^\top \psi^*(\bx_1)$ on the downstream task satisfies:
	\begin{center}
		$\ER_{\psi^*}[\hat \bW] \leq e_{\apx}(\phi_1) +\tilde \cO\left(\frac{k}{n_2}\sigma^2\right)$.
	\end{center}
%	\begin{equation*}
%		\ER_{\psi^*}[\hat \bW] \leq e_{\apx}(\phi_1) +\tilde \cO\left(\frac{k}{n_2}\sigma^2\right). 
%	\end{equation*}
\end{theorem}

%\revise{Comment that everywhere $\bx_2$ can be replaced with $\phi_2(\bx_2)$, only full rank-ness needs to be satisfied.}
%\revise{Weakening of assumption as in Assumtion~\ref{assumption:CI_with_latent_variables} can also be done for these cases.}

%Theorem \ref{thm:CI_linear_sample_complexity} is also true with Assumption \ref{assumption:CI_with_latent_variables} instead of exact CI, if we replace $k$ by $km$. 
Thus with SSL, the requirement of labels is reduced from complexity for $D_1$ to $\cO(k)$.  %( or $\cO(km)$). 

% !TEX root = colt2021-SSL.tex

\section{Beyond conditional independence}
\label{sec:beyondCI}
In the previous section, we focused on the case where we have exact CI. A weaker but more realistic assumption is that $Y$ captures some portion of the dependence between $X_1$ and $X_2$ but not all.  %We utilize covariance operators in RKHS to quantify this. %we assume that there exists some latent variable $Z$ such that $X_1$ and $X_2$ are either conditional independence or has small correlation given both $Y$ and $Z$. 
%To quantify this, we introduce covariance operator to measure approximate CI. This is a generalization from the jointly-Gaussian case which we defer to Section \ref{sec:joint_gaussian_apx_CI}. 
\begin{comment} 
Informally, by discarding the CI assumption in Theorem \ref{thm:linear_CI}, the excess risk of downstream task of the two-stage SSL procedure is now bounded by:
\begin{align} \cO\left( \epsilon^2_{\tCI} +\Trace(\bSigma_{YY|X_1})\frac{d_2+\log(d_2/\delta)}{n_2}  \right), \text{ where }\epsilon_{\tCI}=\|\bSigma_{X_1X_2| Y}\|_F. \label{eq:linear-ACI}\end{align}
Compared to Theorem \ref{thm:linear_CI}, this bound has an additional term $\cO(\epsilon^2_{\tCI})$.
This captures the fact that $f^*$ is close but no longer exactly linear in $\psi^*$ with ACI.
Additionally, the representation $\psi^*(\bX_1^{(\down)})$ is rank $d_2$ instead of $k$, so our second term also gets worse.
However, if we do PCA on $\psi^*(\bX_1^{(\down)})$ first and use the selected features to predict downstream task, we could sharpen the bound to $\cO\left( \epsilon^2_{\tCI}+\Trace(\bSigma_{YY|X_1}) \frac{k  + \log (k/\delta)}{n_2}\right)$. Due to space limit, the formal statement of \eqref{eq:linear-ACI} is presented in Section \ref{sec:joint_gaussian_apx_CI}.
\end{comment} 
%
We quantify this notion of approximate ACI through a quantity $\epsilon^2_{\tCI}$ (Definition~\ref{def-approx-CI}), and show excess risk bounds for the representation learned from SSL\footnote{Results for jointly-Gaussian variables is in Appendix \ref{sec:joint_gaussian_apx_CI}; ACI is quantified by the partial covariance matrix.}.
%In this setting, we learn a representation $\tilde \psi$ from a function class $\cH$ by using finite samples of unlabeled data, and then use this representation to learn a linear classifier $\hat \bW$ to predict $Y$ from $n_2$ labeled samples $(X_1, Y)$.
%We consider two types of hypothesis classes for the representation $\psi$, 1) $\cH_1$ that is a linear representation class on {\em known} features $\phi : \cX_1\rightarrow \R^D$, 2) $\cH_u$ that has universal approximation power (e.g. neural networks).
%In both cases we compare the performance of the learned predictor $\hat \psi$ to the best representation in $\cH$ to predict the label $Y$.
In particular, the excess risk will have the form $\tilde{\cO}\left(\frac{d_2}{n_2} + \epsilon^2_{\tCI} + \epsilon^2_{\pre}\right)$, which suggests that only $n_2 = \cO(d_2)$ labeled samples will be required to get small error on downstream task, as long as approximate CI is satisfied ($\epsilon^2_{\tCI}$ is small) and the pretext task is solved well enough ($\epsilon^2_{\pre}$ is small).
This is in contrast to not doing SSL, where many more labeled samples will be required to learn a solve the downstream task that learns a complicated representation function from scratch.
We now describe the SSL method on finite samples, followed by the definition of ACI which we use to discuss the main excess risk bound and its consequences.
% the second two terms capture the notion of approximate CI $\epsilon^2_{\tCI}$ and that we do well on the pretext task up to an error of $\epsilon^2_{\pre}$ due to finite unlabeled data, while the first term captures the reduced labeled sample complexity, i.e. only $\gO(d_2)$ suffice if 

%\subsection{Learnability with general function space}\label{subsec:general_function_class}
\textbf{SSL with finite samples and general function space:}\label{subsec:general_function_class}
%We state the main result with finite samples for both pretext task and downstream task. 
Let $\bX_1^{\pre}=[\bx_1^{(1,\pre)},\cdots,\bx_1^{(n_1,\pre)}]^\top \in \R^{n_1\times d_1}$ and $\bX_2=[\bx_2^{(1)},\cdots,\bx_2^{(n_1)}]^\top \in \R^{n_1\times d_2}$ be $n_1$ training samples for pretext task, where $(\bx_1^{(i,\pre)},\bx_2^{(i)})$ is sampled from $P_{X_1X_2}$.
The $n_2$ labeled samples for the downstream task are defined as $\bX_1^{\down}\in \R^{n_2\times d_1}$, $\bY\in \R^{n_2\times d_3}$\footnote{$d_3=k$ and $Y\equiv \phi_y(Y)$ (one-hot encoding) refers multi-class classification task, $d_3=1$ refers to regression.}.
Given a representation function space $\cH:\cX_1\rightarrow\R^{d_2}$, we learn $\tilde \psi$ from $\cH$ using the $n_1$ unlabeled samples and then use the $n_2$ labeled samples to learn a linear classifier on the learned representation $\tilde \psi(\bX_1^{\down})$ to fit $\bY$.
This process is summarized below.
\begin{align}
	1)~~ \tilde\psi:= \argmin_{\psi\in \cH} \frac{1}{n_1}\|\bX_2-\psi(\bX_1^{\pre})\|_F^2, \ 2)~~ \hat \bW\leftarrow \argmin_{\bW} \frac{1}{2n_2}\|\bY-\tilde\psi(\bX_1^{\down})\bW\|_F^2.
\end{align}

In our main results, we consider two types of function spaces: $\cH\in \{\cH_1, \cH_u\}$. Recall that $\cH_1=\{\psi(\cdot) = \bB\phi_1(\cdot); \bB\in\R^{d_2\times D_1}\}$ is a class of {\em linear representations} induced by feature map $\phi_1:\cX_1\rightarrow\R^{D_1}$.
We use $\cH_u$ to denote a function space with universal approximation power (e.g. deep networks) that ensures $\psi^*=\E[X_2|X_1]\in \cH_u$. 
We define the optimal predictor in each case as $f^*_{\cH}(X_1)=\E^L[Y|\phi_1(X_1)]$ when $\cH=\cH_1$, $f^*_{\cH}=f^*$ for $\cH=\cH_u$, we define excess risk as
\begin{equation*}
\ER_{\tilde{\psi}}(\hat \bW) := \E_{X_1} \left[\|f^*_{\cH}(X_1) - \hat\bW^\top \tilde{\psi}(X_1) \|_2^2\right].
\end{equation*} 
%\revise{Is $\ER$ definition correct?}
%where $f^* = \E[Y|X_1].$% is the optimal solution in function space $\cH_1^k$.

\textbf{Approximate conditional independence:} Our new assumption will generalize Assumption~\ref{assump:independence} in two ways, 1) we allow for additional latent variables $Z$ that together with $Y$ could potentially make $X_1$ and $X_2$ independent, and 2) we allow this conditional independence to be approximate.
Note that allowing for extra latent variable can trivially make $X_1$ and $X_2$ to be conditionally independent by picking a large enough $Z$ (e.g. $Z=(X_1,X_2))$.
However the following assumption, that needs the pretext target $X_2$ to correlate with all instances of variable $\bar{Y}=[Y,Z]$ (analogous to Lemma~\ref{lemma:discrete_case_CI}), will impose this restriction on how large $Z$ can be.
\begin{assumption}[Correlation between $X_2$ and $Y,Z$]
	\label{assumption:feature_map_approx_CI}
	Suppose there exists latent variable $Z\in \cZ, |\cZ|=m$ that ensures 
%	\begin{equation*}
	%\label{eqn:bounded_conditional_covariance_operator}
	$\bSigma_{\phi_{\bar y}X_2}\text{ is full column rank and } \|\bSigma_{Y\phi_{\bar y}}\bSigma_{X_2\phi_{\bar y}}^\dagger\|_2 = 1/\beta$,
%	\end{equation*}
	%(with $r=|\cY||\cZ|\leq d_2$), 
	where $A^\dagger$ is pseudo-inverse, and  $\phi_{\bar y}$ is the one-hot embedding for $\bar Y=[Y,Z]$.
%	\jnote{$\beta$ assumption is wrong unless $Y \in R^k$ is one-hot. $\beta$ is size of linear predictor from $X_2 \to Y$ so must depend on both scale of $X_2 $ and $Y$.}.
\end{assumption}
Just as in Section~\ref{sec:CI}, this assumption will not assume away the problem (Example~\ref{example:mixture_gaussian} can be suitably extended).
The additional term $1/\beta$ here captures both the ``scale'' of $X_2$ and also the strength of correlation between $X_2$ and $[Y,Z]$ that was discussed after Lemma~\ref{lemma:discrete_case_CI}.
For $\bSigma_{\phi_{\bar y}X_2}$ to be full column rank, it is essential that $d_2 \ge km$, and this already gives an upper bound on the size of $Z$.
Given this restriction on $Z$ (and thus $\bar{Y}$), we define the notion of approximate conditional independence.

\begin{definition}[Approximate conditional independence with function space $\cH$] For $\bar{Y}=[Y,Z]$,
\label{def-approx-CI}
~	\\
 1.	For $\cH = \cH_1$, define $\epsilon_{\tCI}:= \|\bSigma_{\phi_1\phi_1}^{-1/2} \bSigma_{\phi_1 X_2|\phi_{\bar y} } \|_F$. \\
2.	For $\cH=\cH_u$, define $\epsilon_{\tCI}^2:=\E_{X_1}[\|\E[X_2|X_1]-\E_{\bar Y}[\E[X_2|\bar Y]|X_1]\|^2]$. 
\end{definition}
Firstly we note that this is indeed an extension of exact CI, since exact CI in both cases will imply that $\epsilon_{\tCI}=0$.
We present a unified analysis in the appendix that shows the $\epsilon_{\tCI}$ for the second case is same as the first case, with covariance operators instead of matrices (A direct derivation is in Claim \ref{claim:relation_eps_CI}).  %actually the Hilbert-Schmidt norm (generalized Frobenius norm from matrices to operators) of the cross covariance operator between $X_1$ and $X_2$ given $Y,Z$.  
We also present more relaxed and general form of the above assumptions in Appendix~\ref{apx:general}.
\begin{comment}
\begin{remark}
	We note the quantities in the assumption are invariant to the different choices of feature map $\phi_1:\cX_1\rightarrow \cF_1$ or the inner product $\langle\cdot,\cdot\rangle_{\cF}$ but only depend on the function space $\cH_1$. Specifically, with universal feature map, the terms 
	$\|\cC_{\phi_1\phi_1}^{-1/2} \cC_{\phi_1X_2|\phi_{\bar y}}\|_{\HS}^2 = \E_{X_1}[\|\E[X_2|X_1]-\E_{\bar Y}[\E[X_2|\bar Y]|X_1]\|^2] $  %and the term $\cC_{\phi_{\bar y}\phi_{\bar y}}^{-1}\cC_{\phi_{\bar y}X_2}$ 
	only depends on the joint distribution $P_{X_1X_2YZ}$. $1/\beta$ bounds the spectral norm of the optimal $\bW^*$.  %Specifically, for a universal RKHS with eigenfunctions $\varphi_1,\varphi_2,\cdots \in \cH_1$, 
	%\revise{This is an important point. Perhaps move it to right after Assumption \ref{assumption:bounded_error_non_universal}?}
\end{remark}
\end{comment}
With this assumption, we are ready to present our main bound.

\textbf{Bound on excess risk:}
Recall that we assume that the residual term $N:=Y-\E[Y|X_1]$ is mean zero and $\sigma^2$-subgaussian.
Before showing our main result, analogous to Assumption~\ref{assumption:bounded_error_finite_dim}, for the class $\cH_1$ with non-universal features $\phi_1$, we will need an assumption\footnote{This rules out the failure if one chooses a very simple function class to learn $\E[X_2|X_1]$. In practice we usually use neural networks (with universal approximation power) and this bound should be very small.} on the residual $a:=f^*- f^*_{\cH_1}=\E[Y|X_1]-\E^L[Y|\phi_1(X_1)]$: 
\begin{assumption}(Bounded approximation error on pretext phase \citep{hsu2012random})
	\label{assumption:bounded_error_non_universal}
There exists a universal constant $b_0$, such that $\|\bSigma_{\phi_1\phi_1}^{-1/2}\phi_1(X_1)a(X_1)^\top\|_F\leq b_0\sqrt{d_2}$ almost surely. 	
\end{assumption}

\begin{theorem}
	\label{thm:main_result_approximate_CI}
 For a fixed $\delta\in (0,1)$, under Assumptions \ref{assumption:feature_map_approx_CI},\ref{assumption:subgaussian_psi}  for $\tilde{\psi}$ and $\psi^*$ and \ref{assumption:bounded_error_non_universal} for non-universal feature maps, if $n_1,n_2\gg \rho^4(d_2+\log 1/\delta)$, and we learn the pretext tasks such that:
$ \E\|\tilde \psi(X_1)-\psi^*(X_1)\|_F^2\leq  \epsilon^2_{\pre}.$
Then the generalization error for downstream task w.p. $1-\delta$ is: 
\begin{align}
 \label{eqn:main_result_approx_CI}
\ER_{\tilde \psi}(\hat \bW) \leq \tilde\cO\left(\underbrace{\sigma^2\frac{d_2}{n_2}}_{\text{estimation error}} + \underbrace{ \frac{\epsilon^2_{\tCI}}{\beta^2} +\frac{\epsilon_{\pre}^2}{\beta^2 }}_{\text{approximation error}}\right)%\revise{ + e_{\apx}(\phi_1)?}.
\end{align} 
\end{theorem}
We defer the proof to the appendix. 
The proof technique is similar to that of Section \ref{sec:CI}. The difference is that now $\tilde\psi(\bX^{(\down)}) \in \R^{n_2\times d_2}$ will be an approximately low rank matrix, where the low rank part is the high-signal features that implicitly comes from $Y,Z$ that can linearly learn downstream task. The remaining part comes from $\epsilon_{\tCI} $ and $\epsilon_{\pre}$ and causes the approximation error. Again by selecting the top $km$ (dimension of $\phi_{\bar y}$) features we could further improve the bound:

%Without Assumption \ref{assumption:bounded_error_non_universal}, we could show $ \E[\|f^*_{\cH_1} - \hat\bW^\top \tilde\psi(X_1) \|^2] $ is bounded by the RHS of \eqref{eqn:main_result_approx_CI}. Namely, our learned function is close enough to the best predictor within the function space $\cH_1^k$.

\begin{remark}
	\label{remark:pca}
By applying PCA on $\tilde \psi(\bX_1^{\down})$ and keeping the top $km$ principal components only, we can improve the bound in  Theorem \ref{thm:main_result_approximate_CI} to 	
$\ER_{\tilde \psi}(\hat \bW) \leq \tilde\cO\left( \sigma^2\frac{km}{n_2} + \frac{\epsilon_{\tCI}^2}{\beta^2} +\frac{\epsilon_{\pre}^2}{\beta^2 }\right)$.
%Our learned representation $\tilde \psi:\R^{d_1}\rightarrow \R^{d_2}$ captures the information for $Y$ and $Z$ with cardinality $km\leq d_2$. Therefore we could simply select the most important features to predict $Y$. Specifically, if we do PCA on $\tilde \psi(\bX_1^{\down})$ and use the top $km$ features to predict $Y$, we could further improve the bound in
% \begin{align}
% \label{eqn:main_result_with_PCA}
% \ER_{\tilde \psi}(\hat \bW) \leq \tilde\cO\left( \sigma^2\frac{km}{n_2} + \frac{\epsilon_{\tCI}^2}{\beta^2} +\frac{\epsilon_{\pre}^2}{\beta^2 }\right).
% \end{align} 
\end{remark}
%\jnote{add remark saying if we do pcr/ridge regression on $\psi$ can get $k/n_2$.}

We take a closer look at the different sources of errors in Lemma~\ref{remark:pca}: 1) The first term is estimation error on learning with finite samples $n_2$ with noise level $\sigma^2$ in $Y-f^*(X_1)$; 2)  $\epsilon_{\tCI}$ measures the approximate CI; and 3) $\epsilon_{\pre}$ is the error from not learning the pretext task exactly. The first term is optimal ignoring log factors as we do linear regression on $mk$-dimensional features. The second and third term together form approximation error. They are non-reducible due to the fact that $f^*$ is not exactly linear in $\psi$ and we use it as a fixed representation. Fine-tuning the representations might be necessary to get rid of these terms when we have sufficient downstream labeled data. We leave this exploring this as future work.
Compared to traditional supervised learning, learning $f^*_{\cH}$ requires sample complexity scaling with the (Rademacher/Gaussian) complexity of $\cH$ (see e.g. \cite{bartlett2002rademacher,shalev2014understanding}), which is very large for complicated models such as deep networks. %\qi{might want to mention other unsupervised methods/contrastive loss are not directly comparable}
Thus SSL can significantly reduce the labeled sample complexity down from this complexity measure of $\cH$ to $\tilde{\cO}(km)$, demonstrating the power of predicting what you already know using unlabeled data.
In Section \ref{sec:classification}, we consider a similar result for classification.

%\qi{add remark saying all the results could be improved to O(k/n) if we do pca first}
%\qi{will combine the robust analysis in the next section to here}

% !TEX root = colt2021-SSL.tex

\section{Example: Topic Modeling}
\label{sec:topic_model}

In this section, we will demonstrate how our framework can be instantiated for mixed-membership models including topic models, not just clustering.
Topic modeling for text has a rich literature \citep{papadimitriou2000latent,hofmann1999probabilistic,blei2003latent,arora2012learning,arora2013practical} and is used for analyzing and designing algorithms for information retrieval, dimensionality reduction and data analysis for large text corpora.
We describe the basic setup below, followed by how our results for reconstruction-based SSL can be instantiated to learn such models.

For a set $S$, let $\Delta_{S}$  denote the set of all distributions on $S$.
In the topic modeling framework, generation of a text document with a vocabulary set $[V]=\{1,\dots,V\}$ is governed by certain latent topics from the set $[k]$, where $k$ is the total number of topics.
Each topic $i\in[k]$ is associated with a distribution over the vocabulary $[V]$ that is denoted by vector $A_{i}\in\Delta_{[V]}$; stack these vectors into the columns of a matrix $A\in\R^{V\times k}$.
A document $X=(x_1,\dots,x_n)\in[V]^{N}$ of length $N$ is then sampled from a mixture of the $k$ topics $\mu\in\Delta_{[k]}$.
The generative process is described below:
\begin{enumerate}
	\item Sample a topic mixture $\mu\sim\tau$, where $\tau$ is some underlying distribution over $\Delta_{k}$, i.e. $\tau\in\Delta_{\Delta_{[k]}}$
	\item For each $i\in[N]$, sample a topic $t_{i}\sim\mu$ and sample a word $x_{i}\sim A_{t_{i}}$ from the topic
\end{enumerate}
For the reconstruction SSL task, we evenly split the document as $X = (\bar{X_1}, \bar{X_2})$, where $\bar{X_1}$ and $\bar{X_2}$ denote the first and second halves of the document; note that $\bar{X_1},\bar{X_2} \in [V]^{N/2}$.
We let $X_1$ and $X_2$ be the multiset of words in the two halves by using the normalized bag-of-words representation, i.e. $X_i = \frac{2}{N} \text{bag-of-words}(\bar{X_i})\in\R^{V},~i\in\{1,2\}$\footnote{We only need $X_2$ to be the bag-of-word representation, $X_1$ can be an ordered sentence.}.
The SSL task is to learn a representation $\psi \in \{\psi(\cdot) = \bB\phi_1(\cdot); \bB\in\R^{V\times V}\}$ that minimizes $\left\|\psi(X_1) - X_2\right\|^2$.

%Note that due to independent sampling of words in the topic modeling assumption, $
The downstream task is chosen to be a linear function of the topic posterior distribution $\mu$ for a given document $X$, i.e. $Y = w^{\top} \E[\mu | X] + N$, where $N$ is 0 mean and $\sigma^2$-subgaussian.
%\ns{Example of what this task could mean, like separating between sports related topics and science related topics.}
The error of a predictor $f:[V]^{N} \rightarrow \R$ is measured as $\E_{X,Y}\left[\left(f(X) - Y\right)^2\right]$, the optimal predictor being $f^{*}(X) = \E\left[ Y \mid X \right]$.

A crucial property of topic model described above is that words in the document are sampled independently given the topic mixture $\mu$, thus giving us the property: $X_1 \perp X_2 \mid \mu$.
Although the cardinality of $\mu\in\Delta_{[k]}$ (that implicitly shows up in Theorem~\ref{thm:main_result_approximate_CI}) is infinite, we can still show the benefit of SSL using our theoretical framework.
We will show appropriate bounds for $\epsilon_{\tCI}$ and $\beta$, that show up in Theorem~\ref{thm:main_result_approximate_CI}, using the topic model generative process.
%We make the following standard assumptions about the topic modeling distribution, motivated by prior work \citep{arora2012learning,arora2013practical}.
%\begin{assumption}
%\label{asmp:topic_model}
%	Let $A\in\R^{V\times k}$ be the word-topic matrix and $\Gamma = \E_{\mu\sim\tau}\left[\mu\mu^{\top}\right]$ be the topic covariance matrix, then the following hold
%	\begin{itemize}
%	\item (Anchor word) The word-topic matrix $A$ is $p$-separable for $p>0$, i.e. for every topic $i\in[k]$ there is a word $j$ such that $A_{i}(j) \ge p$ and $A_{i'}(j) = 0$ when $i' \neq i$
%	\item $\Gamma$ is full rank, so condition number $\kappa = \frac{\lambda_{\max}(\Gamma)}{\lambda_{\min}(\Gamma)} < \infty$
%	\end{itemize}
%\end{assumption}
%

\begin{corollary}
\label{thm:topic_model}
	Given a topic model characterized by $(A, \tau)$, suppose $\Gamma = \E_{\mu\sim\tau}\left[\mu\mu^{\top}\right]$ is the topic covariance matrix and let $\kappa = \frac{\lambda_{\max}(\Gamma)}{\lambda_{\min}(\Gamma)} < \infty$ be its condition number.
	Let $\epsilon_{\tCI}$ be the definition (2) from Definition~\ref{def-approx-CI} and $\beta$ as defined in Assumption~\ref{assumption:feature_map_approx_CI}, then there exists a latent variable $\bar{Y}\in\bar{\cY}$ such that the following hold
%	Suppose the topic model satisfies Assumption~\ref{asmp:topic_model}, then there exists a latent variable $\bar{Y}\in\bar{\cY}$ such that the following hold
	\begin{enumerate}
		\item $\bar{Y}$ takes $k$ distinct values, i.e. $|\bar{\cY}| = k$
		\item $X_1$ and $X_1$ are uncorrelated given $\bar{Y}$, which implies $\epsilon_{\tCI} = 0$.
		\item $\E[Y | X_1]$ is a linear function of $\E[\bar{Y} | X_1]$
		\item $\beta^{-1} \le \kappa\|w\|_2/\lambda_{\min}(A)$
	\end{enumerate}
\end{corollary}
The proof is presented in Section~\ref{sec:topic_model_proof}.
%Note that the $p$-separability is not necessarily needed, and the bound with $\lambda_{\min}(A)$ can be invoked instead.
Thus the upper bound from Theorem~\ref{thm:main_result_approximate_CI} will look like $\tilde{\cO}\left(\sigma^2\frac{k}{n_2} + \epsilon_{\pre}^2 \frac{\kappa \|w\|_2}{\lambda_{\min}(A)} \right)$, thus requiring only $\cO(k)$ samples for the downstream task.

% !TEX root = colt2021-SSL.tex

\section{Conditional distribution decomposition: SimSiam, CCA, ACE}
\label{sec:learn_joint_distribution} 

In this section we establish the connection between SimSiam~\cite{chen2021exploring} and non-linear CCA between $X_1$ and $X_2$ and the alternating conditional expectation (ACE) algorithm.
We show how our previous analysis can be extended to this setting and how the problem relates to decomposing the conditional distribution of $X_2 \mid X_1$.

\subsection{Theoretical guarantees for non-linear CCA}% and Contrastive Learning}
%\ns{Another attempt}
In the previous sections, we used $\psi$ to predict $X_2$ given $X_1$.
As discussed in Remark~\ref{remark:encoder_decoder}, we could have predicted $\eta(X_2)$ from $X_1$ for any function $\eta$, with all bounds depending on the function $\eta$.
An alternative is to avoid choosing a specific $\eta$, but instead simultaneously learn an $\eta$ that can be easily predicted from $X_1$.
We further show how our problem setup and analysis can capture the popular method of SimSiam, an SSL method that does not use negative samples.

%This is equivalent to approximately learn the conditional distribution of $X_2|X_1$, as we will demonstrate in the next section. 

% this is alternating conditional expectation (ACE)..
% it is equivalent to non-linear CCA, with the formulation very close to contrastive learning

%This is a common practice in some self-supervised learning methods like contrastive learning \citep{}.
We first formulate the aforementioned problem and show that it corresponds to performing non-linear canonical component analysis (CCA) \citep{hardoon2004canonical} on the joint distribution of $(X_1, X_2)$. We let $L^2(X)$ denotes the Hilbert space of square integrable function with respect to the measure $P_X$, the marginal distribution of $X$.
For instance, in our context of SSL, for a function $g:\R^{d_2}\rightarrow \R$, we denote $\|g\|^2_{L^2(X_2)} = \int g^2(x_2)dP_{X_2}(x_2)$ and thus $L^2(X_2)=\{ g:\R^{d_2}\rightarrow \R~ \mid ~\|g\|^2_{L^2(X_2)} <\infty.   \}$.
%Furthermore for any scalar function $g$ of $X$, we denote $\|g\|^2_{L^2(X)} = \int g^2(x)dP_{X}(x)$.

For zero-mean representation functions $\psi: \psi_i\in L^2(X_1), \eta: \eta_i\in L^2(X_2), i\in [k]$, we consider the generalized alternating conditional expectation (ACE) algorithm (\cite{makur2015efficient,breiman1985estimating,buja1990remarks}) that optimizes the following: %we define the following reconstruction based objective
%\begin{align*}
%	L(f, g) = \E_{X_1, X_2} \left[\left\|f(X_1) - g(X_2)\right\|^{2}\right], ~\text{s.t.}~\bSigma_{f,f} = \bSigma_{g,g} = I_{k}
%\end{align*}
%\ns{Notation issues: we had used $f$ for a $Y | X_1$ predictor earlier but for representation function here; may be we should use $\psi$. we used $k$ for number of classes earlier but for representation dimension here, should use $d$ may be. Can use something like the following instead.} 
\begin{align}
\label{eqn:ACE} 
\min_{\psi,\eta}	L_{\text{ACE}}(\psi, \eta) := \E_{X_1, X_2} \left[\left\|\psi(X_1) - \eta(X_2)\right\|^{2}\right], ~\text{s.t.}~\bSigma_{\psi, \psi} = \bSigma_{\eta, \eta} = \bI_{k}
\end{align}
Here $\bSigma_{\psi,\psi} \in \R^{k\times k}$ and $(\bSigma_{\psi,\psi})_{i,j} = \E_{X_1}[\psi_i(X_1)\psi_j(X_1)]$ and similarly for $\eta:\cX_2\rightarrow \R^k$. 
As we will show in Proposition \ref{prop:ace=cca}, the above objective is equivalent to the following non-linear CCA: %has the following connection to non-linear CCA \ns{state obvious connection}.
\begin{align*}
\max_{\psi,\eta}	L_{\text{CCA}}(\psi, \eta) := \E_{X_1, X_2} \left[\psi(X_1)^\top \eta(X_2)\right], ~\text{s.t.}~\bSigma_{\psi, \psi} = \bSigma_{\eta, \eta} = \bI_{k}.
\end{align*}
\paragraph{Connection to SimSiam:}
In the setting for the SimSiam~\cite{chen2021exploring} method, $X_1$ and $X_2$ are two randomly augmented images. %In the setting of CLIP~\cite{radford2021learning}, $(X_1, X_2)$ are image, text pairs collected from the internet. 
The non-linear CCA problem is almost identical to SimSiam, except that we use normalization of representation instead of stop-gradient to prevent representation collapse.
CCA maximizes the inner product of the representations for each positive pairs $(X_1, X_2)$ generated from their joint distribution. At the same time, the normalization constraint ensures that the representation doesn't collapse to trivial function, so we do not need negative samples. 
We now demonstrate how our previous analysis can easily apply to non-linear CCA.

%\ns{Present a result (similar to Corollary~\ref{coro:PCR_non-zero_maximal_correlation}) for the optimal solution $\psi^*$, with definition of $\tilde\epsilon_{\tCI} = \max_{\E g(X_2)^2 = 1} \E_{X_1} (\E[g(X_2)|X_1] - \E[\E[g(X_2)|Y]|X_1])^2$. Hopefully $\tilde \beta$ can also be defined without using many operators. Can then also state Corollary~\ref{thm:PCR_symmetric} type result with the Bayes error assumption.}

\begin{theorem}[General theorem for non-linear CCA]
	\label{thm:PCR_general_main}
Let $\psi:\cX_1\rightarrow \R^k, \eta:\cX_2\rightarrow \R^k$ be the solution of Eqn. \eqref{eqn:ACE}. Denote scalars $\sigma_i:= \E_{X_1X_2}[\psi_i(X_1)\eta_i(X_2)]$. 
%	Now treat $\psi(x_1)=[u_1(x_1),\cdots u_k(x_1)]: \cX_1\rightarrow \R^k$ as the representation. 
Then the approximation error of $\psi$ satisfies:
	\begin{align*}
	e_{\apx}(\psi):= & \min_{\bW\in\R^{k\times k}}\E[\|f^*(X_1)-\bW^\top \psi(X_1)\|^2]\\
	\leq & \sum_{y=1}^k \min_{g_y\in L^2(X_2)} 2(\|(\cT_k - \cL)\circ g_y\|^2_{L^2(X_1)} +\|\cL \circ g_y - f^*_y\|^2_{L^2(X_1)}) . 
	\end{align*}	
	Here $f^*$ is the optimal function to predict the one-hot encoder of $Y$ with $X_2$, i.e., $f^*_y(x_1) = \E[1(Y=y) | X_1=x_1 ] = P(Y=y|X_1=x_1)$. Here $(\cT_k \circ g_y)(x_1):= \sum_{i=1}^k\sigma_i \E[\eta_i(X_2) g_y(X_2)] \psi_i(x_1)$, and $(\cL\circ g_y)(x_1):=\E_Y[\E_{X_2}[g_y(X_2)|Y]|X_1=x_1 ]. $
\end{theorem}

%\ns{Some definitions are missing, like $L^2(X_1),L^2(X_2)$. Are there simpler ways to define the operators $\cL$ and $\cT$?}

%\begin{theorem}[General theorem for non-linear CCA]
%	\label{thm:PCR_general_main_}
%Let $\psi:\cX_1\rightarrow \R^k, \eta:\cX_2\rightarrow \R^k$ be the solution of Eqn. \eqref{eqn:ACE}. 
%%	Now treat $\psi(x_1)=[u_1(x_1),\cdots u_k(x_1)]: \cX_1\rightarrow \R^k$ as the representation. 
%Then the approximation error of $\psi$ satisfies:
%	\begin{align*}
%	e_{\apx}(\psi):= & \min_{\bW\in\R^{k\times k}}\E[\|f^*(X_1)-\bW^\top \psi(X_1)\|^2]\\
%	\leq & \sum_{y=1}^k \min_{g_y\in L^2(X_2)} 2(\|(\cT_k - \cL)\circ g_y\|^2_{L^2(X_1)} +\|\cL \circ g_y - f^*_y\|^2_{L^2(X_1)}) . 
%	\end{align*}	
%	where $f^*_y$ is the optimal function to predict the one-hot encoder of $Y$ with $X_2$, i.e., $f^*_y(x_1) = \E[1(Y=y) | X_1=x_1 ] = P(Y=y|X_1=x_1)$. For scalars $\sigma_i:= \E_{X_1X_2}[\psi_i(X_1)\eta_i(X_2)]$, the operators are defined as $(\cT_k \circ g_y)(x_1):= \sum_{i=1}^k\sigma_i \E[\eta_i(X_2) g_y(X_2)] \psi_i(x_1)$, and $(\cL\circ g_y)(x_1):=\E_Y[\E_{X_2}[g_y(X_2)|Y]|X_1=x_1 ].$
%\end{theorem}
The proof of this theorem and its corollaries below can be found in Appendix~\ref{sec:proof_cca_ace}.
With this theorem, we can apply different choices of $g_y$ to derive the generalization bound. 
If we choose $g_y$ such that $\E[g_y(X_2)|Y=y] = 1(Y=y)$, 
 we get the following generalization bound:

%\jnote{Explain choices of $g_y$ for both corollaries.}
\begin{corollary}[Generalization bound with non-linear CCA.]% when $X_2$ and $Y$ have non-zero maximal correlation. ]
	\label{coro:PCR_non-zero_maximal_correlation_main}
	In the same setting of Theorem \ref{thm:PCR_general_main}, and suppose the learned $\psi$ satisfies Assumption \ref{assumption:subgaussian_psi}, % suppose the $(k-1)$-th maximal correlation between $X_2$ and $Y$ is not zero, 
	then we have:
	$$ ER_{\psi}(\hat\bW) \leq \tilde O\left(\frac{k\tilde\epsilon_{\tCI}^2}{\tilde\lambda^2} + \sigma^2 \frac{k}{n_2} \right) .$$
	Here $ \tilde\epsilon_{\tCI}^2 := \max_{\|g\|_{L^2(X_2)}=1} \E_{X_1} (\E[g(X_2)|X_1] - \E[\E[g(X_2)|Y]|X_1])^2  $ is the measure of approximate conditional independence, and $\tilde\lambda$ is the ($k-1$)-th maximal correlation between $X_2$ and $Y$\footnote{The definition and more discussion of maximal correlation between two random variable are deferred in Definition \ref{def:max_correlation} and the next subsection.}. 
	%\ns{Link back to definitions of $\tilde\epsilon_{\tCI}^2$ and $\tilde\beta^2$.}
\end{corollary}
%\ns{Can link to the appropriate definition for $k$-th maximal correlation, which I believe is Definition~\ref{def:max_correlation}}

%\paragraph{When $X_1$ and $X_2$ are two views from the same data distribution. }

\begin{assumption}[$\alpha$-Bayes error]
	\label{assump:small_bayes_error} 
	We assume $Y$ is almost deterministic when predicting from either $X_1$ or $X_2$.   Specifically, there exists a classifier $g_1^*$ such that $P_{X_1,Y}(g_1^*(x)\neq y) \leq \alpha$; there exists $g_2^*$ such that 	$P_{X_2,Y}(g_2^*(x)\neq y) \leq \alpha$.
\end{assumption}
If we choose $g_y(x_2)=1(g_2^*(x_2)=y),\forall y\in [k]$ where $g_2^*:=\E[Y|X_2]$ in Theorem \ref{thm:PCR_general_main}, we get the following corollary:
\begin{corollary}[Guarantees with small Bayes error]
	\label{thm:PCR_symmetric_main} 
	Under the same setting and algorithm as Corollary \ref{coro:PCR_non-zero_maximal_correlation_main}, if additionally we assume $\alpha$-Bayes error (Assumption \ref{assump:small_bayes_error}), we have that the generalization error also satisfies:
	\begin{equation*}
	ER_{\psi}(\hat\bW) \leq \tilde O\left(\frac{\alpha}{1-\lambda} + \sigma^2 \frac{k}{n_2} \right) ,
	\end{equation*}	
	where $\lambda$ is the $k$-th maximal correlation between $X_1$ and $X_2$. 
\end{corollary}
%Notice $\tilde \epsilon_{\tCI}\in [0,1]$. Therefore as long as $\tilde \epsilon_{\tCI}$ is strictly smaller than 1, and 
When the joint distribution of $X_1,X_2$ is non-degenerate, $\lambda<1$. Therefore when Bayes error is small, the learned representation will yield a good downstream performance. 
%\qi{intuitive way to understand this..}

This corollary and the clustering setting is inspired by Theorem 3.7 in \cite{haochen2021provable}, which showed a similar result for a spectral contrastive loss.
%For a symmetric case, e.g. $X_1$ and $X_2$ are augmentations of $X$, \cite{haochen2021provable} discusses how the term $1 - \lambda$ can be lower bounded by a Dirichlet conductance quantity through invoking a higher-order Cheeger's inequality; we refer the reader to Definition 3.3 and Lemma B.4 from that paper for the connection and its interpretation.
Our corollary here shows that non-linear CCA achieves similar guarantees as spectral contrastive loss, without needing any negative samples.
% and when there is no symmetry between $X_1$ and $X_2$.

\begin{remark}
	All the results in this section holds in the same way when replacing $Y$ with the more fine-grained labels $\tilde Y=[Y,Z]$ as discussed in the previous section, and by replacing $k$ by the cardinality of $\tilde Y$. 
\end{remark}

\subsection{Connection to  ACE algorithm and maximal correlation}
%For cleaner presentation, we denote by $\{\psi_i\}_{i=1}^k\in \onb_{1,k}$ (and vice versa $\{\eta_i\}_{i=1}^k \in \onb_{2,k}$) to indicate that $\{\psi_i\}_{i=1}^k$ form an orthogonal basis of $L^2(X_1)$, or namely  $\langle \psi_i, \psi_j \rangle_{L^2(X_1)}=\delta_{i,j} ,\forall i,j\in[k]$ \ns{$\delta_{i=j}$ or something instead?}.
In this section, we review the variational formulation of our problem, and a closer look at the Breiman and Friedman's alternating conditional expectation (ACE) algorithm~\cite{makur2015efficient,breiman1985estimating,buja1990remarks}. 
Recall $L^2(X_1)$ and $L^2(X_2)$ are the square integrable function with respect to the marginal distribution of $X_1$ and $X_2$. 
 We will understand the maximal correlation and the ACE algorithm on the operator $\cT:L^2(X_2)\rightarrow L^2(X_1)$, where $(\cT\circ g)(x_1):= \E[g(X_2)|X_1 = x_1 ]$ for any $g\in L^2(X_2)$. We will show that ACE algorithm decomposes the operator $\cT$ and also implicitly defines the maximal correlation between the two random variables $X_1$ and $X_2$. 

%\begin{proposition}
Due to Courant–Fischer–Weyl min-max principle, the top singular value of $\cT$ can be computed by the variational problem 
$$\max_{\|u\|_{L^2(X_1)} =1 , \|v\|_{L^2(X_2)}=1} \left\{ \langle u, \cT v\rangle \equiv \int  p(x_1,x_2)u(x_1) v(x_2) dx_1 dx_2 \right\}.$$ 

%=\bar f^\top \bar P g$\footnote{We will use $\bar f^\top g:= \langle \bar f, \bar g\rangle_{L_2(R^d)}$}. 
	
The top $k$ singular vectors of $\cT$ can be computed by the variational problem

\begin{align}
\notag 
\{\psi_i\}_{i=1}^k, \{\eta_i\}_{i=1}^k	\leftarrow  & \argmax_{\psi,\eta} \left\{  \sum_{i=1}^k \int \langle \psi_i, \cT \eta_i\rangle \equiv \E_{X_1, X_2} \left[\psi(X_1)^\top \eta(X_2)\right] \right\},\\
	\label{eq:k-svd}
&  ~\text{s.t.}~\bSigma_{\psi, \psi} = \bSigma_{\eta, \eta} = I_{k}.
\end{align}

%\ns{$f$ is used here instead of $\psi$, we had earlier used $f$ for $Y | X_1$ predictor, could be confusing.}
%	\label{prop:T-sval-Pbar}
%\end{proposition}
%\begin{proof}
%	$\|\cT\|_2 = \max_{\|g\|_{X_2} =1=\|f\|_{X_1}} \langle f, \cT g\rangle_{X_1}  $. Apply the change of variables, $\bar f(x) = \sqrt{p(x) } f(x)$, and plugin. \jnote{Qi add details and rank k version.}
%\end{proof}

\begin{lemma}
ACE algorithm (Eqn. \eqref{eq:k-svd}) with $k$-dimensional vector-valued functions solves the ($k+1$)-SVD of $\cT$, and the top singular vectors of $\cT$ is always achieved by constant functions $u(x_1)\equiv 1$ and $v(x_2)\equiv 1$. 
%The $k$-th maximal correlation $\rho_k := \E_{X_1,X_2} [\psi_k(X_1)\eta_k(X_2)]$, where  $\{\psi_i\}_{i=1}^k, \{\eta_i\}_{i=1}^k = \argmax_{f\in \onb_{1,k},g\in \onb_{2,k}, \E f=\E g=0} \sum_i \E [ \psi_i(X_1) \eta_i(X_2)]$ is equal to $\sigma_{k+1} (\cT)$, the ($k+1$)-th singular value of $\cT$.

%\jnote{Change this lemma to directly say that we get SVD of T when we solve ACE}.
%\ns{Where is $\sigma$ defined?}
\end{lemma}
\begin{proof}
	Observe that the top singular value $\sigma_1(\cT)$ is achieved by the top singular functions $u_1(x_1) =1\in L^2(X_1)$ and $v_1(x_2) = 1\in L^2(X_2)$. The constraint $\E f(X_1) =0$ corresponds to $ \langle u_1, f \rangle_{X_1}=0$, i.e., $f$ being in the complement subspace of the top left singular vector of $\cT$, and vice versa for $X_2$.  By the Courant-Fischer characterization of singular values, $\rho_1$ is the variational problem corresponding to $\sigma_2 (\cT)$. Similarly, $\psi_k,\eta_k$ are the $(k+1)$-th singular vectors of $\cT$ since they 	
	since $\rho_k = \langle \cT \eta_{k}, \psi_k\rangle $.   %\jnote{Qi add details and rank k.}
\end{proof}
The second proposition shows that the variational form can be solved by the famous ACE algorithm of Breiman and Friedman~\cite{makur2015efficient,breiman1985estimating,buja1990remarks}.
\begin{proposition}
	\label{prop:ace=cca}
	%ACE solves $\min_{\|g\|_{X_2}=1 ,\E g=\E f=0, \E g^2=1=\E f^2 \|f\|_{X_1}=1}\E (g(X_2)- f(X_1))^2$ \jnote{Qi: define shorthand for all the constraints like $\E g=0=\E f, 1=\E g^2=\E f^2$}.  With the change of variable of $\bar f(x_1)= \sqrt{p(x_1)} f(x_1)$ and $\bar g(x_2) = \sqrt{p(x_2)}g(x_2)$, this solves $\max_{\|\bar g \|_2 =1, \|\bar f\|_2=1}  \int \bar f(x_1) \frac{p(x_1,x_2)}{\sqrt{p(x_1)p(x_2)}} \bar g(x_2) dx_1 dx_2$ (for the second singular value)\jnote{clarify that ACE top solution is the second to the eval problem. Alternatively do not subtract top sval/svec from the ACE definition (this seems cleaner).}.
	
	The generalized ACE algorithm solves \eqref{eqn:ACE}, 
and is equivalent to the solution of non-linear CCA as in \eqref{eq:k-svd}.
%\jnote{This equation should go in the beginning of the section so it is clear what the algorihtm is. Then put a proposition afterwards, saying it is computing SVD of T.}
%\ns{I agree with Jason's restructuring suggestions. This minimization problem is the simplest to present and understand and starting out with it makes the most sense. Perhaps presenting an end-to-end result with the solution to this objective like Theorem~\ref{thm:PCR_general} or even Corollary~\ref{coro:PCR_non-zero_maximal_correlation} might be the most intuitive way to present.}
\end{proposition}
\begin{proof}
%\min_{\{\psi_i\}_{i=1}^k\in \onb_{1,k}, \atop \{\eta_i\}_{i=1}^k\in \onb_{2,k}}&
	\begin{align*}
& \E \sum_{i=1}^k(\eta_i(X_2)- \psi_i(X_1))^2\\
= & \int_{x_1,x_2} p(x_1,x_2) \sum_{i=1}^k(\eta_i(x_2) - \psi_i(x_1))^2  \\
= & \sum_{i=1}^k \int_{x_1,x_2} (\eta_i^2(x_2) + \psi_i^2(x_1)) p (x_1,x_2) dx_1dx_2 -2 \sum_{i=1}^k \int_{x_1,x_2} p(x_1,x_2) \eta_i(x_2)\psi_i(x_1) dx_1dx_2 \\
= &\sum_i \left(\E_{X_1}[\psi_i^2(X_1)]+E_{X_2}[\eta_i^2(X_2)]-2\langle \psi_i,\cT \eta_i \rangle \right)\\
= & 2k - 2\sum_{i=1}^k \langle \psi_i,\cT \eta_i \rangle \tag{Due to the orthogonality constraints}.   
	\end{align*}
%\ns{The first step uses orthogonality, worth mentioning.} 

Therefore the solution of ACE is equivalent to that of non-linear CCA. %finds $\argmin_{ \{\psi_i\}_{i=1}^k\in \onb_{1,k}, \atop \{\eta_i\}_{i=1}^k\in \onb_{2,k}} \E \sum_{i=1}^k(\eta_i(X_2)- \psi_i(X_1))^2$ and is equivalent to \\$\argmax_{ \{\psi_i\}_{i=1}^k\in \onb_{1,k}, \atop \{\eta_i\}_{i=1}^k\in \onb_{2,k}} \sum_{i=1}^k \langle \psi_i,\cT \eta_i \rangle $. 

%	The remainder of the proof follows from Proposition \ref{prop:T-sval-Pbar}.
	
%	\jnote{Qi: add rank k proof.}
\end{proof}
In summary, these two propositions show that calculating the SVD of $\cT$ corresponds to conducting the alternating conditional expectation algorithm~\cite{makur2015efficient,breiman1985estimating,buja1990remarks}. %solving functional CCA on the density $p(x_1,x_2)$~\cite{buja}.%, or equivalently PCA on $\bar P$.

Finally, the generalized maximal correlation between $X_1$ and $X_2$ is associated with the singular values of $\cT$. 
\begin{definition}[$k$-th maximal correlation]
\label{def:max_correlation}
	For every $k\geq 1$, we define the $k$-th maximal correlation between $X_1$ and $X_2$ as:
	\begin{align*}
	 \lambda_k =  & \max_{f_i,g_i,i\in[k]} \min_{1\leq i\leq k} \E[f_i(X_1)g_i(X_2)],\\
	 \text{s.t. } &  \bSigma_{f,f} = \bI, \bSigma_{g,g} = \bI, \E[f_i(X_1)] = 0, \E[g_i(X_2)] = 0. 
	\end{align*}
\end{definition}
As shown in Propostion 3 and Theorem 2 of \cite{makur2015efficient}, the $k$-th maximal correlation is the ($k+1$)-th singular value of $\cT$ and therefore can be calculated from the ACE algorithm: $\lambda_k  = \E[\psi_k(X_1)\eta_k(X_2)  ]$ when $\psi,\eta$ solves Eq. \eqref{eqn:ACE}. One can also refer to \cite{makur2015efficient} for more geometric interpretation for the maximal correlation between two random variables.

% !TEX root = SSL_ICLR.tex
\section{Experiments}\label{sec:experiments}
In this section, we empirically verify our claim that SSL performs well when ACI is satisfied.
More details for experiments can be found in Section~\ref{sec:more_exp}, including experiments in the text domain.

% !TEX root = self_supervised_arxiv.tex

%In this section, we conduct empirical studies to verify our theoretical findings. 

%\begin{figure}[!tb]
%\minipage{0.48\textwidth}
%  \includegraphics[width=\linewidth]{figures/sim_reciprocal.png}
%  \caption{MSE of $\psi$ learned with different $n_1$. The MSE scales linearly with $1 / n_2$ as indicated by our analysis. Simulations are repeated 100 times, with the mean shown in solid line and one standard deviation shown in shadow.}\label{fig:exp-sim1}
%\endminipage\hfill
%\minipage{0.48\textwidth}
%  \includegraphics[width=\linewidth]{figures/bow_rand300_sst_clf_.png}
%  \caption{Classification accuracy on SST of baseline $\phi_1(\bx_1)$, i.e. bag-of-words, and learned $\psi(\bx_1)$ for the two settings.}\label{fig:exp-sst}\label{exp-sst}
%\endminipage\hfill
%\vspace{-10pt}
%\end{figure}

%\begin{figure}[!tb]
%    \centering     %%% not \center
%    \hspace{-3pt}
%    \subfigure{\includegraphics[width=0.48\textwidth]{figures/sim_reciprocal.png}}
%    \hspace{-5pt}
%    \subfigure{\includegraphics[width=0.48\textwidth]{figures/bow_rand300_sst_clf_.png}}
%    \hspace{-10pt}
%    \caption{(a) MSE of $\psi$ learned with different $n_1$. The MSE scales linearly with $1 / n_2$ as indicated by our analysis. Simulations are repeated 100 times, with the mean shown in solid line and one standard deviation shown in shadow. (b) Classification accuracy on SST of baseline $\phi_1(\bx_1)$, i.e. bag-of-words, and learned $\psi(\bx_1)$ for the two settings.}\label{fig:exp-sst}
%    \vspace{-15pt}
%\end{figure}

\begin{figure*}[!tb]
   \centering     %%% not \center
   \protect\begin{tabular}{cccc}
 \hspace{-0.4cm}
\includegraphics[height=77pt]{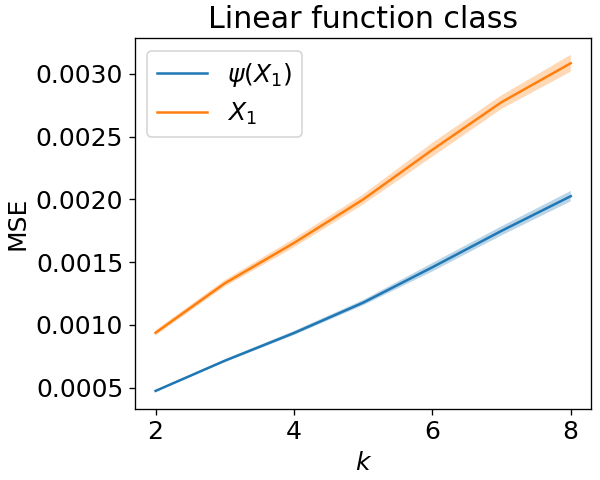} \hspace{-0.5cm} & 
\includegraphics[height=77pt]{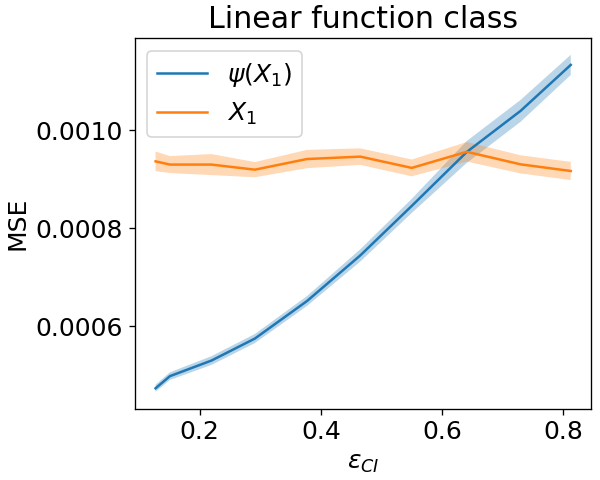}\hspace{-0.4cm} &
\includegraphics[height=77pt]{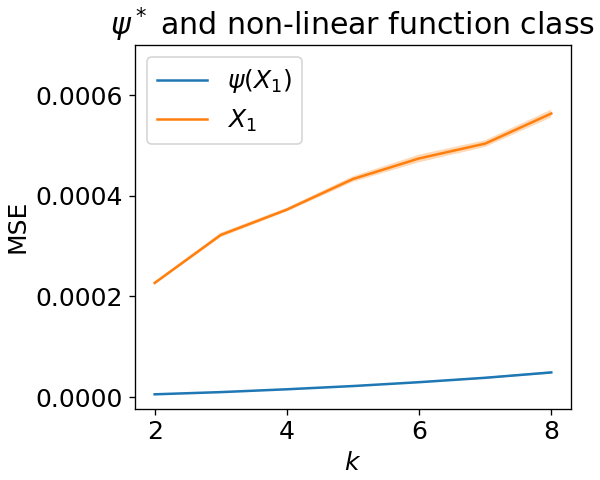}\hspace{-0.5cm} &
\includegraphics[height=77pt]{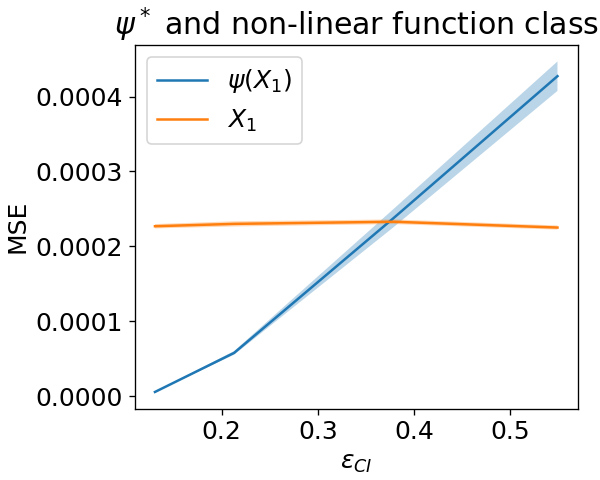}\hspace{-10pt}\\
(a) & (b) & (c) & (d)
   \end{tabular}
   \caption{\textbf{Left two}: how MSE scales with $k$ (the dimension of $Y$) and $\epsilon_{CI}$ (Definition~\ref{def-approx-CI})  with the linear function class. \textbf{Right two}: how MSE scales with $k$  and $\epsilon$ with $\psi^*$ and non-linear function class. Mean of $30$ trials are shown in solid line and one standard error is shown by shadow.} 
   \label{fig:simulation}
\end{figure*}
\paragraph{Simulations. }

With synthetic data, we verify how excess risk (ER) scales with the cardinality/feature dimension of $\cY$ ($k$), and ACI ($\epsilon_{CI}$ in Definition \ref{def-approx-CI}). We consider a mixture of Gaussian data and conduct experiments with both linear function space ($\cH_1$ with $\phi_1$ as identity map) and universal function space $\cH_u$.
%see Theorem \ref{thm:CI_nonlinear_sample_complexity}) and ACI ($\epsilon_{CI}$, see Definition \ref{def-approx-CI}) %\ref{thm:main_result_approximate_CI}).
We sample the label $Y$ uniformly from $\{1, ..., k\}$. For $i$-th class, the centers $\mu_{1i} \in \mathbb{R}^{d_1}$ and $\mu_{2i} \in \mathbb{R}^{d_2}$ are uniformly sampled from $[0,10)$.
Given $Y=i$, $\alpha\in [0,1]$, let $X_1 \sim \mathcal{N}(\mu_{1i}, \mathbf{I})$, $\hat{X_2} \sim \mathcal{N}(\mu_{2i}, \mathbf{I})$, and $X_2=(1-\alpha) \hat{X_2} + \alpha X_1$. Therefore $\alpha$ is a correlation coefficient: $\alpha=0$ ensures $X_2$ being CI with $X_1$ given $Y$ and when $\alpha=1$, $X_2$ fully depends on $X_1$. (if $d_1\neq d_2$, we append zeros or truncate to fit accordingly). 

%We sample the pretext dataset $\{\bx_1^{(i, \rm{pre})}, \bx_2^{(i)} \}_{i=1}^{n_1}$ to learn $\psi$ and sample $\{\bx_1^{(i, \down)}, y_i \}_{i=1}^{n_2}$ to for the downstream linear prediction of $Y$.
% as the metric which is the finite-sample-estimator of the excess risk $\ER$.

We first conduct experiments with linear function class. We learn a linear representation $\psi$ with $n_1$ samples and the linear prediction of $Y$ from $\psi$ with $n_2$ samples. We set $d_1=50$, $d_2=40$, $n_1=4000$, $n_2=1000$ and ER is measured with Mean Squared Error (MSE). As shown in Figure \ref{fig:simulation}(a)(b), the MSE of learning with $\psi(X_1)$ scales linearly with $k$ as indicated in Theorem \ref{thm:CI_linear_sample_complexity}, and scales linearly with $\epsilon_{CI}$ associated with linear function class as indicated in Theorem \ref{thm:main_result_approximate_CI}.
Next we move on to general function class, i.e., $\psi^*=\E[Y|X_1]$ with a closed form solution (see example \ref{example:mixture_gaussian}). We use the same parameter settings as above. For baseline method, we use kernel linear regression to predict $Y$ using $X_1$ (we use RBF kernel which also has universal approximation power). %Our algorithm is still to learn a linear prediction on $\psi^*$ to predict $Y$.
As shown in Figure \ref{fig:simulation}(c)(d), the phenomenon is the same as what we observe in the linear function class setting, and hence they respectively verify Theorem \ref{thm:CI_nonlinear_sample_complexity} and Theorem \ref{thm:main_result_approximate_CI} with $\cH_u$.

\begin{figure}[!tb]
	\centering     %%% not \center
	\hspace{-3pt}
	\subfigure{\includegraphics[height=105pt]{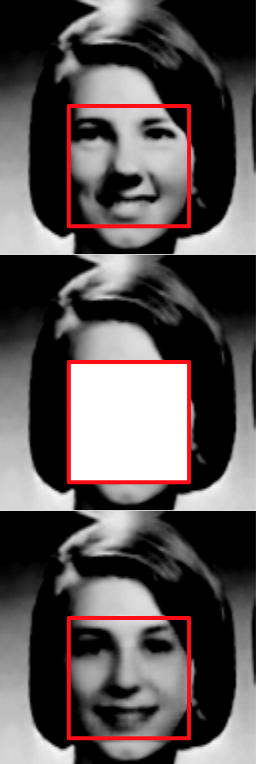}}
	\hspace{5pt}
	\subfigure{\includegraphics[height=120pt]{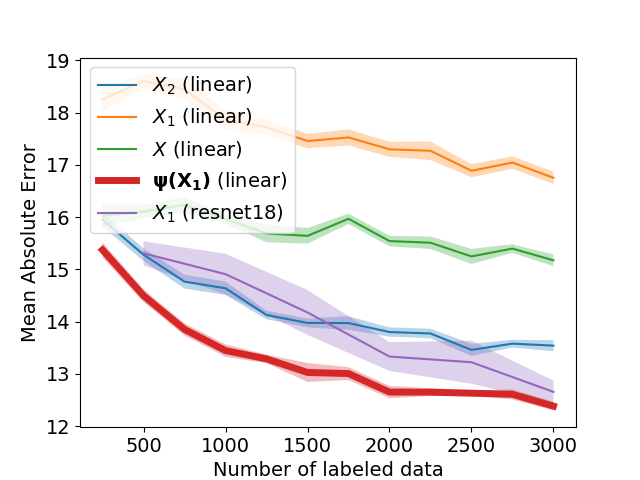}}
	\hspace{-10pt}
	\subfigure{\includegraphics[height=120pt]{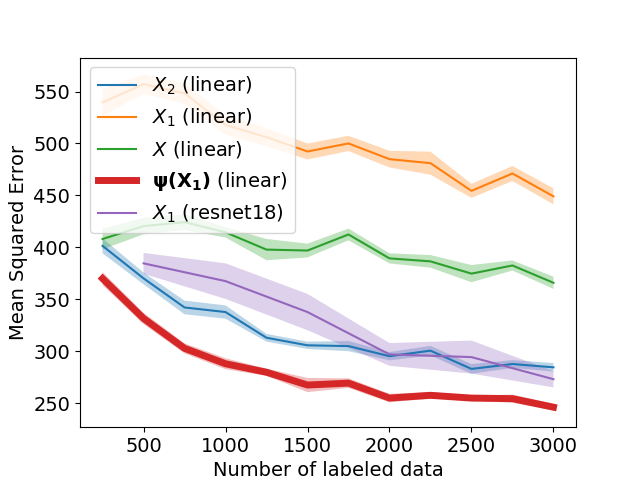}}
	\hspace{-10pt}
	\caption{\textbf{Left}: Example of the $X_2$ (in the red box of the 1st row), the $X_1$ (out of the red box of the 1st row), the input to the inpainting task (the second row), $\psi(X_1)$ (the 3 row in the red box), and in this example $Y=1967$. \textbf{Middle}: Mean Squared Error comparison of yearbook regression predicting dates. \textbf{Right}: Mean Absolute Error comparison of yearbook regression predicting dates. Experiments are repeated 10 times, with mean shown in solid line and one standard deviation in shadow.}
	\label{fig:yb-appendix}
\end{figure}

\paragraph{Computer Vision Task. }
We verify if learning from $\psi$ is more effective than learning directly from $X_1$, in a realistic setting (without enforcing conditional independence). Specifically, we test on the Yearbook dataset \citep{ginosar2015century}, and try to predict the date when the portraits are taken (denoted as $Y_D$), which ranges from $1905$ to $2013$.
We resize all the portraits to be $128$ by $128$. We crop out the center $64$ by $64$ pixels (the face), and treat it as $X_2$, and treat the outer rim as $X_1$ as shown in Figure \ref{fig:yb-appendix}. 
%We crop out the face as $X_2$, and treat the outer rim as $X_1$ as shown on the right (see Appendix for details).
Our task is to predict $Y_D$, which is the year when the portraits are taken, and the year ranges from $1905$ to $2013$. 
For $\psi$, we learn $X_2$ from $X_1$ with standard image inpainting techniques \citep{pathak2016context}, and full set of training data (without labels). 
After that we fix the learned $\psi$ and learn a linear model to predict $Y_D$ from $\psi$ using a smaller set of data (with labels). 
Besides linear model on $X_1$, another strong baseline that we compare with is using ResNet18 \citep{he2016deep} to predict $Y_D$ from $X_1$. With the full set of training data, this model is able to achieve a Mean Absolute Difference of $6.89$, close to what state-of-the-art can achieve \citep{ginosar2015century}. ResNet18 has similar amount of parameters as our generator, and hence roughly in the same function class.
We show the MSE result as in Figure \ref{fig:yb-appendix}. Learning from $\psi$ is more effective than learning from $X_1$ or $X_2$ directly, with linear model as well as with ResNet18. Practitioner usually fine-tune $\psi$ with the downstream task, which leads to more competitive performance \citep{pathak2016context}.

% !TEX root = self_supervised_arxiv.tex

\section{Conclusion}
In this work we theoretically quantify how an approximate conditional independence assumption that connects pretext and downstream task data  distributions can give sample complexity benefits of self-supervised learning on downstream tasks.
Our theoretical findings are also supported by experiments on simulated data and also on real CV and NLP tasks.
We would like to note that approximate CI is only a sufficient condition for a useful pretext task.
We leave it for future work to investigate other mechanisms by which pretext tasks help with downstream tasks.

\bibliographystyle{alpha}
\bibliography{ref}

%%%%%%%%%%%%%%%%%%%%%%%%%%%%%%%%%%%%%%%%%%%%%%%%%%%%%%%%%%%%

%%%%%%%%%%%%%%%%%%%%%%%%%%%%%%%%%%%%%%%%%%%%%%%%%%%%%%%%%%%%

\appendix

% !TEX root = gaussian_model.tex

\newpage
\section{Some Useful Facts}

\subsection{Relation of Inverse Covariance Matrix and Partial Correlation}
For a covariance matrix of joint distribution for variables $X,Y$, the covariance matrix is
\begin{align*}
\begin{bmatrix}
\bSigma_{XX} & \bSigma_{XY} \\
\bSigma_{YX} & \bSigma_{YY} 
\end{bmatrix}= \begin{bmatrix}
\bSigma_{X_1X_1} & \bSigma_{X_1X_2}  &  \bSigma_{X_1Y} \\
\bSigma_{X_2X_1} & \bSigma_{X_2X_2} &\bSigma_{X_2Y}\\
\bSigma_{YX_1} & \bSigma_{X_2Y} & \bSigma_{YY} 
\end{bmatrix}.
\end{align*}
Its inverse matrix $\bSigma^{-1}$ satisfies 
\begin{align*}
\bSigma^{-1} = 
\begin{bmatrix}
\bA & \rho \\
\rho^\top & \bB 
\end{bmatrix}.
\end{align*}
Here $\bA^{-1} = \bSigma_{XX}-\bSigma_{XY}\bSigma_{YY}^{-1}\bSigma_{YX}\equiv \cov(X-\E^L[X|Y], X-\E^L[X|Y]):=\bSigma_{XX\cdot Y} $, the partial covariance matrix of $X$ given $Y$. 

\subsection{Relation to Conditional Independence} 
\begin{proof}[Proof of Lemma \ref{lemma:CI_vs_cross_covariance}]
	
\begin{fact}
	\label{fact:CI_vs_partial_covariance}
When $X_1\bot X_2|Y$, the partial covariance between $X_1,X_2$ given $Y$ is $0$:
\begin{align*}
 \bSigma_{X_1X_2\cdot Y}:= & \cov(X_1-\E^L[X_1|Y], X_2-\E^L[X_2|Y])\\
\equiv  & \bSigma_{X_1X_2}-\bSigma_{X_1Y}\bSigma_{YY}^{-1}\bSigma_{YX_2}  = 0. 
\end{align*}

\end{fact}

The derivation comes from the following: 

\begin{lemma}[Conditional independence (Adapted from \cite{huang2010testing})]
	For random variables $X_1,X_2$ and a random variable $Y$ with finite values, conditional independence $X_1\bot X_2|Y$ is equivalent to:
	\begin{equation}
	\sup_{f\in N_1,g\in N_2}\E[f(X_1)g(X_2)|Y]=0. 
	\end{equation}	
	Here $N_i=\{f:\R^{d_i}\rightarrow R:E [f(X_i)|Y]=0\}$, $i=1,2$.     
\end{lemma}

Notice for arbitrary function $f$, $\E[f(X)|Y] = \E^L[f(X)|\phi_y(Y)]$ with one-hot encoding of discrete variable $Y$. 
Therefore for any feature map we can also get that conditional independence ensures:
\begin{align*}
\bSigma_{\phi_1(X_1)\phi_2(X_2)| Y}:= &  \cov(\phi_1(X_1)-\E^L[\phi_1(X_1)|\phi_y(Y)], \phi_2(X_2)-\E^L[\phi_2(X_2)|\phi_y(Y)])\\
= & \E[\bar \phi_1(X_1) \bar \phi_2(X_2)^\top ] = 0.
\end{align*}
Here $\bar \phi_1(X_1) = \phi_1(X_1) - \E[\phi_1(X_1)|\phi_y(Y)]$ is mean zero given $Y$, and vice versa for $\bar \phi_2(X_2)$. 
This thus finishes the proof for Lemma \ref{lemma:CI_vs_cross_covariance}. 
\end{proof}

\subsection{Technical Facts for Matrix Concentration}
We include this covariance concentration result that is adapted from Claim A.2 in \cite{du2020few}:
\begin{claim}[covariance concentration for gaussian variables] \label{claim:concentration_gaussian_covariance}
Let $\bX=[\bx_1,\bx_2,\cdots \bx_n]^\top \in \R^{n\times d}$ where each $x_i\sim \cN(0,\bSigma_X)$. 
Suppose $n \gg k+\log(1/\delta)$ for $\delta\in(0,1)$.
	Then for any given matrix $B \in \R^{d\times m}$ that is of rank $k$ and is independent of $\bX$, with probability at least $1-\frac{\delta}{10}$ over $\bX$  we have
	\begin{equation} \label{eqn:concentration_gaussian_covariance}
	0.9 \bB^\top \bSigma_X \bB \preceq \frac{1}{n} \bB^\top \bX^\top \bX \bB
	 \preceq 1.1 \bB^\top \bSigma_X \bB.
	\end{equation}
\end{claim}
And we will also use Claim A.2 from \cite{du2020few} for concentrating subgaussian random variable.

\begin{claim}[covariance concentration for subgaussian variables] \label{claim:concentration_subgaussian_covariance}
	Let $\bX=[\bx_1,\bx_2,\cdots \bx_n]^\top \in \R^{n\times d}$ where each $\bx_i$ is $\rho^2$-sub-gaussian. 
	Suppose $n \gg \rho^4(k+\log(1/\delta))$ for $\delta\in(0,1)$.
	Then for any given matrix $B \in \R^{d\times m}$ that is of rank $k$ and is independent of $\bX$, with probability at least $1-\frac{\delta}{10}$ over $\bX$  we have
	\begin{equation} \label{eqn:concentration_subgaussian_covariance}
	0.9 \bB^\top \bSigma_X \bB \preceq \frac{1}{n} \bB^\top \bX^\top \bX \bB
\preceq 1.1 \bB^\top \bSigma_X \bB.
	\end{equation}
\end{claim}

\begin{claim}
	\label{claim:bound_projected_gaussian}
Let $\bZ\in \R^{n\times k}$ be a matrix with row vectors sampled from i.i.d Gaussian distribution $\cN(0,\bSigma_Z)$. Let $\bP\in \R^{n\times n}$ be a fixed projection onto a space of dimension $d$. Then with a fixed $\delta\in (0,1)$, we have:
\begin{equation*}
\|\bP\bZ\|_F^2\lesssim \Trace(\bSigma_Z)(d + \log (k/\delta)), 
\end{equation*}	
with probability at least $1-\delta$.
\end{claim}
\begin{proof}[Claim  \ref{claim:bound_projected_gaussian}]
	Each $t$-th column of $Z$ is an $n$-dim vector that is i.i.d sampled from Gaussian distribution $\cN(0,\bSigma_{tt})$.  
	\begin{align*}
	\|\bP\bZ\|_F^2= &\sum_{t=1}^k \|\bP\bz_t\|^2\\
	 = & \sum_{t=1}^k \bz_t^\top \bP \bz_t.
	\end{align*}
Each term satisfy $\bSigma_{kk}^{-1} \|\bP\bz_t\|^2 \sim \chi^2(d) $, and therefore with probability at least $1-\delta'$ over $\bz_t$,
\begin{equation*}
\bSigma_{kk}^{-1} \|\bP \bz_t\|^2 \lesssim d + \log(1/\delta'). 
\end{equation*}
Using union bound, take $\delta'=\delta/k$ and summing over $t\in [k]$ we get:   
\begin{align*}
\|\bP\bZ\|_F^2 \lesssim \Trace(\bSigma_Z)(d + \log (k/\delta)).
\end{align*}
\end{proof}

\begin{comment} 
\begin{theorem}[Hanson-Wright Inequality (Theorem 1.1 from \cite{rudelson2013hanson})] 
	Let $X=(X_1,X_2,\cdots X_n)\in \R^n$ be a random vector with independent components $X_1$ which satisfy $\E[X_i]=0$ and $\|X_i\|_{\psi_2}\leq K$. Let $A$ be an $n\times n$ matrix. Then, for every $t\geq 0$,
	\begin{align*}
	\prob\left[|X^\top AX-\E[X^\top AX]|>t\right]\leq 2\exp\left\{-c\min\left(\frac{t^2}{K^4\|A\|_F^2},\frac{t}{K^2\|A\|} \right)\right\}.
	\end{align*}
\end{theorem}
\end{comment} 

\begin{theorem}[Vector Bernstein Inequality (Theorem 12 in \cite{gross2011recovering})] 
Let $X_1,\cdots, X_m$
be independent zero-mean vector-valued random variables. Let
$$ N = \| \sum_{i=1}^m  X_i\|_2.$$ 
Then 
\begin{align*}
\prob[ N\geq \sqrt{V}+t] \leq \exp\left(\frac{-t^2}{4V} \right),
\end{align*} 
where $V= \sum_i\E\|X_i \|_2^2$ and $t\leq V/(\max\|X_i\|_2)$.
\end{theorem}

\begin{lemma}
	\label{lemma:bound_projected_subgaussian}
	Let  $\bZ\in \R^{n\times d}$ be a matrix whose row vectors are $n$ independent mean-zero (conditional on $\bP$ being a rank-$d$ projection matrix) $\sigma$-sub-Gaussian random vectors.  With probability $1-\delta$:
		 \begin{align*}
		 \|\bP\bZ\|^2_F \lesssim \sigma^2 (d+\log(d/\delta)). 
		 \end{align*}
\end{lemma}
\begin{proof}[Proof of Lemma \ref{lemma:bound_projected_subgaussian}]
%First we note from Theorem 2.1 of \cite{buldygin2013sub} the subgaussian norm for each $z_{ti}$ that follows Bernouli-$p_{ti}$ distribution satisfy $\|\bz_{ti}\|_{\psi_2}^2\leq K=\frac{p_{ti}-(1-p_{ti})}{2(\ln p_{ti}-\ln q_{ti})}\leq 1/4$, where $q_{ti}=1-p_{ti}$. 
Write $\bP=\bU\bU^\top=[\bu_1,\cdots, \bu_d]$ where $\bU$ is orthogonal matrix in $\R^{n\times d}$ where $\bU^\top \bU=I$. Notice $\|\bU\bU^\top\bZ\|_F^2=\Trace(\bZ^\top UU^\top UU^\top \bZ) = \Trace(\bZ^\top UU^\top \bZ)$. Therefore:
\begin{align*}
\|\bP\bZ\|_F^2= & \|\bU^\top \bZ\|_F^2\\ = & \sum_{j=1}^d \|\bu_j^\top \bZ\|^2\\
= & \sum_{j=1}^d \|\sum_{i=1}^n  \bu_{ji} \bz_i\|^2,
\end{align*}
where each $\bz_i\in \R^k$ being the $i$-th row of $\bZ$ is a centered independent $\sigma$ sub-Gaussian random vectors.  
To use vector Bernstein inequality, we let $X:= \sum_{i=1}^n X_i$ with $X_i$ taking the value of $\bu_{ji} \bz_i$. We have $X_i$ is zero mean: $\E[X_i] = \E[\bu_{ji} \E[\bz_i|\bu_{ji}]]=\E[\bu_{ji}\cdot 0]=0$. 
\begin{align*}
V:= & \sum_i \E\|X_i \|_2^2\\
= & \sum_i \E[\bu_{ji}^2 \bz_i^\top \bz_i ]\\
= &  \sum_i \E_{\bu_{ji}} [\bu_{ji}^2 \E[\|\bz_i\|^2_2 | \bu_{ji}]]\\
\leq & \sigma^2\sum_i \E_{\bu_{ji}} [\bu_{ji}^2]\\
= &\sigma^2.
\end{align*}
Therefore by vector Bernstein Inequality, with probability at least $1-\delta/d$, $\|X\|\leq \sigma(1+\sqrt{\log (d/\delta)})$. Then by taking union bound, we get that $ \|\bP\bZ\|^2 = \sum_{j=1}^d\|\bu_j^\top \bZ\|^2\lesssim \sigma^2 d(1+\log(d/\delta))$ with probability $1-\delta$. 

\end{proof}

\begin{comment} 
\begin{corollary}
	\label{coro:multinomial}
	Let  $\bZ\in \R^{n\times k}$ be a matrix whose row vectors are $n$ independent samples
from centered (conditioned on $\bP$) multinomial probabilities $(p_1,p_2,\cdots p_k)$ (where $p_t$ could be different across each row). Let $\bP\in \R^{n\times n}$ be a projection onto a space of dimension $d$ (that might be dependent with $\bZ$). Then we have 
\begin{align*}
\|\bP\bZ\|^2 \lesssim d(1+\log(d/\delta)). 
\end{align*}
with probability $1-\delta$.
\end{corollary} 
\end{comment} 

\section{Warm-up: jointly Gaussian variables}
\label{sec:joint_gaussian}
We assume $X_1,X_2,Y$ are jointly Gaussian, and so the optimal regression functions are all linear, i.e., $\E[Y|X_1]=\E^L[Y|X_1]$.
We also assume data is centered: $\E[X_i]=0$ and $\E[Y]=0$.
Non-centered data can easily be handled by learning an intercept.
All relationships between random variables can then be captured by the (partial) covariance matrix. Therefore it is easy to quantify the CI property and establish the necessary and sufficient conditions that make $X_2$ a reasonable pretext task.  

\begin{assumption}[Jointly Gaussian]
	\label{assumption:jointly_gaussian}
	$X_1,X_2,Y$ are jointly Gaussian.
\end{assumption}

\begin{assumption}[Conditional independence]
	\label{assumption:conditional_independence}
	$X_1\bot X_2|Y$.
\end{assumption}

\begin{claim}[Closed-form solution]
	\label{claim:linear_conditional_expectation}
	Under Assumption \ref{assumption:jointly_gaussian}, the representation function and optimal prediction that minimize the population risk can be expressed as follows:
	\begin{align}
	\psi^*(\bx_1):=\E^L[X_2|X_1=\bx_1] &= \bSigma_{X_2X_1}\bSigma_{X_1X_1}^{-1}\bx_1  \\
	\text{Our target }f^*(\bx_1):=\E^L[Y|X_1=\bx_1] & = \bSigma_{YX_1}\bSigma_{X_1X_1}^{-1}\bx_1.
	\end{align}
	%\revise{call this $\psi^*$ for consistency?}
\end{claim}

Our prediction for downstream task with representation $\psi^*$ will be:
$g(\cdot) :=  \E^L[Y|\psi^*(X_1)] $. 
Recall from Equation~\eqref{eqn:partial_cov} that the partial covariance matrix between $X_1$ and $X_2$ given $Y$ is $\bSigma_{X_1X_2| Y} \equiv \bSigma_{X_1X_2}-\bSigma_{X_1Y}\bSigma_{YY}^{-1}\bSigma_{YX_2}.$
This partial covariance matrix captures the correlation between $X_1$ and $X_2$ given $Y$. For jointly Gaussian random variables, CI is equivalent to $ \bSigma_{X_1X_2| Y} =0 $. %Meanwhile when $\bSigma_{X_1X_2| Y}$ is of small norm, it means given $Y$, $X_1$, and $X_2$ has little correlation. When $\bSigma_{X_1X_2| Y} $ is low rank, it means there exists some additional latent variable $Z$ such that $ \cov(X_1, X_2|\bar Y)=0,$ where $\bar Y = [Y,Z]$. 
We first analyze the approximation error based on the property of this partial covariance matrix.

\begin{lemma}[Approximation error]
	\label{lemma:linear_CI} 
	Under Assumption \ref{assumption:jointly_gaussian}, \ref{assumption:conditional_independence}, if $\bSigma_{X_2Y}$ has rank $k$, we have $f^*(\bx_1)\equiv \bW^*\psi^*(\bx_1)$, i.e., $e_{\apx}(\psi^*)=0$.
\end{lemma}

\begin{remark}
	$\bSigma_{X_2Y} $ being full column rank implies that $\E[X_2|Y]$ has rank $k$, i.e., $X_2$ depends on all directions of $Y$ and thus captures all directions of information of $Y$. This is a necessary assumption for $X_2$ to be a reasonable pretext task for predicting $Y$.
	$e_{\apx}(\psi^*)=0$ means $f^*$ is linear in $\psi^*$. Therefore $\psi^*$ selects $d_2$ out of $d_1$ features that are sufficient to predict $Y$. 
\end{remark}

Next we consider the estimation error that characterizes the number of samples needed to learn a prediction function $f(\bx_1)=\hat{\bW}\psi^*(\bx_1)$ that generalizes. 
%With finite samples $\bX_1,\bY$ with each sample $(\bx_1^{(i)},\by^{(i)}) \sim P_{X_1Y}$, we do linear regression on top of the learned representation $\psi^*$:
%\begin{equation*}
%\hat \bW\leftarrow \argmin_{\bW} \frac{1}{2n_2}\|\bY-\psi^*(\bX_1)\bW\|_F^2, 
%\end{equation*}
%We are interested in the excess risk that measures generalization:
%\revise{Define $\ER$ concretely so that it can be used and referred everywhere.}
%\begin{equation*}
%\ER_{\psi^*}(\hat \bW) := \E \|f^*(X_1) - \hat\bW^\top \psi^*(X_1) \|_2^2.
%\end{equation*}

\begin{theorem}[Excess risk]
	\label{thm:linear_CI} 
	Fix a failure probability $\delta\in (0,1)$. Under Assumption \ref{assumption:jointly_gaussian},\ref{assumption:conditional_independence}, if $n_2\gg k+\log(1/\delta)$, excess risk of the learned predictor $\bx_1\rightarrow \hat \bW\psi^*(\bx_1)$ on the target task satisfies
	\begin{equation*}
	\ER_{\psi^*}(\hat \bW) \leq  \cO \left(\frac{\Trace(\bSigma_{YY|X_1})(k  + \log (k/\delta))}{n_2}\right),  
	\end{equation*} 
	with probability at least $1-\delta$. %Here $\sigma^2=\Trace(\bSigma_{YY|X_1})$.
\end{theorem}
Here $\bSigma_{YY|X_1}\equiv \bSigma_{YY}-\bSigma_{YX_1}\bSigma_{X_1X_1}^{-1}\bSigma_{X_1Y}$ captures the noise level and is the covariance matrix of the residual term $Y-f^*(X_1)=Y-\bSigma_{YX_1}\bSigma_{X_1X_1}^{-1}X_1$. Compared to directly using $X_1$ to predict $Y$, self-supervised learning reduces the sample complexity from $\tilde \cO(d_1)$ to $\tilde \cO(k)$.
We generalize these results even when only a weaker form of CI holds. 
\begin{assumption}[Conditional independence given latent variables]
	\label{assumption:CI_with_latent_variables}
	There exists some latent variable $Z\in \R^m$ such that $X_1\bot X_2|\bar Y $, and $\bSigma_{X_2\bar Y}$ is of rank $k+m$, where $\bar Y = [Y,Z]$. 
	%$\sigma_k(P_{\bSigma_{X_1Y}}\bSigma_{X_1X_2})\geq C, $ for some positive $C$. Here the operator $P_A = A(A^\top A)^\dagger A^\top$ is the projection onto the column space of matrix $A$.  
\end{assumption}
This assumption lets introduce some reasonable latent variables that capture the information between $X_1$ and $X_2$ apart from $Y$.
$\bSigma_{X_2\bar Y}$ being full rank says that all directions of $\bar Y$ are needed to predict $X_2$, and therefore $Z$ is not redundant. For instance, when $Z=X_1$, the assumption is trivially true but $Z$ is not the minimal latent information we want to add.
Note it implicitly requires $d_2\geq k+m$.
\begin{corollary}
	\label{coro:linear_CI_with_latent_variable} 
	Under Assumption \ref{assumption:jointly_gaussian}, \ref{assumption:CI_with_latent_variables}, we have $f^*(\bx_1)\equiv \bW^*\psi^*(\bx_1)$, i.e.,  the approximation error $e_{\apx}(\psi^*)$ is 0. We can also generalize Theorem~\ref{thm:linear_CI} by replacing $k$ by $k+m$.
\end{corollary}
%Under CI with latent variable, 

\section{Omitted Proofs with Conditional Independence}
\begin{proof}[Proof of Lemma \ref{lemma:linear_CI}]
		$$\cov(X_1|Y,X_2|Y) = \bSigma_{X_1X_2}-\bSigma_{X_1Y}\bSigma_{YY}^{-1}\bSigma_{YX_2}=0.$$ By plugging it into the expression of $\E^L[X_2|X_1] $, we get that 
	\begin{align*}
	\psi(x_1):=\E^L[X_2|X_1=x_1] &= \bSigma_{X_2X_1}\bSigma_{X_1X_1}^{-1}x_1\\
	& = \bSigma_{X_2Y}\bSigma_{YY}^{-1}\bSigma_{YX_1} \bSigma_{X_1X_1}^{-1}x_1\\
	= & \bSigma_{X_2Y}\bSigma_{YY}^{-1} \E^L[Y|X_1]. 
	\end{align*}
Therefore, as long as $\bSigma_{X_2Y}$ is rank $k$, it has left inverse matrix and we get:
$\E^L[Y|X_1=x_1] = \bSigma_{X_2Y}^\dagger \bSigma_{YY}  \psi(x_1)$. Therefore there's no approximation error in using $\psi$ to predict $Y$.  

\end{proof}

\begin{proof}[Proof of Corollary \ref{coro:linear_CI_with_latent_variable}]
	Let selector operator $\bS_y$ be the mapping such that $\bS_y\bar Y = Y$, we overload it as the matrix that ensure $\bS_y \bSigma_{\bar YX}=\bSigma_{YX} $ for any random variable $X$ as well. 
	
	From Lemma \ref{lemma:linear_CI} we get that there exists $W$ such that $ \E^L[\bar Y|X_1] = \bW \E^L[X_2|X_1] $, just plugging in $\bS_y$ we get that $\E^L[Y|X_1] = (\bS_y\bW) \E^L[X_2|X_1]$. 
	
	\begin{comment} 	
	Recall $\psi(x_1) = B^\top x_1$, where $B = \bSigma_{X_1X_1}^{-1} \bSigma_{X_1X_2} $, and $f^*(x_1) \equiv  \E[Y|X_1=x_1] = (\Theta^*)^\top x_1$, where $\Theta^* := \bSigma_{X_1X_1}^{-1}\bSigma_{X_1Y} $. Since sp$(\bSigma_{X_1Y}) \subseteq $sp$(\bSigma_{X_1X_2})$ implies  $e_\apx=0 $, we show this.
	
	Now notice $\bSigma_{X_1X_2} = \bSigma_{X_1X_2|Y} + \bSigma_{X_1Y}\bSigma_{YY}^{-1}\bSigma_{YX_2}$. Define $V := P_{\bSigma_{X_1Y}}\bSigma_{X_1 X_2}$. Assumption \ref{assumption:aligned_X1Y_direction} ensures $\sigma_k (V)>C>0$, so $V$ is rank $k$.

	Using $\bSigma_{X_1X_2} = V + P^\perp_{\bSigma_{X_1Y}} \bSigma_{X_1X_2}= V + P^\perp_{\bSigma_{X_1Y}} \bSigma_{X_1X_2|Y} $. Using Claim \ref{remark:low_rank_partial_covariance} $\rank(V) + \rank(P^\perp_{\bSigma_{X_1Y}} \bSigma_{X_1X_2|Y}) \le k+m\le d_2$. Further $V$ and $P^\perp_{\bSigma_{X_1Y}} \bSigma_{X_1X_2|Y}$ have orthogonal column spaces, so  $\linspan(V) \subset \linspan(\bSigma_{X_1,X_2})$.

	Since $V$ is rank $k$ and obviously sp$(\Theta^*)=$sp$(\bSigma_{X_1X_1}^{-1} V)$, and recall that $B =\bSigma_{X_1X_1}^{-1}\bSigma_{X_1X_2} $. We conclude that sp$(\Theta^*)\subseteq$ sp$(B)$. 	
	
	%Notice when $\sigma_k(P_{\bSigma_{X_1Y}\bSigma_{X_1X_2} )\geq C>0$, sp$(\bSigma_{X_1Y})\subseteq $sp$(\bSigma_{X_1X_2})$. 
	
	\end{comment} 
\end{proof}

\begin{proof}[Proof of Theorem \ref{thm:linear_CI}]
Write $f^*(X_1) = \E[Y|X_1] = (\bA^*)^\top X_1$. 
$\E^L[Y|X_1=x_1] = \bSigma_{X_2Y}^\dagger \bSigma_{YY}  \psi(x_1) $.  Let $\bW^* = \bSigma_{YY} \bSigma_{YX_2}^\dagger $. From Lemma \ref{lemma:linear_CI} we know $f^*=\bW^*\psi$. Recall noise $N = Y-f^*(X_1)$ is mean zero conditional on $X_1$. We write $\bN = \bY-f^*(\bX_1)$.

First we have the basic inequality,
	\begin{align*}
	\frac{1}{2n_2}\|\bY - \psi(\bX_1)\hat \bW\|_F^2 \leq  & \frac{1}{2n_2}\|\bY - \bX_1 A^* \|_F^2 \\
	= & \frac{1}{2n_2}\|\bY - \psi(\bX_1) \bW^* \|_F^2 = \frac{1}{2n_2}\|\bN\|_F^2.
	\end{align*}
	Therefore by rearranging both sides, we have:
	\begin{align*}
	\|\psi(\bX_1) \bW^* - \psi(\bX_1)\hat \bW \|^2 \leq &  2\langle \bN, \psi(\bX_1) \bW^* - \psi(\bX_1)\hat \bW \rangle \\
	= & 2\langle P_{\psi(\bX_1)} \bN, \psi(\bX_1) \bW^* - \psi(\bX_1)\hat W\rangle \\
	\leq & 2  \|P_{\psi(\bX_1)} \bN \|_F \|\psi(\bX_1) \bW^* - \psi(\bX_1)\hat W\|_F \\
	\Rightarrow \|\psi(\bX_1) \bW^* - \psi(\bX_1)\hat \bW \| \leq & 2 \|P_{\psi(\bX_1)} \bN\|_F\\
	\lesssim & \sqrt{ \Trace(\bSigma_{YY|X_1})(k  + \log k/\delta)}. \tag{from Claim \ref{claim:bound_projected_gaussian}}
	\end{align*} 
The last inequality is derived from Claim \ref{claim:bound_projected_gaussian} and the fact that each row of $\bN$ follows gaussian distribution $\cN(0,\bSigma_{YY|X_1})$.	
Therefore 
	$$ \frac{1}{n_2} \|\psi(\bX_1) W^* - \psi(\bX_1)\hat W \|_F^2 \lesssim \frac{\Trace(\bSigma_{YY|X_1})(k  + \log k/\delta)}{n_2} .  $$

Next we need to concentrate $1/n\bX_1^\top \bX_1$ to $\bSigma_X$. 
Suppose $\E^L[X_2|X_1]=\bB^\top X_1$, i.e., $\psi(x_1)=\bB^\top x_1$, and $\psi(\bX_1)=\bX_1 \bB$.
With Claim \ref{claim:concentration_gaussian_covariance} we have $ 1/n\psi(\bX_1)^\top \psi(\bX_1) = 1/n \bB^\top \bX_1^\top \bX_1 \bB $ satisfies:
\begin{equation*}
0.9\bB^\top \bSigma_X \bB \preceq 1/n_2\psi(\bX_1)^\top \psi(\bX_1) \preceq 1.1 \bB^\top \bSigma_X \bB 
\end{equation*}
Therefore we also have:
	\begin{align*}
&	 \E[\|(\bW^*-\hat \bW)^\top \psi(x_1)\|^2] \\
= &	\|\bSigma_X^{1/2} \B (\bW^* - \hat \bW) \|_F^2\\
\leq & \frac{1}{0.9 n_2} \|\psi(\bX_1) \bW^* - \psi(\bX_1)\hat \bW \|_F^2
\lesssim \frac{\Trace(\bSigma_{YY|X_1})(k  + \log k/\delta)}{n_2} .
	\end{align*}

%	\qi{add concentration of $\hat X_1$}
\end{proof}

\subsection{Omitted Proof for General Random Variables}

\begin{proof}[Proof of Lemma \ref{lemma:discrete_case_CI}]
	%This is direct result from the covariance matrix of $\phi_1(X_1),\phi_2(X_2),Y$. 
	Let the representation function $\psi$ be defined as:
	\begin{align*}
	\psi(\cdot):= \E[X_2|X_1 ] = & \E[\E[X_2|X_1,Y]|X_1] \\
	= &\E[\E[X_2|Y]|X_1]\tag{uses CI}\\
	= & \sum_{y}  P(Y=y|X_1)\E[X_2|Y=y]\\
	=: & f(X_1)^\top \bA,
	\end{align*}
	where $f: \R^{d_1}\rightarrow \Delta_{\cY} $ satisfies $f(x_1)_y = P(Y=y|X_1=x_1)$, and $\bA \in \R^{\cY\times d_2} $ satisfies $\bA_{y,:} = \E[X_2|Y=y]$. Here $\Delta_d$ denotes simplex of dimension $d$, which represents the discrete probability density over support of size $d$. 
	
	Let $\bB=\bA^{\dagger}\in \R^{\cY\times d_2} $ be the pseudoinverse of matrix $A$, and we get $\bB\bA=\bI$ from our assumption that $\bA$ is of rank $|\cY|$. Therefore $f(x_1) = \bB \psi(x_1),\forall x_1$. Next we have:
	\begin{align*}
	\E[Y|X_1=x_1] = & \sum_{y}   P(Y=y|X_1=x_1)\times y \\
	= &  \bY f(x_1)  \\
	= & (\bY \bB)\cdot \psi(X_1).
	\end{align*}
	Here we denote by $\bY\in \R^{k\times \cY}, \bY_{:,y}=y$ that spans the whole support $\cY$.
	Therefore let $\bW^*=\bY \bB$ will finish the proof. 
	
\end{proof}

\begin{proof}[Proof of Theorem \ref{thm:CI_nonlinear_sample_complexity}]
With Lemma \ref{lemma:discrete_case_CI} we know $e_\apx=0$, and therefore $\bW^*\psi(X_1) \equiv f^*(X_1)$. Next from basic inequality and the same proof as in Theorem \ref{thm:linear_CI} we have:
\begin{align*}
\|\psi(\bX_1) \bW^* - \psi(\bX_1)\hat \bW \| \leq & 2 \|\bP_{\psi(\bX_1)} \bN\|_F 
\end{align*} 
Notice $\cN$ is a random noise matrix whose row vectors are independent samples from some centered distribution. Note we assumed $\E[\|N\|^2|\bX_1]\leq \sigma^2$. $P_{\psi(\bX_1)}$ is a projection to dimension $k$. From Lemma \ref{lemma:bound_projected_subgaussian} we have: %Corollary \ref{coro:multinomial} we have:
\begin{align*}
\|f^*(\bX_1) - \psi(\bX_1)\hat \bW \| \leq & \sigma\sqrt{k(1+\log k/\delta)} . 
\end{align*} 
%The only difference is that $P_{\Psi}$ is a rank $c$ matrix. 

Next, with Claim \ref{claim:concentration_subgaussian_covariance} we have when $n\gg \rho^4(k+\log(1/\delta))$, since $\bW^*-\hat \bW\in \R^{d_2\times k}$,
\begin{align*}
& 0.9(\bW^*-\hat \bW)^\top \bSigma_{\psi}(\bW^*-\hat \bW)\\
 \preceq & \frac{1}{n_2} (\bW^*-\hat \bW)^\top \sum_i\psi(x_1^{(i)})\psi(x_1^{(i)})^\top  (\bW^*-\hat \bW) \preceq 1.1(\bW^*-\hat \bW)^\top \bSigma_{\psi}(\bW^*-\hat \bW)
\end{align*}
And therefore we could easily conclude that:
\begin{align*}
\E\|\hat\bW^\top \psi(X_1) - f^*(X_1) \|^2
\lesssim & \sigma^2\frac{k(1+\log (k/\delta))}{n_2} . 
\end{align*}

\end{proof}

\subsection{Omitted proof of linear model with approximation error}
\begin{proof}[Proof of Theorem \ref{thm:CI_linear_sample_complexity}]

First we note that $Y=f^*(X_1)+N$, where $\E[N|X_1]=0$ but $Y-(\bA^*)^\top X_1$ is not necessarily mean zero, and this is where additional difficulty lies. Write approximation error term $a(X_1):=f^*(X_1)-(\bA^*)^\top X_1$,  namely $Y=a(X_1)+(\bA^*)^\top X_1+N$. Also, $(\bA^*)^\top X_1 \equiv (\bW^*)^\top\psi(X_1)$ with conditional independence.  

Second, with KKT condition on the training data, we know that $\E[ a(X_1)X_1^\top] =0$. 

Recall $\hat \bW = \argmin_{\bW}\|\bY-\psi(\bX_1) \bW \|^2_F $. We have the basic inequality,
\begin{align*}
\frac{1}{2n_2}\|\bY - \psi(\bX_1)\hat \bW\|_F^2 \leq  & \frac{1}{2n_2}\|\bY - \bX_1 \bA^* \|_F^2 \\
= & \frac{1}{2n_2}\|\bY - \psi(\bX_1) \bW^* \|_F^2.\\
\text{i.e., } \frac{1}{2n_2}\|\psi(\bX_1)\bW^* + a(\bX_1)+\bN - \psi(\bX_1)\hat \bW\|_F^2 \leq & \frac{1}{2n_2}\| a(\bX_1)+\bN\|_F^2.
\end{align*}
Therefore 
\begin{align}
\nonumber 
& \frac{1}{2n_2}\|\psi(\bX_1) \bW^* - \psi(\bX_1)\hat \bW \|^2\\
\notag
 \leq &  -\frac{1}{n_2}\langle a(\bX_1) + \bN, \psi(\bX_1) \bW^* - \psi(\bX_1)\hat \bW \rangle \\
\label{eqn:basic_inequality_linear}
= & - \frac{1}{n_2}\langle a(\bX_1),\psi(\bX_1) \bW^* - \psi(\bX_1)\hat \bW\rangle - \langle \bN, \psi(\bX_1) \bW^* - \psi(\bX_1)\hat \bW\rangle
\end{align}
With Assumption \ref{assumption:bounded_error_finite_dim} and by concentration $0.9 \frac{1}{n_2}\bX_1\bX_1^\top \preceq  \bSigma_{X_1}\preceq 1.1 \frac{1}{n_2}\bX_1\bX_1^\top $, we have 
\begin{align}
\label{eqn:empirical_bounded_apx}
\frac{1}{\sqrt{n_2}}\|a(\bX_1) \bX_1^\top \bSigma_{X_1}^{-1/2}\|_F\leq 1.1b_0\sqrt{k}
\end{align}

Denote $\psi(\bX_1)=\bX_1\bB$, where $\bB = \bSigma_{X_1}^{-1}\bSigma_{X_1X_2}$ is rank $k$ under exact CI since $\bSigma_{X_1X_2}=\bSigma_{X_1Y}\bSigma_{Y}^{-1}\bSigma_{YX_2}$. We have
\begin{align*}
& \frac{1}{n_2}\langle a(\bX_1),\psi(\bX_1) \bW^* - \psi(\bX_1)\hat \bW\rangle\\
 = & \frac{1}{n_2}\langle a(\bX_1), \bX_1 \bB \bW^* - \bX_1 \bB \hat \bW\rangle\\
= & \frac{1}{n_2}\langle \bSigma_{X_1}^{-1/2}\bX_1^\top a(\bX_1),   \bSigma_{X_1}^{1/2}(\bB \bW^* - \bB \hat \bW) \rangle\\
\leq & 1.1b_0 \sqrt{\frac{k}{n_2}} \|\bSigma_{X_1}^{1/2}(\bB \bW^* - \bB \hat \bW)\|_F  \tag{from Ineq. \eqref{eqn:empirical_bounded_apx}}
\end{align*}

Back to Eqn. \eqref{eqn:basic_inequality_linear}, we get 
\begin{align*} 
&\frac{1}{2n_2}\|\psi(\bX_1) \bW^* - \psi(\bX_1)\hat \bW \|^2_F\\
 \lesssim &  \sqrt{\frac{k}{n_2}} \|\bSigma_{X_1}^{1/2}(\bB \bW^* - \bB \hat \bW)\|_F + \frac{1}{n_2}\|P_{\bX_1}\bN\|_F\| \bX_1(\bB \bW^* - \bB \hat \bW)\|_F \\
\lesssim & \left(\frac{\sqrt{k}}{n_2} + \frac{1}{n_2}\|P_{\bX_1}\bN\|_F\right) \| \bX_1(\bB \bW^* - \bB \hat \bW)\|_F\\
\Longrightarrow &  \frac{1}{\sqrt{n_2}}\|\psi(\bX_1) \bW^* - \psi(\bX_1)\hat \bW \|_F \lesssim  
\sqrt{\frac{k(1+\log k/\delta) }{n_2}}. \tag{from Lemma \ref{lemma:bound_projected_subgaussian}} 
\end{align*} 
Finally, by concentration we transfer the result from empirical loss to excess risk and get:
\begin{align*}
\E[\|\psi(X_1) \bW^* - \psi(X_1)\hat \bW \|^2] \lesssim \frac{k(1+\log(k/\delta))}{n_2}.  
\end{align*}

\end{proof}	

\begin{comment} 
\begin{proof}[Proof of Claim \ref{remark:low_rank_partial_covariance}]
	Denote by three random variables $A,B,C$ s.t. $A=X_1|Y,B=X_2|Y, C=Z|Y$. Since $ X_1\bot X_2|Y,Z $ we have $A\bot B|C$. Therefore $ \bSigma_{AB|C} \equiv \bSigma_{AB}-\bSigma_{AC}\bSigma_{CC}^{-1}\bSigma_{CB} = 0 $. 
	\begin{align*}
	  \bSigma_{X_1X_2|Y} = & \bSigma_{AB} \\
	  = & \bSigma_{AC}\bSigma_{CC}^{-1}\bSigma_{CB} \\
	  = & \bSigma_{X_1Z|Y}\bSigma_{ZZ|Y}^{-1}\bSigma_{ZX_2|Y},
	\end{align*}
	which is at most of rank $m$, the dimension of $Z$. 
	 
\end{proof}

\subsection{General Random Variables with Conditional Independence}

\begin{proof}[Proof of Theorem \ref{thm:CI_nonlinear_sample_complexity}]
	Notice $f^*$ is simply $\E[Y|X_1]$ and $f^*$ is given by $(w^*)^\top \psi(\cdot)$ from Lemma \ref{lemma:discrete_case_CI}. Therefore when learning with linear regression we have $e_{\apx}=0$. %Plus the data is rank $1$.  
	Input data becomes $\bar\Phi = [ \psi(x_1^{1})|\psi(x_1^{2})|\cdots |\psi(x_1^{n})]$, where $x_1^{j}$ are i.i.d samples from the data distribution. Then we will easily get:
	\begin{align*}
	ER(\hat{w}) := \E_{X_1}[(\hat{w}^\top \psi(x_1)-f^*(x_1))^2] \leq O(\text{rank}(\bar\Phi)/n) = O(|\cY|/n). 
	\end{align*}
\end{proof}
\end{comment} 

\subsection{Argument on Denoising Auto-encoder or Context Encoder}
\label{sec:auto-encoder} 
\begin{remark}
	\label{remark:encoder_decoder}
	We note that since $X_1\bot X_2|Y$ ensures $X_1\bot h(X_2)|Y$ for any deterministic function $h$, we could replace $X_2$ by $h(X_2)$ and all results hold. Therefore in practice, we could use $h(\psi(X_1))$ instead of $\psi(X_1)$ for downstream task. Specifically with denoising auto-encoder or context encoder, one could think about $h$ as the inverse of decoder $D$ ($h=D^{-1}$) and  use $D^{-1}\psi\equiv E$ the encoder function as the representation for downstream tasks, which is more commonly used in practice. 
\end{remark}

This section explains what we claim in Remark \ref{remark:encoder_decoder}. For context encoder, the reconstruction loss targets to find the encoder $E^*$ and decoder $D^*$ that achieve
\begin{equation}
\label{eqn:encoder_decoder}
\min_{E}\min_D \E\|X_2- D(E(X_1)) \|_F^2, 
\end{equation}
where $X_2$ is the masked part we want to recover and $X_1$ is the remainder.

If we naively apply our theorem we should use $D^*(E^*(\cdot))$ as the representation, while in practice we instead use only the encoder part $E^*(\cdot)$ as the learned representation. We argue that our theory also support this practical usage if we view the problem differently. Consider the pretext task to predict $(D^*)^{-1}(X_2)$ instead of $X_2$ directly, namely, 
\begin{equation}
\label{eqn:encoder_decoder2}
\bar E\leftarrow \argmin_{E} \E\| (D^*)^{-1}(X_2) - E(X_1)  \|^2, 
\end{equation}
and then we should indeed use $E(X_1)$ as the representation.
On one hand, when $X_1\bot X_2|Y$, it also satisfies $X_1\bot (D^*)^{-1}(X_2)|Y$ since $(D^*)^{-1}$ is a deterministic function of $X_2$ and all our theory applies. On the other hand, the optimization on \eqref{eqn:encoder_decoder} or \eqref{eqn:encoder_decoder2}  give us similar result. Let 
$$ E^*=\argmin_{E}\E[\|X_2- D^*(E(X_1))\|^2],  $$
and $\E\|X_2- D^*(E^*(X_1))\|^2\leq \epsilon,$ then with pretext task as in \eqref{eqn:encoder_decoder2}
we have that: %\jnote{Should there be expectation below or is this finite datset?}:
\begin{align*}
\E\|(D^*)^{-1}(X_2)-E^*(X_1)\|^2 = & \E\|(D^*)^{-1}(X_2) - (D^*)^{-1}\circ D^*(E^*(X_1))\|^2 \\
\leq & \|(D^*)^{-1}\|_{\text{Lip}}^2 \E \|X_2 - D^*(E^*(X_1)) \|^2\\
\leq & L^2 \epsilon,
\end{align*}
where $L:=\|(D^*)^{-1}\|_{\text{Lip}}$ is the Lipschitz constant for function $(D^*)^{-1}$. This is to say, in practice, we optimize over \eqref{eqn:encoder_decoder}, and achieves a good representation $E^*(X_1)$ such that $\epsilon_{\pre} \leq L\sqrt{\epsilon}$ and thus performs well for downstream tasks. (Recall $\epsilon_{\pre}$ is defined in Theorem \ref{thm:main_result_approximate_CI} that measures how well we have learned the pretext task.)

\section{Omitted Proofs Beyond Conditional Independence} 

\subsection{Warm-up: Jointly Gaussian Variables}
\label{sec:joint_gaussian_apx_CI}
%Apart from conditional independence merely on $Y$, we introduce some more latent variables which still induce $0$ approximation error:

%\begin{claim}
%	\label{remark:low_rank_partial_covariance}
%	When $ X_1\bot X_2|Y,Z$, where $Z\in \R^{m}$, then we have $\bSigma_{X_1X_2|Y}$ is of rank at most $\min\{d_2,m\}$. 
%\end{claim}

%\jnote{Bunch of matrices aren't bold}
As before, for simplicity we assume all data is centered in this case.
\begin{assumption}[Approximate Conditional Independent Given Latent Variables]
	\label{assumption:linear_approximate_CI}
	Assume there exists some latent variable $Z\in \R^m$ such that $$\|\bSigma_{X_1}^{-1/2}\bSigma_{X_1,X_2|\bar Y}\|_F\leq \epsilon_{\tCI},$$ $\sigma_{k+m}(\bSigma_{Y\bar{Y} }^{\dagger}\bSigma_{\bar{Y}X_2})=\beta>0$ \footnote{$\sigma_k(\bA)$ denotes $k$-th singular value of $\bA$, and $\bA^\dagger$ is the pseudo-inverse of $\bA$.} and $\bSigma_{X_2,\bar Y}$ is of rank $k+m$, where $\bar Y = [Y,Z]$.  
	%$\sigma(P_{\bSigma_{X_1Y}}\bSigma_{X_1X_2})\geq C, $ for some positive $C$. Here the operator $P_A = A(A^\top A)^\dagger A^\top$ is the projection onto the column space of matrix $A$.  
\end{assumption}

%Notice Assumption \ref{assumption:aligned_X1Y_direction} is much weaker than the fact that the column span of $\bSigma_{X_1Y}$ is in that of  $\bSigma_{X_1X_2}$. It could be interpreted as follows. In order to predict each $i$-th coordinate of $Y$, $Y_i|X_1$ is not in the null space of $X_2|X_1$, i.e., the features of $X_1$ that are used to predict $X_2$ can not be completely useless in predicting $Y$.  %that for each $i$-th dimension in $Y$, $\bSigma_{X_1Y_i}$ is correlated with $\bSigma_{X_1X_2},$ i.e., $\bSigma_{X_1Y_i}$ is not in the null space of $\bSigma_{X_1X_2},$ which 

%Remark: final error will be approximation error plus noise that depends on sample size. 

When $X_1$ is not exactly CI of $X_2$ given $Y$ and $Z$, the approximation error depends on the norm of $\|\bSigma_{X_1}^{-1/2}\bSigma_{X_1,X_2|\bar Y}\|_2$. 
%one could conduct PCR to reduce some variance. Let $\hat W$ be obtained from principle component regression, we get: 
Let $\hat \bW$ be the solution from Equation~\ref{eqn:W_hat}.

\begin{theorem}
	\label{thm:linear_general_apx} 
	Under Assumption \ref{assumption:linear_approximate_CI} with constant $\epsilon_{\tCI}$ and $\beta$, then the excess risk satisfies 
	\begin{align*}
	\ER_{\psi^*}[\hat\bW]:=\E[\|\hat\bW^\top \psi^*(X_1) - f^*(X_1)\|^2_F] \lesssim \frac{\epsilon_{\tCI}^2}{\beta^2}+\Trace(\bSigma_{YY|X_1}) \frac{d_2   + \log (d_2/\delta)}{n_2}. 
	\end{align*}
	%$ \ER(\hat W) \lesssim \sigma^2  \order{\sqrt{\frac{\epsilon_{\tCI}}{\beta }} \|\bSigma^{1/2}_{X_1X_1}\Theta^*\|_2 + \sqrt{\frac{d_2}{n_2}} }$. 
	%Here $\Theta^*$ is the optimal solution $ \bSigma_{X_1X_1}^{-1}\bSigma_{X_1Y}$. 
\end{theorem}

\begin{proof}[Proof of Theorem \ref{thm:linear_general_apx}]
	Let $\bV:=f^*(\bX_1)\equiv  \bX_1\bSigma^{-1}_{X_1X_1}\bSigma_{1Y}$ be our target direction. Denote the optimal  representation matrix by $\Psi:=\psi(\bX_1)\equiv \bX_1 \bA$ (where $\bA := \bSigma_{X_1X_1}^{-1} \bSigma_{X_1X_2} $). % concentrates to $ \Psi = \bSigma^{-1/2}_{X_1X_1} \bSigma_{X_1X_2} $. 

Next we will make use of the conditional covariance matrix:
$$ \bSigma_{X_1X_2|\bar Y} := \bSigma_{X_1X_2}-\bSigma_{X_1\bar Y}\bSigma_{\bar Y}^{-1}\bSigma_{\bar Y X_2}, $$
and plug it in into the definition of $\Psi$:
%\jnote{cite conditional covariance formula or wherever this is from.}	
	\begin{align*}
	\Psi = & \bX_1\bSigma^{-1}_{X_1X_1} \bSigma_{X_1\bar Y}\bSigma_{\bar Y}^{-1}\bSigma_{\bar Y X_2} +  \bX_1\bSigma^{-1}_{X_1X_1} \bSigma_{X_1X_2|\bar Y}\\
	=: & \bL + \bE,
	\end{align*}
	where  $\bL := \bX_1\bSigma^{-1}_{X_1X_1} \bSigma_{X_1\bar Y}\bSigma_{\bar Y}^{-1}\bSigma_{\bar Y X_2}$ and $\bE := \bX_1\bSigma^{-1}_{X_1X_1} \bSigma_{X_1X_2|\bar Y}$. We analyze these two terms respectively. 
	
For $\bL$, we note that span$(\bV)\subseteq $span$(\bL)$: $\bL \bSigma^{\dagger}_{X_2\bar Y}\bSigma_{\bar Y} = \bX_1\bSigma^{-1}_{X_1X_1} \bSigma_{X_1\bar Y} $. By right multiplying the selector matrix $S_Y$ we have:  $\bL \bSigma^{\dagger}_{X_2\bar Y}\bSigma_{\bar Y Y} = \bX_1\bSigma^{-1}_{X_1X_1} \bSigma_{X_1 Y} $, i.e., $\bL \bar \bW = \bV$, where $\bar \bW:= \bSigma^{\dagger}_{X_2\bar Y}\bSigma_{\bar Y Y}$. From our assumption that $\sigma_{r}(\bSigma_{\bar Y Y}^{\dagger}\bSigma_{\bar Y X_2})=\beta$, we have $\|\bar \bW\|_2\leq \|\bSigma_{X_2\bar Y}^{\dagger}\bSigma_{\bar Y} \|_2 \leq 1/\beta$. (Or we could directly define $\beta$ as $\sigma_{k} (\bSigma_{Y\bar Y}^{\dagger}\bSigma_{\bar Y X_2} ) \equiv \|\bar \bW\|_2$. ) %\jnote{Why not just directly assume $\|\bar W\|$ is small}

By concentration, we have $\bE = \bX_1\bSigma^{-1}_{X_1X_1} \bSigma_{X_1X_2|\bar Y}$ converges to $\bSigma^{-1/2}_{X_1X_1} \bSigma_{X_1X_2|\bar Y}$. Specifically, when $n\gg k+\log 1/\delta $, $ \|\bE\|_F\leq 1.1 \|\bSigma^{-1/2}_{X_1X_1} \bSigma_{X_1X_2|\bar Y}\|_F \leq 1.1\epsilon_{\tCI}$ (by using Lemma \ref{claim:concentration_gaussian_covariance} ).  
Together we have $\|\bE \bar \bW \|_F\lesssim \epsilon_{\tCI}/\beta$. 	
	%Now $L$ is the summation of a rank $k$ matrix and a rank $m$ matrix which is of dimension at most $m+k$. From our Assumption \ref{assumption:aligned_X1Y_direction} we have that $P_V(\Psi)\geq C$ and  $ \|E\|_2 \leq  \alpha C$ with $\alpha<1$. Therefore $P_V(L)>0$. With similar reasoning as in the proof of Lemma \ref{lemma:linear_low_rank_apx} we have sp$(V)\subseteq$sp$(L)$.  
	
	%We denote the PCA of $\Psi$ as $\Psi_r$. We set $r = m+k$. 
	
%	Notice with our assumption, $\|E\|_F\leq 1.1\epsilon_{\tCI}$ and $\bSigma_{r}(\bSigma_{\bar Y}^{-1}\bSigma_{\bar{Y}X_2})=\beta$. Together we get $\|\bE \bar \bW\|_F\lesssim \epsilon_{\tCI}/\beta$, where $\bar \bW:= \bSigma_{\bar Y X_2}^\dagger \bSigma_{\bar YY}$.  
		Let $\hat \bW = \argmin_{\bW}\|\bY-\Psi \bW \|^2$. We note that $\bY=\bN+\bV= \bN+ \Psi\bar \bW - \bE \bar \bW$ where $\bV$ is our target direction and $\bN$ is random noise (each row of $\bN$ has covariance matrix $\bSigma_{YY|X_1}$).
%	 Notice from our assumption $\|\bar \bW\|_2\leq \|\bSigma_{\bar Y X_2}^\dagger \bSigma_{\bar YY}\|\leq 1/\beta$ \jnote{needs explanation. Seems that $\beta$ should have been defined with $\bSigma_{\bar Y Y}^{\dagger} \bSigma_{\bar Y X_2}$ , not $\bSigma_{\bar Y }^{-1} \bSigma_{\bar Y X_2}$ then?}
	 
From basic inequality, we have:
		\begin{align*}
	\|\Psi\hat\bW - \bY \|_F^2 \leq & \|\Psi\bar \bW - \bY\|^2_F  = \|\bN -\bE\bar \bW \|_F^2. \\
\Longrightarrow \|\Psi\hat \bW - \bV -\bE\bar \bW \|^2 \leq & 2\langle \Psi\hat \bW - \bV -\bE\bar \bW, \bN-\bE \bar \bW\rangle \\
\Longrightarrow \|\Psi\hat \bW - \bV -\bE\bar \bW \| \leq & \|P_{[\Psi,E,V]}\bN\| + \|\bE\bar \bW\|\\
	\Longrightarrow \|\Psi\hat \bW - \bV\|\lesssim & \|\bE\|_F\|\bar \bW\|+ (\sqrt{d_2} + \sqrt{\log 1/\delta})\sqrt{\Trace(\bSigma_{YY|X_1})}. \tag{from Lemma \ref{claim:bound_projected_gaussian}}\\
	\leq & \sqrt{n_2} \frac{\epsilon_{\tCI}}{\beta} + (\sqrt{d_2} + \sqrt{\log 1/\delta})\sqrt{\Trace(\bSigma_{YY|X_1})}  \tag{from Assumption \ref{assumption:linear_approximate_CI}}.
	\end{align*} 
	
	\begin{comment} 
	\begin{align*}
	\|\Psi\hat\bW - \bV \|_F \leq & \|P_{\Psi}^\perp \bV\|_F + \|P_{\Psi}\bN\|_F\\
	= & \|P_{\Psi}^\perp P_{\bV} \bV\|_F + \|P_{\Psi}\bN\|_F \\
	= & \|P_{\Psi}^\perp P_{\bL} \bV\|_F + \|P_{\Psi}N\|_F \\
	\leq & \|P_{\Psi}^\perp P_{\bL}\|_2 \|\bV\|_F + \|P_{\Psi}\bN\|_F \\
	= &  \|\sin \Theta(\Psi, \bL)\|_2 \|\bV\|_F + \|P_{\Psi}\bN\|_F \tag{from Lemma \ref{lemma:error_of_PCA}} \\
	\leq & \frac{1.1\epsilon_{\tCI}}{\beta} \|\bV\|_F + \sqrt{d_2k}\bSigma+\sqrt{\log 1/\delta} \tag{from Lemma \ref{lemma:bound_projected_subgaussian} and \ref{lemma:error_of_PCA}}. 
	\end{align*}
\end{comment} 
Next, by the same procedure that concentrates $\frac{1}{n_2}\bX_1^\top\bX_1 $ to $\bSigma_{X_1X_1}$ with Claim \ref{claim:concentration_gaussian_covariance},  we could easily get %\jnote{add proof details. What is B? Should it be A?}:
\begin{equation*}
\ER[\hat\bW]:=\E[\|\hat\bW^\top \psi(X_1) - f^*(X_1)\|^2] \lesssim \frac{\epsilon_{\tCI}^2}{\beta^2}+\Trace(\bSigma_{YY|X_1})\frac{d_2   + \log 1/\delta}{n_2}. 
\end{equation*}

\end{proof}

\subsection{Measuring conditional dependence with cross-covariance operator}
\label{section:cross_covariance_operator} 
%Next we will look at more general function space. 
%Let $(\cH_{x}, k_x)$ be RKHSs of functions on $\cX$, with measurable positive definite kernels $k_x$. 
%We denote by $P_X$ the marginal distribution of $X$.
$L^2(P_X)$ denotes the Hilbert space of square integrable function with respect to the measure $P_X$, the marginal distribution of $X$. We are interested in some function class $\cH_x\subset L^2(P_X)$ that is induced from some feature maps:

\begin{definition}[General and Universal feature Map]
	We denote feature map $\phi: \cX\rightarrow \cF$ that maps from a compact input space $\cX$ to the feature space $\cF$. $\cF$ is a Hilbert space associated with inner product: $\langle \phi(\bx),\phi(\bx')\rangle_{\cF}$. The associated function class is: $\cH_x = \{h:\cX\rightarrow \R|\exists w\in \cF, h(\bx)= \langle w,\phi(\bx)\rangle_{\cF} ,\forall \bx\in \cX  \}.$ We call $\phi$ universal if the induced $\cH_x$ is dense in $L^2(P_X)$. %\revise{Use $\psi$ instead of $f$ to avoid confusion, since $f$ is used for $\E[Y|X_1]$.}
\end{definition}
Linear model is a special case when feature map $\phi=Id$ is identity mapping and the inner product is over Euclidean space. A feature map with higher order polynomials correspondingly incorporate high order moments \citep{fukumizu2004dimensionality,gretton2005measuring}. For discrete variable $Y$ we overload $\phi$ as the one-hot embedding. 
\begin{remark}
	For continuous data, any universal kernel like Gaussian kernel or RBF kernel induce the universal feature map that we require \citep{micchelli2006universal}. Two-layer neural network with infinite width also satisfy it, i.e., $\forall \bx \in \cX\subset \R^d, \phi_{NN}(\bx):\cS^{d-1}\times \R\rightarrow \R, \phi_{NN}(\bx)[\bw,b] = \sigma(\bw^\top \bx+b) $ \citep{barron1993universal}. %Without loss of generality, we assume feature map is properly shifted such that: $\E_X[\phi(X)]=0$. 
	%For classification task with $k$ different classes, $\cY=\{1,2,\cdots k\}$ and $\phi(y) = \be_{y}\in \R^k,\forall y\in \cY$.  
	%We respectively use $\cH_1,\cH_2,$ and $\cH_y$ and the corresponding subscripts to denote the function claFss for random variables $X_1, X_2$ and $Y$.
\end{remark} 
When there's no ambiguity, we overload $\phi_1$ as the random variable $\phi_1(X_1)$ over domain $\cF_1$, and $\cH_1$ as the function class over $X_1$. 
Next we characterize CI using the cross-covariance operator.
\begin{definition}[Cross-covariance operator] For random variables $X\in \cX, Y\in \cY$ with joint distribution $P:\cX\times \cY\rightarrow \R$, and associated feature maps $\phi_x$ and $\phi_y$, we denote by $\cC_{\phi_x \phi_y}=\E[\phi_x(X)\otimes \phi_y(Y)] = \int_{\cX\times \cY} \phi_x(x) \otimes \phi_y(y)dP(x,y),$ the (un-centered) cross-covariance operator. Similarly we denote by $\cC_{X\phi_y} = \E[X\otimes \phi_y(Y)]:\cF_y\rightarrow \cX$. 
\end{definition}
To understand what $\cC_{\phi_x\phi_y}$ is, we note it is of the same shape as $\phi_x(x)\otimes \phi_y(y)$ for  each individual $x\in \cX, y\in \cY $. It can be viewed as an operator: $\cC_{\phi_x\phi_y}: \cF_y \rightarrow \cF_x$, $ \cC_{\phi_x\phi_y} f = \int_{\cX\times \cY} \langle \phi_y(y), f\rangle \phi_x(x)  dP(x,y), \forall f\in \cF_y $. For any $f\in \cH_x$ and $g\in \cH_y$, it satisfies:
$
\langle f, \cC_{\phi_x\phi_y} g \rangle_{\cH_x} = \E_{XY}[f(X)g(Y)]
$\citep{baker1973joint,fukumizu2004dimensionality}. CI ensures $\cC_{\phi_1X_2|\phi_y}=0$ for arbitrary $\phi_1,\phi_2$: 
\begin{lemma}
	\label{lemma:CI_vs_cross_covariance}
	With one-hot encoding map $\phi_y$ and arbitrary $\phi_1$, $X_1\bot X_2|Y$ ensures:
	\begin{comment} 
	\begin{equation}
	\label{eqn:partial_cov_nonlinear_definition}
	\E[(\phi_1(X_1)-\E^L[\phi_1(X_1)|\phi_y(Y)])\otimes (X_2-\E^L[X_2|\phi_y(Y)])] = 0. 
	\end{equation}
	Therefore 
	\end{comment} 
	\begin{equation}
	\label{eqn:partial_cov_nonlinear_formula}
	\cC_{\phi_1X_2|\phi_y} := \cC_{\phi_1X_2} - \cC_{\phi_1\phi_y}\cC_{\phi_y\phi_y}^{-1}\cC_{\phi_y X_2} = 0.
	\end{equation}
\end{lemma} 
A more complete discussion of cross-covariance operator and CI can be found in \citep{fukumizu2004dimensionality}. Also, recall that an operator $\cC:\cF_y\rightarrow \cF_x$ is  Hilbert-Schmidt (HS) \citep{reed2012methods} if for complete orthonormal systems (CONSs) $\{\zeta_i\}$ of $\cF_x$ and $\{\eta_i\}$ of $\cF_y$,  $\|\cC\|^2_{\HS}:=\sum_{i,j}\langle \zeta_j, \cC\eta_i\rangle^2_{\cF_x}<\infty $. %And its Hilbert-Schmidt (HS) norm $\|\cC\|_{\HS}$  is defined by it. 
The Hilbert-Schmidt norm generalizes the Frobenius norm from matrices to operators, and we will later use $\|\cC_{\phi_1X_2|\phi_y}\|$ to quantify approximate CI.
%\revise{Something is off with the definition? May be $\cC:\cF_y\rightarrow \cF_x$ and HS norm is $\sum_{i,j}\langle \phi_j, \cC\phi_i\rangle_{\cF_x} $. Also probably better to use something other than $\phi$ and $\psi$ to avoid abuse of notation.}

We note that covariance operators \citep{fukumizu2009kernel,fukumizu2004dimensionality,baker1973joint}  are commonly used to capture conditional dependence of random variables. In this work, we utilize the covariance operator to quantify the performance of the algorithm even when the algorithm is \textit{not a kernel method}.

\subsection{Omitted Proof in General Setting} 

\begin{claim}
	\label{claim:nonlinear_conditional_expectation}
	For feature maps $\phi_1$ with universal property, we have:
	\begin{align*}
	\psi^*(X_1):=&\E[X_2|X_1]= \E^L[X_2|\phi_1]\\ 
	= & \cC_{X_2 \phi_1}\cC_{\phi_1\phi_1}^{-1}\phi_1(X_1).  \\
	\text{Our target }f^*(X_1):=&\E[Y|X_1] = \E^L[Y|\phi_1]\\
	= & \cC_{Y\phi_1}\cC_{\phi_1\phi_1}^{-1}\phi_1(X_1).
	\end{align*}
	For general feature maps, we instead have:
	\begin{align*}
	\psi^*(X_1):=&\argmin_{f\in \cH_1^{ d_2}} \E_{X_1X_2}\|X_2-f(X_1)\|_2^2\\ 
	= & \cC_{X_2 \phi_1}\cC_{\phi_1\phi_1}^{-1}\phi_1(X_1).  \\
	\text{Our target }f^*(X_1):=&\argmin_{f\in \cH_1^{ k}} \E_{X_1Y}\|Y-f(X_1)\|_2^2\\
	= & \cC_{Y\phi_1}\cC_{\phi_1\phi_1}^{-1}\phi_1(X_1).
	\end{align*}
\end{claim}
To prove Claim \ref{claim:nonlinear_conditional_expectation}, we show the following lemma:
\begin{lemma}
	\label{lemma:universal_feature_property} 
	Let $\phi:\cX\rightarrow \cF_x$ be a universal feature map, then for random variable $Y\in \cY$ we have:
	\begin{equation*}
	\E[Y|X] = \E^L[Y|\phi(X)].
	\end{equation*}
\end{lemma}

\begin{proof}[Proof of Lemma \ref{lemma:universal_feature_property}]
	Denote by $\E[Y|X=x]=:f(x)$. Since $\phi$ is dense in $\cX$, there exists a linear operator $a:\cX\rightarrow \R$ such that $\int_{x\in \cX} a(x)\phi(x)[\cdot]dx=f(\cdot)$ a.e. Therefore the result comes directly from the universal property of $\phi$. 
\end{proof}

\begin{proof}[Proof of Claim \ref{claim:nonlinear_conditional_expectation}]
We want to show that for random variables $Y,X$, where $X$ is associated with a universal feature map $\phi_x$, we have $\E[Y|X] = \cC_{Y\phi_x(X)}\cC_{\phi_x(X)\phi_x(X)}^{-1}\phi_x(X) $.

First, from Lemma \ref{lemma:universal_feature_property}, we have that $ \E[Y|X] = \E^L[Y|\phi_x(X)]$. Next, write $A^*: \cF_x\rightarrow \cY$ as the linear operator that satisfies
\begin{align*}
 & \E[Y|X] = A^*\phi_x(X)\\
\text{s.t. } & A^* =  \argmin_A\E[\|Y-A\phi_x(X)\|^2 ].
\end{align*}
Therefore from the stationary condition we have $A^* \E_X[\phi_x(X)\otimes \phi_x(X)] = \E_{XY}[Y\otimes \phi_x(X)]$. Or namely we get $A^* = \cC_{Y\phi_x}\cC_{\phi_x\phi_x}^{-1}$ simply from the definition of the cross-covariance operator $\cC$. 	
\end{proof}

\begin{claim}
	\label{claim:relation_eps_CI}
$\|\cC_{\phi_1\phi_1}^{-1/2} \cC_{\phi_1X_2|\phi_{\bar y}}\|_{\HS}^2 = \E_{X_1}[\|\E[X_2|X_1]-\E_{\bar Y}[\E[X_2|\bar Y]|X_1]\|^2] = \epsilon_{\tCI}^2 $. 
\end{claim}
\begin{proof}
	\begin{align*}
	& \|\cC_{\phi_1\phi_1}^{-1/2} \cC_{\phi_1X_2|\phi_{\bar y}}\|_{\HS}^2\\
	= & \int_{X_1} \left\| \int_{X_2} \left( \frac{p_{X_1X_2}(\bx_1,\bx_2) }{p_{X_1}(\bx_1) }- \frac{p_{X_1\bot X_2|Y}(\bx_1,\bx_2) }{p_{X_1}(\bx_1) } \right)X_2 dp_{\bx_2}   \right\|^2  dp_{\bx_1}\\
	= & \E_{X_1}[\|\E[X_2|X_1]-\E_{\bar Y}[\E[X_2|\bar Y]|X_1]\|^2]. 
	\end{align*}
\end{proof}

\subsection{Omitted Proof for Main Results}
We first prove a simpler version without approximation error.
\begin{theorem}
	\label{thm:approx_CI_no_approx_error}
	For a fixed $\delta\in (0,1)$, under Assumption \ref{assumption:feature_map_approx_CI}, \ref{assumption:subgaussian_psi}, if there is no approximation error, i.e., there exists a linear operator $A$ such that $f^*(X_1)\equiv A\phi_1(X_1)$, if $n_1,n_2\gg \rho^4(d_2+\log 1/\delta)$, and we learn the pretext tasks such that:
	\begin{align*}
	\E\|\tilde \psi(X_1)-\psi^*(X_1)\|_F^2\leq  \epsilon^2_{\pre}.
	\end{align*}
	Then we are able to achieve generalization for downstream task with probability $1-\delta$: 
	\begin{align}
%	\label{eqn:main_result_approx_CI}
 \E[\|f^*_{\cH_1}(X_1) - \hat\bW^\top \tilde\psi(X_1) \|^2] \leq \tilde\cO\{   \sigma^2\frac{d_2}{n_2} + \frac{\epsilon_{\tCI}^2}{\beta^2} +\frac{\epsilon_{\pre}^2}{\beta^2 }\}.
	\end{align} 
\end{theorem}

\begin{proof}[Proof of Theorem \ref{thm:approx_CI_no_approx_error}] %\ref{thm:main_result_approximate_CI}]
%\jnote{Notation doesn't match main paper. We used $C$ for covariance operator.}
We follow the similar procedure as Theorem \ref{thm:linear_general_apx}. For the setting of no approximation error, we have $f^*=f^*_{\cH_1}$, and the residual term $N :=Y - f^*(X_1)$ is a mean-zero random variable with $\E[\|N\|^2|X_1]\lesssim \sigma^2$ according to our data assumption in Section \ref{sec:CI}. $\bN=\bY-f^*(\bX_1^{\down})$ is the collected $n_2$ samples of noise terms. We write $Y\in \R^{d_3}$. For classification task, we have $Y\in \{\be_i, i\in [k]\} \subset \R^k$ (i.e, $d_3=k$) is one-hot encoded random variable. For regression problem, $Y$ might be otherwise encoded. For instance, in the yearbook dataset, Y ranges from 1905 to 2013 and represents the years that the photos are taken. We want to note that our result is general for both cases: the bound doesn't depend on $d_3$, but only depends on the variance of $N$.

Let  $\Psi^*, \bL, \bE,\bV$ be defined as follows: 

Let $\bV=f^*(\bX_1^{\down}) \equiv f^*_{\cH_1}(\bX_1^{\down})\equiv \phi(\bX_1^{\down})\cC^{-1}_{\phi_1}\cC_{\phi_1Y}$ be our target direction. Denote the optimal representation matrix by 
\begin{align*}
\Psi^*:= & \psi^*(\bX_1^{\down})\\
= & \phi(\bX_1^{\down})  \cC_{\phi_1 \phi_1}^{-1} \cC_{\phi_1 X_2} \\
=	& \phi(\bX_1^{\down})\cC^{-1}_{\phi_1\phi_1} \cC_{\phi_1\phi_{\bar y}}\cC_{\phi_{\bar y}}^{-1}\bSigma_{\phi_{\bar y} X_2} +  \phi(\bX_1^{\down})\cC^{-1}_{\phi_1\phi_1} \cC_{\phi_1X_2|\phi_{\bar y}}\\
=: & \bL + \bE,
\end{align*}
where $\bL =  \phi(\bX_1^{\down})\cC^{-1}_{\phi_1\phi_1} \cC_{\phi_1\phi_{\bar y}}\cC_{\phi_{\bar y}}^{-1}\cC_{\phi_{\bar y} X_2}$ and $\bE = \phi(\bX_1^{\down})\cC^{-1}_{\phi_1\phi_1} \cC_{\phi_1X_2|\bar Y}$. 

In this proof, we denote $S_Y$ as the matrix such that $S_Y\phi_{\bar y} = Y$. %\jnote{Is it important each row of $S_Y$ is 1-sprase??}
Specifically, if $Y$ is of dimension $d_3$, $S_Y$ is of size $d_3\times |\cY||\cZ|$. Therefore $S_Y \bSigma_{\phi_y A} = \bSigma_{YA} $ for any random variable $A$. 

Therefore, similarly we have: 
$$ \bL \bSigma_{X_2\phi_{\bar y}}^{\dagger}\bSigma_{\phi_{\bar y}\phi_{\bar y}}S_Y^\top = \bL \bSigma_{X_2\phi_{\bar y}}^{\dagger}\bSigma_{\phi_{\bar y}Y} = \bL \bar \bW = \bV  $$
where $\bar \bW:= \bSigma_{X_2\phi_{\bar y}}^{\dagger}\bSigma_{\phi_{\bar y}Y}$ satisfies $\|\bar \bW\|_2=1/\beta$. 
Therefore span$(\bV)\subseteq $span$(\bL)$ since we have assumed that $\bSigma_{X_2\phi_{\bar y}}^{\dagger}\bSigma_{\phi_{\bar y}Y}$ to be full rank.

%\jnote{needs some justification now, since Y and $\phi_y$ are different after fixing typo}. 
On the other hand, $\bE = \phi_1(\bX_1^{\down})\cC^{-1}_{\phi_1\phi_1} \cC_{\phi_1X_2|\bar Y}$ concentrates to $\cC^{-1/2}_{\phi_1\phi_1} \cC_{\phi_1X_2|\phi_{\bar y}}$. Specifically, when $n\gg k+\log 1/\delta $, $ \frac{1}{n_2}\|\bE\|^2_F\leq 1.1 \|\cC^{-1/2}_{\phi_1\phi_1} \cC_{\phi_1X_2|\phi_{\bar y}}\|^2_F \leq 1.1\epsilon^2_{\tCI}$ (by using Lemma \ref{claim:concentration_subgaussian_covariance} ). Together we have $\|\bE \bar \bW\|_F\lesssim \epsilon_{\tCI}/\beta$.  	

We also introduce the error from not learning $\psi^*$ exactly: $\bE^{\pre} = \Psi - \Psi^* := \tilde \psi(\bX_1^{\down})-\psi^*(\bX_1^{\down})$. With proper concentration and our assumption, we have that $ \E\|\psi(X_1)-\psi^*(X_1)\|^2 \leq \epsilon_{\pre} $ and $ \frac{1}{\sqrt{n_2}}\|\psi(\bX_1^{\down})-\psi^*(\bX_1^{\down})\|^2 \leq 1.1\epsilon_{\pre}  $. 

Also, the noise term after projection satisfies $ \|P_{[\Psi,\bE, \bV]} \bN \|\lesssim \sqrt{d_2(1+\log d_2/\delta) }\sigma $ as using Corollary \ref{lemma:bound_projected_subgaussian}. %\jnote{$N$ isn't defined. And I guess it doesn't interface with Lemma A.7 properly which needed N to be centered multinomial. $\sigma$ also ins't defined. What is the relation here? $Y= \phi(X_1) W^\ast +N$?} . 
Therefore $\Psi = \Psi^* - \bE^{\pre} = \bL +\bE-\bE^{\pre}$.

Recall that $\hat \bW =\argmin_{\bW} \|\psi(\bX_1^{\down})\bW - \bY\|_F^2. $ And with exactly the same procedure as Theorem \ref{thm:linear_general_apx} we also get that: %\jnote{What is $\hat W, \bar W$, define them. I think when $\phi_y \neq y$, you will need to change definition of $W$, and then the definition of $\beta$ needs to change accordingly. It needs to depend on $\theta^\top \phi_y =y$} :

%\jnote{where is approx error? Can't be exactly the same, there is no approx error in C.1}
\begin{align*}
\|\Psi \hat \bW - \bV\|\leq & 2\|\bE \bar \bW\| + 2\|\bE^{\pre} \bar \bW \| + \|P_{[\Psi,\bE,\bV,\bE^{\pre}]} \bN \|\\
\lesssim & \sqrt{n_2} \frac{\epsilon_{\tCI}+\epsilon_{\pre }}{\beta} + \sigma \sqrt{d_2(1+\log (d_2/\delta))}. 
\end{align*}
With the proper concentration we also get:
\begin{align*}
\E[\|\hat\bW^\top \psi( X_1) - f^*_{\cH_1}(X_1)\|^2] \lesssim & \frac{\epsilon_{\tCI}^2+\epsilon_{\pre}^2}{\beta^2}+\sigma^2\frac{d_2(1+\log (d_2/\delta))}{n_2}. 
\end{align*}
	%\jnote{approximation error is not ever mentioned here.}
\end{proof}

Next we move on to the proof of our main result Theorem \ref{thm:main_result_approximate_CI} where approximation error occurs. 
\begin{proof}[Proof of Theorem \ref{thm:main_result_approximate_CI}]
	%\jnote{Add a real proof.}
	The proof is a combination of Theorem \ref{thm:CI_linear_sample_complexity} and Theorem \ref{thm:approx_CI_no_approx_error}. We follow the same notation as in Theorem \ref{thm:approx_CI_no_approx_error}. Now the only difference is that an additional term $a(\bX_1^{\down})$ is included in $\bY$: 
	\begin{align*}
	\bY = & \bN + f^*(\bX_1^{\down})\\
	= & \bN + \Psi^* \bar \bW+a(\bX_1^{\down})\\
	=& \bN + (\Psi+\bE^{\pre}) \bar \bW + a(\bX_1^{\down})\\
	= & \Psi \bar \bW  + ( \bN + \bE^{\pre}\bar \bW + a(\bX_1^{\down}) ).  
	\end{align*}
From re-arranging $	\frac{1}{2n_2} \|\bY - \Psi \hat \bW \|_F^2 \leq 	\frac{1}{2n_2} \|\bY - \Psi \bar \bW \|_F^2$,	
	\begin{align}
	&\frac{1}{2n_2} \|\Psi (\bar \bW- \hat \bW) + (\bN + \bE^{\pre} + a(\bX_1^{\down}) )  \|_F^2 \leq  \frac{1}{2n_2}\|\bN + \bE^{\pre}\bar \bW + a(\bX_1^{\down}) \|_F^2\\
	\label{eqn:last_ineq_approx_error}
&	\Rightarrow \frac{1}{2n_2} \|\Psi (\bar \bW- \hat \bW)\|_F^2 \leq  \frac{1}{n_2}\langle \Psi (\bar \bW- \hat \bW), \bN + \bE^{\pre}\bar \bW + a(\bX_1^{\down})\rangle. 
	\end{align}
Then with similar procedure as in the proof of Theorem \ref{thm:CI_linear_sample_complexity}, and write $\Psi$ as $\phi(X_1^{\down})\bB$, we have:
\begin{align*}
& \frac{1}{n_2}\langle \Psi (\bar \bW- \hat \bW), a(\bX_1^{\down})\rangle \\
= & \frac{1}{n_2}\langle \bB (\bar \bW- \hat \bW), \phi(\bX_1^{\down})^\top a(\bX_1^{\down})\rangle\\
= & \frac{1}{n_2}\langle \cC_{\phi_1}^{1/2}\bB (\bar \bW- \hat \bW), \cC_{\phi_1}^{-1/2}\phi(\bX_1^{\down})^\top a(\bX_1^{\down})\rangle \\
\leq & \sqrt{\frac{d_2}{n_2}} \|\cC_{\phi_1}^{1/2}\bB (\bar \bW- \hat \bW)\|_F\\
\leq & 1.1 \frac{1}{\sqrt{n_2}}\sqrt{\frac{d_2}{n_2}} \|\phi(\bX_1^{\down})\bB (\bar \bW- \hat \bW)\|_F\\
= & 1.1 \frac{\sqrt{d_2}}{n_2}\|\Psi (\bar \bW- \hat \bW)\|_F. 
\end{align*}
Therefore plugging back to 	\eqref{eqn:last_ineq_approx_error} we get:
\begin{align*}&
\frac{1}{2n_2}\|\Psi (\bar \bW- \hat \bW)\|_F^2 \leq  \frac{1}{n_2}\langle \Psi (\bar \bW- \hat \bW), \bN + \bE^{\pre}\bar \bW + a(\bX_1^{\down})\rangle\\&
\Rightarrow \frac{1}{2n_2}\|\Psi (\bar \bW- \hat \bW)\|_F \leq  \frac{1}{2n_2}\|\bE^{\pre} \bar \bW\|_F + \frac{1}{2n_2}\|P_{\Psi}\bN\|_F + 1.1\frac{\sqrt{d_2}}{n_2}. \\
&\Rightarrow \frac{1}{2\sqrt{n_2}}\|\Psi \hat \bW - f^*_{\cH_1}(\bX_1^{\down})\|_F-\|\bE \bar \bW\|_F \leq  \frac{1}{\sqrt{n_2}}(1.1\sqrt{d_2} + \|\bE^{\pre}\bar \bW\| + \sqrt{d_2+\log (d_2/\delta)} )\\&
\Rightarrow  \frac{1}{2\sqrt{n_2}}\|\Psi \hat \bW - f^*_{\cH_1}(\bX_1^{\down})\|_F \lesssim \sqrt{\frac{d_2(1+\log d_2/\delta)}{n_2}}+\frac{\epsilon_{\tCI}+\epsilon_{\pre }}{\beta}. 
\end{align*}
Finally by concentrating $\frac{1}{n_2}\Psi^\top \Psi$ to $\E[\tilde\psi(X_1)\tilde\psi(X_1)^\top]$ we get:
\begin{align*}
\E[\|\hat \bW^\top \tilde\psi(X_1) - f^*_{\cH_1}(X_1)\|_2^2]\lesssim \frac{d_2(1+ \log d_2/\delta}{n_2}) + \frac{\epsilon_{\tCI}^2+\epsilon_{\pre}^2}{\beta^2}, 
\end{align*}
with probability $1-\delta$. 
\end{proof}

\subsection{Principal Component Regression}

\begin{claim}[Approximation Error of Principle Component Analysis]
	\label{lemma:error_of_PCA} 
	Let matrix $\bA=\bL+\bE\in \R^{n\times d}$ where $\bL$ has rank $r<$size of $\bA$. % and $ \bSigma_r(\bA)=\beta$. 
	Let $\bA_r$ be the rank-$r$ PCA of $A$. 
	Then we have:
	$\|\bA_r - \bL \|_F\leq 2\|\bE\|_F, $ and 		$\|\bA_r - \bL \|_2\leq 2\|\bE\|_2$. 
\end{claim}
\begin{proof}
	Due to the property of PCA, $\|\bA_r - \bA\|_F\leq \|\bE\|_F $ and $\|\bA_r - \bA\|_2\leq \|\bE\|_2$.
	\begin{align*}
	\|\bA_r - \bL\|_2  =& \|\bA_r - \bA  + \bA- \bL\|_2\\
	\leq &  \|\bA_r - \bA\|_F + \|\bE\|_F\\
	\leq & 2\|\bE\|_2. 
	\end{align*}
	Similarly we have $ \|\bA_r - \bL\|_F\leq 2\|\bE\|_F $. 
	\begin{comment} 
	We use Davis Kahan for this proof. 
	First note that $\|\bA-\bL\|= \|\bE\|\leq \epsilon$. Due to the fact that $L$ is rank $r$ and the definition of PCA ($\|\bA_r-\bA\|\leq \|\bE\|$), we have $\|\bA_r-\bL\|_2 = \|\bA_r - \bA + \bA- L\|_2\leq 2\|\bE\|_2\leq 2\epsilon$. 
	From Davis-Kahan we get:
	\begin{align*}
	\|\sin \Theta(\bA_r, \bL)\|_2\leq & \frac{ \|\bA_r-\bL\|_2}{\bSigma_{r}(\bA)-\bSigma_{r+1}(\bL)}\\
	= & \frac{\|A_r-L\|_2}{\bSigma_{r}(\bA)} \\
	\leq & 2\epsilon/\beta. 
	\end{align*}	
	\end{comment} 
\end{proof}
This technical fact could be used to complete the proof for Remark \ref{remark:pca}.

\begin{proof}[Proof of Remark \ref{remark:pca}]
	
	We replace the key steps of \ref{thm:approx_CI_no_approx_error}. 	
	
	%We follow the similar procedure as Theorem \ref{thm:linear_general_apx}. For the setting of no approximation error, we have $f^*=f^*_{\cH_1}$, and the residual term $N :=Y - f^*(X_1)$ is a mean-zero random variable with $\E[\|N\|^2|X_1]\lesssim \sigma^2$ according to our data assumption in Section \ref{sec:CI}. $\bN=\bY-f^*(\bX_1^{\down})$ is the collected $n_2$ samples of noise terms. We write $Y\in \R^{d_3}$. For classification task, we have $Y\in \{\be_i, i\in [k]\} \subset \R^k$ (i.e, $d_3=k$) is one-hot encoded random variable. For regression problem, $Y$ might be otherwise encoded. For instance, in the yearbook dataset, Y ranges from 1905 to 2013 and represents the years that the photos are taken. We want to note that our result is general for both cases: the bound doesn't depend on $d_3$, but only depends on the variance of $N$. 

	Recall  $\Psi^*, \bL, \bE,\bV$ are defined as follows: 
	
	$\Psi^*:=\psi^*(\bX_1^{\down})$ is the optimal representation matrix.  $\Psi_r$ is the features obtained from $r$-PCA of $\Psi^*$. $\Psi^*=\bL + \bE$ which is low rank plus small norm.  ($\bL =  \phi(\bX_1^{\down})\cC^{-1}_{\phi_1\phi_1} \cC_{\phi_1\phi_{\bar y}}\cC_{\phi_{\bar y}}^{-1}\cC_{\phi_{\bar y} X_2}$ and $\bE = \phi(\bX_1^{\down})\cC^{-1}_{\phi_1\phi_1} \cC_{\phi_1X_2|\bar Y}$.  Suppose $r=|\cY||\cZ|$.)
	Let $\bV=f^*(\bX_1^{\down}) \equiv f^*_{\cH_1}(\bX_1^{\down})\equiv \phi(\bX_1^{\down})\cC^{-1}_{\phi_1}\cC_{\phi_1Y} = \bL \bar \bW$ be our target direction, where $\bar \bW:= \bSigma_{X_2\phi_{\bar y}}^{\dagger}\bSigma_{\phi_{\bar y}Y}$.  
	
	Due to representation learning error (finite sample in the first stage) and approximate conditional independence, the target direction $\bV$ is not perfectly linear in $\Psi^*$ or its $r$-PCA features $\Psi$. 
	
	Now with PCR we learn the linear model with $\hat \bW \leftarrow \argmin_{\bW} \|\Psi_r\bW - \bY\|_F^2. $
	Together with \ref{lemma:error_of_PCA} and the same procedure as Theorem \ref{thm:approx_CI_no_approx_error} we also get that: %\jnote{What is $\hat W, \bar W$, define them. I think when $\phi_y \neq y$, you will need to change definition of $W$, and then the definition of $\beta$ needs to change accordingly. It needs to depend on $\theta^\top \phi_y =y$} :
	
	%\jnote{where is approx error? Can't be exactly the same, there is no approx error in C.1}
	
	Let $\bar \bE = \bL - \Psi_r$ is of rank at most $2r$.
	\begin{align*}
	\|\Psi_r\hat\bW - \bY \|_F^2 \leq & \|\Psi_r\bar \bW - \bY\|^2_F  = \|\bN -\bar\bE\bar \bW \|_F^2. \\
	\Longrightarrow \|\Psi_r\hat \bW - \bV -\bar\bE\bar \bW \|^2 \leq & 2\langle \Psi_r\hat \bW - \bV -\bar\bE\bar \bW, \bN-\bar\bE \bar \bW\rangle \\
	\Longrightarrow \|\Psi_r\hat \bW - \bV -\bar\bE\bar \bW \| \leq & \|P_{[\Psi_r,\bL]}\bN\| + \|\bar\bE\bar \bW\|\\
	\Longrightarrow \|\Psi_r\hat \bW - \bV\|\leq & 2\|\bar\bE\|_F\|\bar \bW\|+ \|P_{2r}\bN\| \\
	\lesssim & \|\bE\|_F \|\bar\bW\| + \sigma\sqrt{r}(1+\sqrt{\log(r/\delta)}) .
	\end{align*}
	With concentration on the downstream labeled samples we also get the result in Remark \ref{remark:pca}:
	\begin{align*}
	\E[\|\hat\bW^\top \psi_r( X_1) - f^*_{\cH_1}(X_1)\|^2] \lesssim & \frac{\epsilon_{\tCI}^2+\epsilon_{\pre}^2}{\beta^2}+\sigma^2\frac{r(1+\log (r/\delta))}{n_2}. 
	\end{align*}
	Here $r=|\cY||\cZ|$ . 
\end{proof}

\section{Omitted Proofs Beyond Conditional Independence} 

\subsection{Proof for topic modeling example}
\label{sec:topic_model_proof}

\begin{proof}[Proof for Theorem~\ref{thm:topic_model}]
	We will construct a latent variable $\bar{Y}$ such that $\epsilon_{\tCI} = 0$.
	We pick the domain of $\bar{Y}$ to be $[k]$ and the distribution $P(\bar{Y} | X_1)$ to be the distribution $\E\left[\mu | X_1\right] \in \Delta_{[k]}$, and define $P\left(X_2 | \bar{Y} = i\right) = P\left(X_2 | \mu = e_{i}\right)$.
	More specifically we have
	\begin{align*}
		P(\bar{Y}=i | X_1)
		&= \E\left[\mu | X_1\right](i) = \E\left[\mu(i) | X_1\right]\text{ and thus }\E\left[\bar{Y} | X_1\right] = \E\left[\mu | X_1\right]\\
		P\left(X_2 | \bar{Y} = i\right)
		&= P\left(X_2 | \mu = e_{i}\right)\text{ and thus }
		\E\left[X_2 | \bar{Y} = i\right]
		= \E\left[X_2 | \mu = e_{i}\right]
	\end{align*}
%	Furthermore we define $\E\left[X_2 | \bar{Y} = i\right] = \E\left[X_2 | \mu = e_{i}\right]$.
	To show $\epsilon_{\tCI} = 0$, from Definition~\ref{def-approx-CI} we need to show $\E\left[X_2 | X_1\right] = \E\left[\E\left[X_2 | \bar{Y}\right] | X_1\right]$.
	Since $X_2$ is the bag of words representation, we know that $X_2 = \frac{2}{N} \sum_{i=N/2+1}^{N} {e_{w_i}}$.
	So for any $\mu\in\Delta_{[k]}$ we get
	\begin{align*}
		\E\left[X_2 | \mu\right]
		&=^{(a)} \frac{2}{N} \sum_{i=N/2+1}^{N} \E\left[e_{w_{i}} | \mu\right]
		=^{(b)} \frac{2}{N} \sum_{i=N/2+1}^{N} A \mu = A\mu
	\end{align*}
	where $(a)$ follows from linearity of expectation and $(b)$ follows from the linearity of the probability distribution of each word given $\mu$ for topic models.
	Thus from the definition of $\bar{Y}$, $\E\left[X_2 | \bar{Y} = i\right] = \E\left[X_2 | \mu = e_{i}\right] = Ae_{i}$.
	To check if $\epsilon_{\tCI}=0$, we compute the following
	\begin{align*}
		\E\left[\E\left[X_2 | \bar{Y}\right] | X_1\right]
		&= \sum_{i=1}^{k} \E\left[X_2 | \bar{Y}=i\right] P(\bar{Y} = i | X_1)\\
		&= \sum_{i=1}^{k} A e_{i} ~\E\left[\mu(i) | X_1\right]
		=  A \sum_{i=1}^{k}\E\left[\mu(i)e_{i} | X_1\right]\\
		&= \E\left[A\mu | X_1\right]
		= \E\left[\E\left[X_2 | \mu\right] | X_1\right]
	\end{align*}
	Due to the topic modeling assumption and the independent sampling of words given $\mu$, we know that $X_1 \perp X_2 | \mu$ and thus $\E\left[X_2 | X_1\right] = \E\left[\E\left[X_2 | \mu\right] | X_1\right]$.
	Combining with the above calculation, we get that $\E\left[\E\left[X_2 | \bar{Y}\right] | X_1\right] = \E\left[X_2 | X_1\right]$, thus giving $\epsilon_{\tCI} = 0$.
	This proves points 1. and 2.

	For point 3., note that $\E[Y | X_1] = \E[w^{\top} \mu | X_1] = w^{\top} \E[\mu | X_1] = w^{\top} \E[\bar{Y} | X_1]$.
	
	Finally for point 4., we use the definition $1/\beta = \|\bSigma_{Y\phi_{\bar y}}\bSigma_{X_2\phi_{\bar y}}^\dagger\|_2$.
	For the first term, we note that $\E \left[\phi_{\bar{Y}} | \mu\right] = \E \left[\E \left[\phi_{\bar{Y}} | X_1\right] | \mu\right] = \E \left[\E \left[\bar{\mu} | X_1\right] | \mu\right] = \mu$
	\begin{align*}
		\bSigma_{Y\phi_{\bar y}}
		&= \E_{\mu\sim\tau} \left[Y \phi_{\bar{Y}}^{\top}\right]
		= \E_{\mu\sim\tau} \left[w^{\top} \mu \phi_{\bar{Y}}^{\top}\right]\\
		&= \E_{\mu\sim\tau} \left[w^{\top} \mu \E\left[\phi_{\bar{Y}}^{\top} | \mu\right]\right]
		= \E_{\mu\sim\tau} \left[w^{\top} \mu \mu^{\top}\right]\\
		&= w^{\top} \Gamma
	\end{align*}
	where $\Gamma$ was defined as the topic covariance $\Gamma = \E_{\mu\sim\tau}\left[\mu\mu^{\top}\right]$.
	The second term is
	\begin{align*}
		\bSigma_{X_2\phi_{\bar y}}
		&= \E_{\mu\sim\tau}\left[\E\left[X_2|\mu\right] \E\left[\phi_{\bar Y}^{\top}|\mu\right]\right]
		= \E_{\mu\sim\tau}\left[A\mu\mu^{\top}\right]
		= A \Gamma
	\end{align*}
	The upper bound for $1/\beta$ can be computed as follows
	\begin{align*}
		1/\beta
		&= \left\|\bSigma_{Y\phi_{\bar y}}\bSigma_{X_2\phi_{\bar y}}^\dagger\right\|_2
		= \left\|w^{\top} \Gamma \left(A\Gamma\right)^{\dagger}\right\|_2\\
		&\le \|w\|_2 ~\lambda_{\max}(\Gamma) ~\lambda_{\max}\left(\left(A\Gamma\right)^{\dagger}\right)
		= \|w\|_2 ~\lambda_{\max}(\Gamma) ~\lambda_{\min}\left(A\Gamma\right)^{-1}\\
		&\le \|w\|_2 ~\lambda_{\max}(\Gamma) ~\lambda_{\min}\left(A\right)^{-1} ~\lambda_{\min}\left(\Gamma\right)^{-1}\\
		&= \|w\|_2 ~\frac{\lambda_{\max}(\Gamma)}{\lambda_{\min}\left(\Gamma\right)} ~\lambda_{\min}\left(A\right)^{-1}
		= \frac{\kappa \|w\|_2}{\lambda_{\min}\left(A\right)}
	\end{align*}
%	For $\lambda_{\min}\left(A\right)$, we need to lower bound $\min_{\|v\|_2=1} \|Av\|_2$ for which we will use the anchor word assumption.
%	Let $\pi(i)$ be the anchor word for topic $i\in[k]$; so $A_{i}(\pi(i)) \ge p$ and $A_{i'}(\pi(i)) = 0$ for $i' \neq i$.
%	Using the assumption, for any $v\in\R^{k}$ with $\|v\|_2=1$, we get that $\left|(Av)(\pi(i))\right| = |v_{i}| ~A_{i}(\pi(i)) \ge p |v_{i}|$ and thus $\|Av\|_2^2 \ge \sum_{i} (Av)(\pi(i))^2 \ge p^2 \sum_{i} v_{i}^2 = p^2$.
%	This shows that $\lambda_{\min}\left(A\right) \ge p$, thus combining with the above calculation to prove point 4. and completing the proof.
\end{proof}

\section{Omitted Proofs on Learning the Conditional Distribution} 
\label{sec:proof_cca_ace}

\begin{comment} 
\subsection{Unifying the Operators}
We begin with presenting all operators that maps from Hilbert space of random variables to another Hilbert space to its matrix form. For instance, if $(\cT g)(x_1) = \int_{X_2} T(x_1,x_2) g(x_2) p_{X_2}(x_2) dx_2$, we say $T$ is the matrix form of operator $\cT$.  

\begin{itemize}
	\item Representation operator $ \cT: L^2(X_2)\rightarrow L^2(X_1)$, 
	$$(\cT g)(x_1) := \E[g(X_2)|X_1=x_1], \forall g\in L^2(X_2).$$ It is to verify that its matrix form is $T(x_1,x_2) = \frac{p_{X_1,X_2}(x_1,x_2)}{p_{X_1}(x_1)p_{X_2}(x_2)}$.
	\item Low rank approximation operator $ \cL: L^2(X_2)\rightarrow L^2(X_1)$, 
	$$(\cL g)(x_1) = \E_Y[ \E_{X_2} [g(X_2)|Y] | X_1 = x_1]. $$ 
	Its matrix form is $L(x_1,x_2) = \sum_{y=1}^k  \frac{p_{X_1|Y}(x_1|y)p_{X_2|Y}(x_2|y)p(y)}{p_{X_1}(x_1)p_{X_2}(x_2)} $
	\item $ A: L^2(X_2)\rightarrow L^2(Y), (Ag)(y):=  \E[g(X_2)|Y=y]$ has matrix form: $A(x_2,y) = \frac{p_{X_2,Y}(x_2,y)}{p_{X_2}(x_2)p(y)}$
	\item $B: L^2(Y)\rightarrow L^2(X_1),  (Bh)(x_1):= \E[h(Y)|X_1=x_1] $. Its matrix form is: $B(y,x_1) = \frac{p_{X_1,Y}(x_1,y)}{p_Y(y)p_{X_1}(x_1)}$.
	\item Operator that measures conditional independence: $\cE: = \cT - \cL,$ with matrix form: 
	$ E(x_1,x_2) = \frac{p_{X_1,X_2}(x_1,x_2)-\sum_{y=1}^kp_{X_1|Y}(x_1|y)p_{X_2|Y}(x_2|y)p(y) }{p_{X_1}(x_1)p_{X_2}(x_2)} $.	
\end{itemize}
%\qi{change everything back to $L^2(X)$. It will be more confusing to change to $L^2(R^d)$ and change back.. }

%\jnote{type the 3 lines from notability here.}

%\qi{change to rademacher complexity}
\end{comment}

\subsection{Introducing the Operators on the Hilbert Spaces}

%The previous section requires $X_2$ and $Y$ to have linear correlation \ns{why do we need linear correlation? we just needed correlation matrix to be full rank} (Assumption \ref{assumption:CI_with_latent_variables}) \ns{Cite an assumption in the main paper instead}.
%As proposed in Remark \ref{remark:encoder_decoder}, we can learn $\E[g(X_2)|X_1]$ for any function $g:\cX_2\rightarrow \R$, which only requires linear correlation between $g(X_2) $ and $Y$. To avoid choosing a specific function $g$, we instead optimize over all $g$ which we will explain in the following section.
%\jnote{Move sec 5.1 here, proposition 5.5 here and ACE, and the proposoition/lemma saying its Svd. After we present the algorithm ACE, then go on to these various notation that is decomposing T.}

%Since $X_1\bot X_2|Y$ ensures $X_1\bot h(X_2)|Y$ for any function $h$, we can replace $X_2$ by $g(X_2)$ and all results hold\jnote{should say which results, I guess you mean the ones under exact CI?}.

%  Therefore we could replace $X_2$ with $h(X_2)$ in our algorithm especially when $d_2<km$.%Therefore in practice, we could use $h(\psi(X_1))$ instead of $\psi(X_1)$ for downstream task. %Specifically with denoising auto-encoder or context encoder, one could think about $h$ as the inverse of decoder $D$ ($h=D^{-1}$) and  use $D^{-1}\psi\equiv E$ the encoder function as the representation for downstream tasks, which is more commonly used in practice. 

%%%%%%%%%%%%%%%%%%%%%%%%%%%

We first introduce all the operators. They will help us to present all the theorem of Section \ref{sec:learn_joint_distribution} in a more compact way. We let $L^2(X)$ denotes the Hilbert space of square integrable function with respect to the measure $P_X$, the marginal distribution of $X$.
For instance, in our context of SSL, $L^2(X_2)=\{ g:\R^{d_2}\rightarrow \R |  \int g^2(x_2)dP_{X_2}(x_2) <\infty.   \}$.

%\jnote{This big block can't be in the main paper like this. Somehow simplify or figure out how to state main theorem without all of these definitions.}
\begin{itemize}
	\item Representation operator $ \cT: L^2(X_2)\rightarrow L^2(X_1)$, 
	$$(\cT g)(x_1) := \E[g(X_2)|X_1=x_1], \forall g\in L^2(X_2).$$
	%	\item $\bar P:L_2 (R^d) \to L_2(R^d) $ is the normalized density $P_{x_1,x_2}=p(x_1,x_2) /\sqrt{p(x_1)p(x_2)}$ and $(\bar P g)(x_1) := \int P(x_1,x_2) g(x_2) dx_2)$. $\bar P$ can be thought of as an ``infinite-matrix'' because it has the standard inner product.
	\item Low rank approximation operator $ \cL: L^2(X_2)\rightarrow L^2(X_1)$, 
	$$(\cL g)(x_1) = \E_Y[ \E_{X_2} [g(X_2)|Y] | X_1 = x_1]. $$ Under conditional independence $X_1\bot X_2|Y, \cT =\cL.$
	\begin{itemize}
		\item From the definition of $\cL$ we can decompose it into the following two operators $\cL= \cB\circ \cA$:
		\item $ \cA: L^2(X_2)\rightarrow L^2(Y), (\cA g)(y):=  \E[g(X_2)|Y=y]$
		\item $\cB: L^2(Y)\rightarrow L^2(X_1),  (\cB h)(x_1):= \E[h(Y)|X_1=x_1] $. Our final goal is to compute $\cB\circ \id= \E[Y|X_1=x_1]$, where $\id(y) =y$ is the identity map on $L^2(Y)$.
		\item $\cA^{\dagger}: L^2(Y)\rightarrow L^2(X_2)$ is the inverse operator of $\cA$.  Let $\tilde\beta:=1/\|\cA^{\dagger}\|_{\tHS}.$ This $\tilde \beta \in [ \sigma_k(\cA)/\sqrt{k},\sigma_k(\cA) ]$ where $\sigma_k(\cA)$ is the $(k-1)$-th maximal correlation between $X_2$ and $Y$. 
	\end{itemize}
	\item Operator that measures conditional independence: $\cE: = \cT - \cL,$ 
	$$ \|\cE\|_{\op} := \max_{\|g\|_{L^2(X_2)}=1} \E_{X_1} (\E[g(X_2)|X_1] - \E[\E[g(X_2)|Y]|X_1])^2 =:\tilde\epsilon_{\tCI}.  $$
	
\end{itemize}

\begin{theorem}[Theorem \ref{thm:PCR_general_main} restated]
	\label{thm:PCR_general}
	Conduct SVD on $\cT$: find $k$ orthonormal function $u_1,\cdots u_k$ in $L^2(X_1)$ and orthonormal function $v_1,\cdots v_k \in L^2(X_2)$ and scalars $\sigma_1,\cdots \sigma_k\in \R$ that minimizes:
	\begin{align*}
	L(\{u_i\},\{v_i\},\{\sigma_i\} ):= \max_{\|g\|_{L^2(X_2)}=1 } \| \cT g - \cT_k g \|_{L^2(X_1)}, \text{ where }\cT_k g:=\sum_{i=1}^k \sigma_i \langle v_i, g\rangle_{L^2(X_2)} u_i.  
	\end{align*}
	Now treat $\psi(x_1)=[u_1(x_1),\cdots u_k(x_1)]: \cX_1\rightarrow \R^k$ as the representation. Then the approximation error of $\psi$ satisfies:
	\begin{align*}
	e_{\apx}(\psi):= & \min_{\bW\in\R^{k\times k}}\E[\|f^*(X_1)-\bW^\top \psi(X_1)\|^2]\\
	\leq & \sum_{y=1}^k \min_{g_y\in L^2(X_2)} 2(\|(\cT_k - \cL)\circ g_y\|^2_{L^2(X_1)} +\|\cL \circ g_y - f^*_y\|^2_{L^2(X_1)}) . 
	\end{align*}	
	Here $f^*$ is the optimal function to predict the one-hot encoder of $Y$ with $X_2$, i.e., $f^*_y(x_1) = \E[1(Y=y) | X_1=x_1 ] = P(Y=y|X_1=x_1)$. 
\end{theorem}

%\jnote{Explain choices of $g_y$ for both corollaries.}

When we set $g_y(x_2) = \cA^{\dagger} \circ 1(Y=y)$, we have the following corollary: 
\begin{corollary}[Corollary \ref{coro:PCR_non-zero_maximal_correlation_main} restated ]
	\label{coro:PCR_non-zero_maximal_correlation}
	In the same setting of Theorem \ref{thm:PCR_general}, suppose the $(k-1)$-th maximal correlation between $X_2$ and $Y$ is not zero, then we have:
	$$ ER_{\psi}(\hat\bW) \leq \tilde O(\frac{\tilde\epsilon_{\tCI}^2}{\tilde\beta^2} + \sigma^2 \frac{k}{n_2} ) .$$	
\end{corollary}

Next we present the proof of Theorem \ref{thm:PCR_general}, Corollary \ref{coro:PCR_non-zero_maximal_correlation_main} and Corollary \ref{thm:PCR_symmetric_main}. 

\subsection{Proof of Theorem \ref{thm:PCR_general}} 
\begin{proof}[Proof of Theorem \ref{thm:PCR_general}]
First note that the representation function $\psi:\cX_1\rightarrow \R^k$ is formed by the left singular vectors of $\cT_k$, therefore for any vector $\vw \in \R^k$, there exists a corresponding $g_{\bw} \in L^2(X_2)$ such that $\psi(x_1)^\top \vw \equiv (\Psi\circ g_{\bw}) (x_1)$. In the same way, $\cT_k \circ g = \sum_{i=1}^k \sigma_i \langle v_i,g\rangle u_i =\psi^\top \bw $ where $\bw = \sigma_i \langle v_i,g\rangle$. Therefore for any $g\in L^2(X_2)$, there also exists a $\bw$ such that $\psi(x_1)^\top\bw\equiv (\cT_k \circ g)(x_1)$. 

	\begin{align*}
	\apx(\psi):= & \min_{\bW\in\R^{k\times k}}\E[\|f^*(X_1)-\psi(X_1)\bW\|^2]\\
	= & \sum_{y=1}^k  \min_{\bW\in\R^{k\times k}}\E[\|f^*_y(X_1)-\psi(X_1)^\top \bw_y\|^2] \tag{$\bw_y$ is the $y$-th column vector of $\bW$}\\
	= &  \sum_{y=1}^k  \E[\|f^*_y(X_1)- (\cT_k\circ g_{\bw_y} )(X_1) \|^2]\\
	= &   \sum_{y=1}^k \min_{g_y\in L^2(X_2)} \E[\|f^*_y(X_1)- (\cT_k\circ g_{y} )(X_1) \|^2] \\
	=&  \sum_{y=1}^k \min_{g_y\in L^2(X_2)} \E[\|(f^*_y(X_1) - \cL\circ g_{y}) - ((\cT_k-\cL)\circ g_{y} )(X_1) \|^2]\\
	\leq & \sum_{y=1}^k \min_{g_y\in L^2(X_2)} 2(\|(\cT_k - \cL)\circ g_y\|^2_{L^2(X_1)} +\|\cL \circ g_y - f^*_y\|^2_{L^2(X_1)})\tag{By AM-GM}. 
	\end{align*}	
\end{proof}

\begin{claim}
	\label{claim:near_block_diagonal}
The joint distribution	$p_{X_1,X_2}(x_1,x_2)$ satisfies:
	\begin{align*}
	\int_{X_1,X_2} p_{X_1,X_2}(x_1,x_2)1(g_1^*(x_1)\neq g_2^*(x_2))\leq 2\alpha. 
	\end{align*}
Let functions $w_{1,y}(x_1) = 1(g_1^*(x_1)=y)\in L^2(\cX_1) $, and $w_{2,y}(x_2) = 1(g_2^*(x_2)=y)\in L^2(\cX_2) , \forall y \in [k]$. 
Then we have that:
\begin{align*}
\sum_{y}\langle \cT w_{2,y}, w_{1,y}\rangle \geq 1-2\alpha.
\end{align*}	
\end{claim}
\begin{proof}
\begin{align}
\notag 
& \int_{X_1,X_2} p_{X_1,X_2}(x_1,x_2) 1(g(x_1)\neq g(x_2)) \\
\notag =& \int_{X_1,X_2} \int_Y p_{X_1,x_2,Y}(x_1,x_2,y) 1(g_1^*(x_1)\neq g_2^*(x_2)) \\
\notag \leq &\int_{X_1,X_2} \int_Y p_{X_1,x_2,Y}(x_1,x_2,y) \left(1(g_1^*(x_1)\neq y) + 1(g_2^*(x_2) \neq y) \right)\\
\notag = & \int_{X_1,Y} p_{X_1,Y}(x_1,y) 1(g_1^*(x_1\neq y)) + \int_{X_2,Y} p_{X_2,Y}(x_2,y) 1(g_2^*(x_2)\neq y) \\
\label{eqn:block_diagonal}
= & P(g_1^*(x_1)\neq y) + P(g_2^*(x_2)\neq y)\leq 2\alpha. 
\end{align}
Meanwhile, 
\begin{align*}
& \sum_{y}\langle \cT w_{2,y}, w_{1,y}\rangle\\
= & \sum_{y} \int_{X_1}\left(\int_{X_2} T(x_1,x_2)w_{2,y}(x_2)p_{X_2}(x_2)dx_2\right) w_{1,y}(x_1)p_{X_1}(x_1)dx_1 \\
= & \sum_{y} \int_{X_1,X_2} 1(g_1^*(x_1)=y) 1(g_2^*(x_2 ) = y)  p_{X_1,X_2}(x_1, x_2)  \tag{since $T(x_1,x_2):=\frac{p_{X_1,X_2}(x_1,x_2)}{p_{X_1}(x_1)p_{X_2}(x_2)}$}  \\
= & \int_{g_1^*(X_1)=g_2^*(X_2)}p_{X_1,X_2}(x_1,x_2)  \\
= &  1-  \int_{X_1,X_2} p_{X_1,X_2}(x_1,x_2) 1(g(x_1)\neq g(x_2))  \\
  \geq & 1-2\alpha. \tag{from Ineq. \eqref{eqn:block_diagonal}}
\end{align*}		
\end{proof}

\begin{claim}
	\label{claim:T_norm1} 
The top eigenvalue of $T$ is $1$. 
\end{claim}
\begin{proof}
First we show that $\|\cT\|_{op}:=\max_{u\neq 0} \frac{\|\cT u\|_{L^2(X_1)}}{\|u\|_{L^2(X_2)} }\leq 1$. 
For any $ u \in L^2(R^d)$, we have that
\begin{align*}
\|\cT u\|^2 = & \|\E[u(X_2)|X_1]\|^2_{L^2(X_1)}  \\
 = & \int_{x_1}  \E[u(X_2)|X_1]^2  p_{X_1}(x_1)  d x_1\\
 \leq & \int_{x_1}  \E[u^2(X_2)|X_1]  p_{X_1}(x_1)  d x_1 \tag{Jensen's inequality that $\E^2[X]\leq \E[X^2]$} \\
 = & \E[u^2(X_2)] = \|u\|^2_{L^2(X_2)}. 
\end{align*}
Second, let $u(x_2)\equiv 1$ and $v(x_1)\equiv 1$, we have $\int_{x_1} T(x_1,x_2)u(x_2)dx_2 = 1 = v(x_1). $ Therefore we have $\| T u\|_{L^2(X_1)} =1$ for $u=1$ and $\|u\|_{L^2(X_2)}=1$. Therefore $\|T\|_{\op}= 1$. 
\end{proof}
\begin{lemma}
\label{lemma:Ew_small} 
Let $w_{1,y}, w_{2,y} ,\forall y\in [k]$ be the same from Lemma \ref{claim:near_block_diagonal}. Then we have: 
\begin{align*}
\sum_y \langle \cL  w_{2,y}, w_{1,y}\rangle \geq 1-2\alpha.  
\end{align*}
	Therefore $\sum_y \|\cL w_{2,y} - w_{1,y}\|^2 \leq 4\alpha. $
\end{lemma}
\begin{proof}
\begin{align*}
& \sum_y \langle \cL  w_{2,y}, w_{1,y}\rangle\\
= & \sum_y   \sum_h \int_{x_1,x_2} p(x_1|h) p(x_2|h)p(h) 1(g_1^*(x_1)=y)1(g_2^*(x_2)=y) dx2dx_1\\
= & \sum_h \int_{x_1,x_2} p(x_1|h) p(x_2|h)p(h) 1(g_1^*(x_1)=g_2^*(x_2)) dx2dx_1\\
= & \sum_h \int_{x_1,x_2} p(x_1|h) p(x_2|h)p(h) (1- 1(g_1^*(x_1)\neq g_2^*(x_2)) )  dx2dx_1\\
= & \sum_h \int_{x_1,x_2} p(x_1|h) p(x_2|h)p(h) dx_2dx_1 - \sum_h \int_{x_1,x_2} p(x_1|h) p(x_2|h)p(h)1(g_1^*(x_1)\neq g_2^*(x_2)) dx_2dx_1\\
= & 1-\sum_h \int_{x_1,x_2} p(x_1|h) p(x_2|h)p(h)1(g_1^*(x_1)\neq g_2^*(x_2)) dx_2dx_1.
\end{align*}
\begin{align*} 
&\sum_h \int_{x_1,x_2} p(x_1|h) p(x_2|h)p(h)1(g_1^*(x_1)\neq g_2^*(x_2)) dx_2dx_1\\
\leq &  \sum_y \int_{x_1,x_2} p(x_1|y) p(x_2|y)p(y) (1(g_1^*(x_1)\neq y) + 1( g_2^*(x_2) \neq y) ) dx2dx_1\\
= & \sum_y \left(\int_{x_1} p(x_1|y) \int_{x_2} p(x_2,h) 1(g_2^*(x_2)\neq y) dx_1 +  \int_{x_2} p(x_2|y) \int_{x_1} p(x_1,h) 1(g_1^*(x_1)\neq y) dx_2 \right)\\
= & \sum_y( P_{X_1,Y}(g_1^*(x_1\neq y)) + P_{X_1,Y}(g_1^*(x_1\neq y))  )\\
 \leq & 2\alpha.  
\end{align*}
Therefore $\sum_y \langle \cL w_{2,y}, w_{1,y}\rangle \geq 1-2\alpha$. 
$\sum_y \|\cL w_{2,y}- w_{1,y}\|^2 = \sum_y (\|\cL w_{2,y}\|^2 + \|w_{1,y}\|^2  -2 \langle w_{1,y},\cL w_{2,y}\rangle) \leq 2-2(1-2\alpha)=4\alpha$.
\end{proof} 

\begin{lemma}
\label{lemma:T_k_small_perturbation}
Let $T_k(x_1,x_2)$ be the rank-$k$ approximation of  $T(x_1,x_2)$, i.e.,
$T_k(x_1,x_2) = \sum_{i=1}^k \sigma_i u_i(x_1) v_i(x_2)  $, where $u_i\in L^2(\cX_1), v_i\in L^2(\cX_2)$. 
Then with the same definition of $w_{1,y}$ and $w_{2,y}$ as Claim \ref{claim:near_block_diagonal}, we have that:
\begin{align*}
\sum_{y=1}^k \|\cT_k w_{2,y} - w_{1,y}\|^2  \leq \frac{16\alpha}{1-\lambda_{k+1}^2},
\end{align*}
where $\lambda_{k+1}$ is the ($k+1$)-th singular value of $\cT$, i.e., the $k$-th maximal correlation between $X_1$ and $X_2$
\end{lemma} 
\begin{proof}
	First, we have that $\sum_y \E[w_{2,y}^2(X_2)] = \sum_y P_{X_2}(g_2^*(X_2)=y) = 1. $

Second from Claim \ref{claim:T_norm1} we know that $\|T\|_{op}:=\max_{\|u\|=1}\|Tu\|=1$.  Also, as we defined that $T=L+E$ with $L$ of rank $k$ and  $\tilde\epsilon_{\tCI}:=\|E\| $, we have $|\lambda_{k+1}|\leq \epsilon_{\tCI}$. 

Write the full decomposition of $T$ as $T(x_1,x_2) = \sum_{i=1}^{\infty} \lambda_i u_i(x_1)v_i(x_2)$. 
We have that:
\begin{align*}
1-2\alpha \leq & \sum_y \langle \cT w_{2,y}, w_{1,y}\rangle\\
 \leq &  \sqrt{\sum_y \|\cT w_{2,y}\|^2} \sqrt{\sum_y \|w_{1,y}\|^2}.
\end{align*}
Therefore $ \sqrt{ \sum_y \|\cT w_{2,y}\|^2}\geq 1-2\alpha. $

Meanwhile,
\begin{align*}
\sum_y \|\cT w_{2,y} \|^2 = & \sum_y ( \| \cT_k w_{2,y} \|^2 + \|(\cT-\cT_k ) w_{2,y}\|^2 ) \\
= & \sum_y ( \| \cT_k P_{\cT_k}w_{2,y} \|^2 + \|(\cT-\cT_k )  P_{\cT_k}^\perp w_{2,y}\|^2 ) \\
\leq & \sum_y ( \|P_{\cT_k} w_{2,y}\|^2 + \lambda_{k+1}^2 (\|w_{2,y}\|^2- \| P_{\cT_k} w_{2,y} \|^2  )  \tag{since $\|\cT\|_{\op} = 1$ and $\|\cT-\cT_k\| = \lambda_{k+1}$} \\
= & (1-\lambda_{k+1}^2) (\sum_y \|P_{\cT_k}w_{2,y}\|^2) + \lambda_{k+1}^2 \tag{since $\sum_y \|w_{2,y}\|^2 =1 $.}
\end{align*}
Therefore $\sum_y \|P_{\cT_k}w_{2,y}\|^2 \geq \frac{(1-2\alpha)^2-\lambda_{k+1}^2}{1-\lambda_{k+1}^2}$ and 
\begin{align*}
\sum_y \|(\cT-\cT_k)w_{2,y}\|^2 \leq & \lambda_{k+1}^2(1-\sum_y \|P_{\cT_k}w_{2,y}\|^2)\\
\leq & \lambda_{k+1}^2(1-\frac{(1-2\alpha)^2-\lambda_{k+1}^2}{1-\lambda_{k+1}^2})\\
= & \frac{4\alpha(1-\alpha) \lambda_{k+1}^2}{1-\lambda_{k+1}^2}. 
\end{align*}
Finally, on one hand we have 
\begin{align*}
\sum_y \|\cT w_{2,y} -w_{1,y}\|^2
= & \sum_y \|\cT w_{2,y}\|^2 + \|w_{1,y}\|^2 -2\langle \cT w_{2,y}, w_{1,y}\rangle \\
\leq 2-2(1-2\alpha) = 4\alpha. 
\end{align*}
On the other hand we have:
\begin{align*}
\sqrt{\sum_y \|\cT_k w_{2,y}- w_{1,y} \|^2 } \leq & \sqrt{\sum_y \|\cT_k w_{2,y}- w_{1,y} \|^2 } + \sqrt{\sum_y\|(\cT-\cT_k) w_{2,y}\|^2} \\
\leq & 2\sqrt{\alpha} + \sqrt{\frac{4\alpha(1-\alpha)}{1-\lambda_{k+1}^2} }\\
\leq & \frac{4\sqrt{\alpha}}{\sqrt{1-\lambda_{k+1}^2}}. 
\end{align*}
Therefore $ \sum_y \|\cT_k w_{2,y}- w_{1,y} \|^2 \leq \frac{16\alpha}{1-\lambda_{k+1}^2}$. %\leq \frac{16\alpha}{1-\epsilon_{\tCI}^2}$ . 
\end{proof}

\begin{proof}[Proof of Corollary \ref{coro:PCR_non-zero_maximal_correlation}]
This is the corollary from Theorem \ref{thm:PCR_general} by taking $g^*_i(y) = \cA^{\dagger} \circ 1(y=i)$ such that $\cL\circ g^*_i \equiv f^*_i,\forall i\in [k]$. 	
This is because $\cL = \cB \circ \cA,$ and $\cL \circ \cA^\dagger \circ 1(y=i) = \cB \circ \id = \E[Y=i|X_1] = f^*_y$. 

Therefore the second term is $0$ in Theorem \ref{thm:PCR_general} and it remains to prove that the first term is small. 

Notice 
\begin{align*}
& \E_{X_1} \|\bar \cE \circ g^*(X_1)\|^2 \\
 = & \|\bar \cE\circ g^* \|^2_{L^2(X_1)} \\
 \leq &  \|\bar \cE\|_{\op}^2 \|\cA^{\dagger}\|_{\op}^2 \sum_y \|1(Y=y)\|^2_{L^2(Y)}\\
 \lesssim & \tilde \epsilon^2_{\tCI}/\tilde\beta^2. 
\end{align*}
Therefore the approximation error is upper bounded by $\tilde\epsilon^2_{\tCI}/\tilde\beta^2$.

\end{proof}

\begin{proof}[Proof of Corollary \ref{thm:PCR_symmetric_main}]

%Similar as in the Proof of Theorem \ref{thm:PCR_general}, we let $g_{\hat \bw}$ corresponds to $\hat\bw $ such that $T_k\circ g_{\hat\bw} = \hat \bw^\top \psi$. 
With Theorem \ref{thm:PCR_general} and we take $g_y(x_2) = w_{2,y}(x_2) = 1(g_2^*(x_2)=y),\forall y\in [k]$ as in Lemma \ref{lemma:Ew_small}. We only need to upper bound
$$ \E_{X_1}\|f^*_y - \cL\circ w_{2,y}\|^2 + \|(\cL-\cT_k)\circ w_{2,y} \|^2. $$
Notice that
\begin{align*}
& \sum_y \E_{X_1} \|(\cL -\cT_k) w_{2,y}\|^2 \\
= & \sum_y \E_{X_1} \|(\cL\circ w_{2,y} - w_{1,y})+(w_{1,y} -\cT_k\circ w_{2,y}) \|^2\\
\leq & 2 \sum_y \E_{X_1} (\|\cL\circ w_{2,y} - w_{1,y}\|^2 +  \|(w_{1,y} -\cT_k\circ w_{2,y}) \|^2)\\
\leq & \frac{16\alpha}{1-\lambda_k^2} +  4\alpha \tag{from Lemma \ref{lemma:T_k_small_perturbation} and \ref{lemma:Ew_small}}.  
\end{align*}
Meanwhile, the other term is
\begin{align*}
& \sum_y \E_{X_1} \|f^*_y - \cL\circ w_{2,y}\|^2\\
\leq & 2 \sum_y \E_{X_1} \|f^*_y - w_{1,y}\|^2 + \|w_{1,y}- \cL\circ w_{2,y}\|^2 \\
\leq & 2 \sum_y \E_{X_1} \|f^*_y - w_{1,y}\|^2 + 8\alpha \tag{from Lemma \ref{lemma:Ew_small}}\\
= & 8\alpha + 2 \sum_y \int_{x_1}( p(y|x_1) - 1(g_1^*(x_1)=y ))^2 p_{X_1}(x_1) dx_1 \\
= & 8\alpha + 2 \sum_y \int_{x_1} p^2(y|x_1) p_{X_1}(x_1) + 1(g_1^*(x_1)=y )^2 p_{X_1}(x_1) -2 \cdot 1(g_1^*(x_1)=y ) p(y|x_1)p_{X_1}(x_1)  dx_1\\
\leq & 8\alpha + 2 \sum_y \int_{x_1} p(y|x_1) p_{X_1}(x_1) + 1(g_1^*(x_1)=y ) p_{X_1}(x_1) -2 \cdot 1(g_1^*(x_1)=y ) p(y|x_1)p_{X_1}(x_1)  dx_1 \tag{ since $p(y|x_1)\leq 1$}\\
= & 8\alpha + 2( 2 - 2 \sum_y 1(g_1^*(x_1)=y ) p(y|x_1)p_{X_1}(x_1)  dx_1 ) = 8\alpha  + 4P_{X_1,Y}(g_1^*(x_1)\neq y) \\
\leq & 12\alpha \tag{since Bayes error is bounded by $\alpha$.}
\end{align*}
Altogether we have the approximation error is upper bounded by $O(\frac{\alpha}{1-\lambda_k^2})$.

\end{proof}
% !TEX root = colt2021-SSL.tex

\section{General Results and Comparison to \cite{tosh2020contrastive_1}}
\label{apx:general_tosh}

We now show a more general form of our results and also connect the multi-view redundancy assumption from \cite{tosh2020contrastive_1} to ours.
%For simplicity we assume $Y$ is a scalar variable; everything remains the same for vector valued $Y$, except that square of values are replaced by square of norm.

\subsection{General Results}
\label{apx:general}

We first note that all our results hold for a generalized version of Assumption~\ref{assumption:feature_map_approx_CI} and Definition~\ref{def-approx-CI} that we state below.
\begin{assumption}
	\label{assumption:feature_map_approx_CI_relaxed}
	Suppose $\bar{Y}$ with $|\bar{Y}|\le m$ is a discrete latent variable that satisfies
	\begin{enumerate}
		\item $\bar{Y}$ makes $X_1$ and $X_2$ approximately CI as in Definition~\ref{def-approx-CI}, i.e. $$\epsilon_{\tCI}^2:=\E_{X_1}\left[\|\E[X_2|X_1]-\E_{\bar Y}[\E[X_2|\bar Y]|X_1]\|^2\right]$$
		\item $\bar{Y}$ also makes $X_1$ and $Y$ approximately CI with $$\epsilon_{\bar{Y}}^2:=\E_{X_1}\left[\|\E[Y|X_1]-\E_{\bar Y}[\E[Y|\bar Y]|X_1]\|^2\right]$$
		\item $\bSigma_{\phi_{\bar y}X_2}\text{ is full column rank and } \|\bSigma_{Y\phi_{\bar y}}\bSigma_{X_2\phi_{\bar y}}^\dagger\|_2 = 1/\beta$, where $A^\dagger$ is pseudo-inverse, and  $\phi_{\bar y}$ is the one-hot embedding for $\bar Y$.
	\end{enumerate}
\end{assumption}

Note that our assumptions from the main paper are a special case of Assumption~\ref{assumption:feature_map_approx_CI_relaxed}, with $\epsilon_{\bar{Y}} = 0$ being satisfied automatically as $\bar{Y} = [Y,Z]$ is explicitly defined to contain $Y$ in it. Unlike Assumption~\ref{assumption:feature_map_approx_CI}, we do not need $Y$ to be a discrete variable, but just need $\bar{Y}$ to be discrete. We state the generalization of Theorem~\ref{thm:main_result_approximate_CI} below

\begin{theorem}
	\label{thm:main_result_approximate_CI_relaxed}
 For a fixed $\delta\in (0,1)$, under Assumptions \ref{assumption:feature_map_approx_CI_relaxed}, \ref{assumption:bounded_error_non_universal} for $\tilde{\psi}$ and $\psi^*$ and \ref{assumption:subgaussian_psi} for non-universal feature maps, if $n_1,n_2\gg \rho^4(d_2+\log 1/\delta)$, and we learn the pretext tasks such that:
$ \E\|\tilde \psi(X_1)-\psi^*(X_1)\|_F^2\leq  \epsilon^2_{\pre}.$
Then the generalization error for downstream task w.p. $1-\delta$ is: 
\begin{align}
 \label{eqn:main_result_approx_CI_relaxed}
\E_{X_1} \left[\|\E[Y|X_1] - \hat\bW^\top \tilde{\psi}(X_1) \|_2^2\right] \leq \tilde\cO\left(\sigma^2\frac{d_2}{n_2} + \frac{\epsilon^2_{\tCI}}{\beta^2} +\frac{\epsilon_{\pre}^2}{\beta^2 } + \epsilon_{\bar{Y}}^2\right)%\revise{ + e_{\apx}(\phi_1)?}.
\end{align} 
\end{theorem}
The result is pretty much the same as Theorem~\ref{thm:main_result_approximate_CI}, except for an additional term of $\epsilon_{\bar{Y}}^2$.
The proof is also very similar, the difference being that $\E[\E[Y|\bar{Y}] | X_1]$ can now be expressed as a linear function of $\psi^*$ instead of $\E[Y | X_1]$, and the additional error incurred during to the mismatch between $\E[Y | X_1]$ and $\E[\E[Y|\bar{Y}] | X_1]$ that is $\epsilon_{\bar{Y}}^2$ will be incurred.

\subsection{Comparison to \cite{tosh2020contrastive_1}}
\label{apx:tosh}

We show guarantees for our algorithm under the assumption from \cite{tosh2020contrastive_1} in the following special case that satisfies: (1) $X_1$ and $X_2$ are {\em exactly} CI given $\bar{Y}$ (thus $\epsilon_{\tCI}=0$), (2) the variation in the target $Y$ is small given $X_1$ and $X_2$.
The assumption from \cite{tosh2020contrastive_1}, in our setting, is equivalent to saying that $\epsilon_{X_1}$ and $\epsilon_{X_2}$ are small, where
\begin{align*}
	\epsilon^2_{X_i} = \E\left[\|\E[Y|X_i] - \E[Y|X_1,X_2]\|^2\right], ~~ i\in\{1,2\}
\end{align*}
A similar assumption of multi-view redundancy also appears in \cite{tsai2020demystifying}; however they state it in terms of information-theoretic quantities instead.
We will show that these assumptions are also almost sufficient to show results in our setting.
In particular we show that if $Y|X_1,X_2$ is almost deterministic (which makes sense for a many regression tasks) and if $\epsilon^2_{X_2}$ is small, then $\epsilon_{\bar{Y}}$ defined in the previous subsection will be small and thus we have meaningful guarantees.

\begin{lemma}
\label{lem:tosh_is_useful}
	Let $\sigma^2_{Y}= \Var[Y|X_1,X_2]$ be the variance of $Y$. $\bar{Y}$ is as defined in Assumption~\ref{assumption:feature_map_approx_CI_relaxed} with the extra condition that $X_1$ and $X_2$ are exactly CI given $\bar{Y}$. Then we have
	\begin{align*}
		\epsilon_{\bar{Y}} \le \sqrt{2}(\sigma_{Y} + \epsilon_{X_2})
	\end{align*}
\end{lemma}
Plugging this into Theorem~\ref{thm:main_result_approximate_CI_relaxed} will give us the desired result. Note however that we did not even use the fact that $\epsilon_{X_1}$ is small. Using this part of the assumption, we can get an even stronger result that shows that even though our learned representation will only $X_1$, if will still predict $Y|X_1,X_2$ well.
\begin{corollary}
	\label{cor:main_result_approximate_CI_relaxed}
 For a fixed $\delta\in (0,1)$, under Assumptions \ref{assumption:feature_map_approx_CI_relaxed}, \ref{assumption:bounded_error_non_universal} for $\tilde{\psi}$ and $\psi^*$ and \ref{assumption:subgaussian_psi} for non-universal feature maps, if $n_1,n_2\gg \rho^4(d_2+\log 1/\delta)$, and we learn the pretext tasks such that:
$ \E\|\tilde \psi(X_1)-\psi^*(X_1)\|_F^2\leq  \epsilon^2_{\pre}.$
Then the generalization error for downstream task w.p. $1-\delta$ is: 
\begin{align*}
% \label{eqn:main_result_approx_CI_relaxed}
	\E_{X_1,X_2} \left[\|\E[Y|X_1,X_2] - \hat\bW^\top \tilde{\psi}(X_1) \|_2^2\right] \leq \tilde\cO\left(\sigma^2\frac{d_2}{n_2} +\frac{\epsilon_{\pre}^2}{\beta^2 } + \epsilon_{\bar{X_1}}^2 + \epsilon_{\bar{X_2}}^2 + \sigma_{Y}^2\right)
\end{align*} 
\end{corollary}
Thus we see that the assumption from \cite{tosh2020contrastive_1} is strong enough for us to be able to show stronger results than just our assumption.
We complete this section by proving Lemma~\ref{lem:tosh_is_useful}
\begin{proof}[Lemma~\ref{lem:tosh_is_useful}]
	We will also make use of the following lemma that is easily proved using Cauchy-Schwarz inequality
	\begin{lemma}
		For random variables $Z_1, \dots, Z_n$ for which $\E[\|Z_i\|^2]<\infty$ for every $i\in[n]$, we have
		\begin{align*}
			\E[\|Z_1 + \dots + Z_n\|^2] \le \left(\sqrt{\E[\|Z_1\|^2]} + \dots + \sqrt{\E[\|Z_n\|^2]}\right)^2
		\end{align*}
	\end{lemma}
	The proof follows from the following sequence of inequalities that uses Jensen's inequality, conditional independence of $X_1$ and $X_2$ and the above lemma. For simplicity we assume that $Y$ is a scalar random variable, the proof is the same for vector values $Y$, except squared values will replaced by norm squared values.
	\begin{align*}
		\epsilon_{\bar{Y}}^2 
		&= \E_{X_1}\left[(\E[Y|X_1] - \E_{\bar{Y}}[\E[Y|\bar{Y}]|X_1])^2\right] = \E_{X_1}\left[(\E_{\bar{Y}}[\E[Y|\bar{Y},X_1]|X_1] - \E_{\bar{Y}}[\E[Y|\bar{Y}]|X_1])^2\right]\\
		&\le \E_{X_1,\bar{Y}}\left[(\E[Y|X_1,\bar{Y}] - \E[Y|\bar{Y}])^2\right]\\
		&= \E_{\bar{Y}}\E_{X_1|\bar{Y}}\E_{X'_1|\bar{Y}}\left[(\E[Y|X_1,\bar{Y}] - \E[Y|X'_1,\bar{Y}])^2\right]\\
		&= \frac{1}{2} \E_{\bar{Y}}\E_{X_1|\bar{Y}}\E_{X'_1|\bar{Y}}\left[(\E_{X_2}[\E[Y|X_1,X_2,\bar{Y}]|\bar{Y}] - \E_{X_2}[\E[Y|X'_1,X_2,\bar{Y}]|\bar{Y}])^2\right]\\
		&\le \frac{1}{2}\E_{\bar{Y}}\E_{X_1|\bar{Y}}\E_{X'_1|\bar{Y}}\E_{X_2|\bar{Y}}\left[(\E[Y|X_1,X_2,\bar{Y}] - \E[Y|X'_1,X_2,\bar{Y}])^2\right]\\
		&= \frac{1}{2}\E\left[(Z_1 + Z_2 + Z_3 + Z_4)^2\right]
	\end{align*}
	where $Z_1 = \E[Y|X_1,X_2,\bar{Y}] - \E[Y|X_1,X_2]$, $Z_2 = - \E[Y|X'_1,X_2,\bar{Y}] + \E[Y|X'_1,X_2]$, $Z_3 = \E[Y|X_1,X_2] - \E[Y|X_2]$ and $Z_4 = - \E[Y|X'_1,X_2] + \E[Y|X_2]$.
	The first and third inequality follow from Jensen's inequality, second inequality follows from $\E[(X-\E[X])^2] = \frac{1}{2} \E[(X-X')^2]$, and the third equality follows from the CI assumption.

	We will bound $\E[Z_1^2] = \E[Z_2^2] \le \E[(\E[Y|X_1,X_2,\bar{Y}] - \E[Y|X_1,X_2])^2] \le \E[(Y - \E[Y|X_1,X_2])^2] = \sigma^2_{Y}$ again from Jensen's inequality.
	$Z_3$ and $Z_4$ can be handled by observing that $\E[Z_3^2] = \E[Z_4^2] = \E[(\E[Y|X_1,X_2] - \E[Y|X_2])^2] = \epsilon^2_{X_2}$.
	
	Thus using the above lemma, we get the desired upper bound on $\epsilon_{\bar{Y}}$.
\end{proof}

% !TEX root = colt2021-SSL.tex

\section{Showing $\E[Y|X_1]\approx\E[Y|X_1,X_2]$}
\label{sec:y_given_x1_x2}

Our main result Theorem~\ref{thm:main_result_approximate_CI} shows that self-supervised learning can help approximate $\E[Y|X_1]$ as a linear function of the learned features $\tilde{\psi}$.
In practice, however, it is more common to predict the label $Y$ using the entire input $X=(X_1, X_2)$ rather than just $X_1$.
We show here that learning $\E[Y|X_1]$ is sufficient, under mild assumptions on the task being solved: the Bayes error of the classification task $(X_1, Y)$ is low.
We first upper bound the discrepancy between $\E[Y|X_1]$ and $\E[Y|X_1,X_2]$ based on the Bayes error rate.
%% Lemma
\begin{lemma}
	\label{lemma:bayes}
Suppose $\|Y\|\le1$ and $k=|\cY|$. Denote the Bayes error for distribution $P_{X_1,Y}$ to be $\text{Bayes-error}(P_{X_1,Y}) = \E_{X_1} \left[1 - \max_{y} P(y|X_1)\right]$\footnote{We abuse notation and use $P(y|X_1)$ instead of $P_{X_1,Y}(y|X_1)$.}. Then we have
\begin{align*}
	\E_{X_1,X_2} \left[\|\E[Y|X_1] - \E[Y|X_1,X_2]\|^2\right] \le 2k~\text{Bayes-error}(P_{X_1,Y})
\end{align*}
%\begin{align*}
%	\E_{X_1,X_2} \left[\|\E[Y|X_1] - \E[Y|X_1,X_2]\|^2\right] \le 1 - \E_{X_1} \left[\|P_{\bar{Y}|X_1}(\cdot|X_1)\|^2\right]
%	\le 1 - \E_{X_1} \left[\max_{\bar{y}}P_{\bar{Y}|X_1}(\bar{y}|X_1)^2\right]
%\end{align*}
%where $\|P_{\bar{Y}|X_1}(\cdot|X_1)\|^2 = \sum_{\bar{y}} P_{\bar{Y}|X_1}(\bar{y}|X_1)^2$. 
\end{lemma}
%We can conclude from this that if $Y$ partitions $\cX_1$, i.e. $\max_{\bar{y}}P(\bar{y}|X_1) \approx 1$ \revise{($\max_{y}P(y|X_1) \approx 1$)} for most $X_1$, 
We will show below (for $\cH=\cH_u$) that if $P_{X_1,Y}$ has low Bayes error, then predicting $\E[Y|X_1]$ is as good as predicting $\E[Y|X_1,X_2]$ up to this small additive error.
%\revise{Results can also be shown for $\bar{Y}$ instead of $Y$. Though in our formulation, results with $Y$ require a weaker assumption that $Y$ (rather than $\bar{Y}$) partitions $\cX_1$. $k$ and $\|Y\|\le1$ in the above result can be replaced with spectral norm of matrix of $Y$s stacked together}
%%
\begin{theorem}
	\label{thm:extension}
Suppose $\epsilon_{\text{Bayes}} = \text{Bayes-error}(P_{X_1,Y})$ and that $\tilde{\psi}$ is $\epsilon^2_{\pre}$-optimal on the SSL task (as in Theorem~\ref{thm:main_result_approximate_CI}).
Under the same conditions as Theorem~\ref{thm:main_result_approximate_CI}, with probability $1-\delta$ we have
\begin{align*}
% \label{eqn:main_result_approx_CI}
	\E_{X_1,X_2} \left[\|\E[Y|X_1,X_2] - \hat\bW^\top \tilde{\psi}(X_1) \|_2^2\right] \leq \tilde\cO\left(\sigma^2\frac{d_2}{n_2} + \frac{\epsilon^2_{\tCI}}{\beta^2} +\frac{\epsilon_{\pre}^2}{\beta^2 }\right) + 2\epsilon_{\text{Bayes}}
\end{align*} 
%where $f^*(X_1,X_2)$ is the optimal predictor $\E[Y|X_1,X_2]$
\end{theorem}
\begin{proof}
	The law of total expectation gives $\E_{X_2}[\E[Y|X_1,X_2]|X_1] = \E[Y|X_1]$, thus it is easy to obtain the following decomposition
	\begin{align*}
		\E_{X_1,X_2} \left[\|\E[Y|X_1,X_2] - \hat\bW^\top \tilde{\psi}(X_1) \|_2^2\right] = 
		&\E_{X_1} \left[\|\E[Y|X_1] - \hat\bW^\top \tilde{\psi}(X_1) \|_2^2\right]\\
		& + \E_{X_1,X_2} \left[\|\E[Y|X_1] - \E[Y|X_1,X_2] \|_2^2\right]
	\end{align*}
	The first term can be upper bounded using Theorem~\ref{thm:main_result_approximate_CI}: $\E_{X_1} \left[\|\E[Y|X_1] - \hat\bW^\top \tilde{\psi}(X_1) \|_2^2\right] = \ER_{\tilde \psi}(\hat \bW) \leq \tilde\cO\left(\sigma^2\frac{d_2}{n_2} + \frac{\epsilon^2_{\tCI}}{\beta^2} +\frac{\epsilon_{\pre}^2}{\beta^2 }\right)$.
	The second term is upper bounded by $2\epsilon_{\text{Bayes}}$ by invoking Lemma~\ref{lemma:bayes}, and this completes the proof
\end{proof}

\begin{proof}[Proof of Lemma~\ref{lemma:bayes}]
	Notice the following inequality
	\begin{align*}
		\E_{X_1,X_2} &\left[\|\E[Y|X_1] - \E[Y|X_1,X_2]\|^2\right]
		= \E_{X_1,X_2} \left[\left\|\sum_{y\in\cY} y\left(P(y|X_1) - P(y|X_1,X_2)\right)\right\|^2\right]\\
		&\le |\cY|(\max_{y}\left\|y\right\|^2)\E_{X_1,X_2} \left[\sum_{y}\left(P(y|X_1) - P(y|X_1,X_2)\right)^2\right]\\
		&\le k\E_{X_1} \left[\E_{X_2}\left[\sum_{y}\left(P(y|X_1) - P(y| X_1,X_2)\right)^2\large\mid X_1\right]\right]
	\end{align*}
	where the first inequality follows from Cauchy-Schwartz and second inequality follows from $\|Y\|\le1$.
	Thus the problem reduces to bounding the inner expectation for every $X_1$.
	We first note that for every $X_1,y$, we have $P(y|X_1) = \E_{X_2} [P(y|X_1,X_2)|X_1]$ from the law of total expectation.
	This gives
%	Furthermore using $P(\bar{y}|X_1,X_2)\in[0,1]$, we can show the following sequence of inequalities
	\begin{align*}
		\E_{X_2}&\left[\sum_{y}\left(P(y|X_1) - P(y| X_1,X_2)\right)^2\large\mid X_1\right]
		= \sum_{y} \E_{X_2}\left[P(y|X_1,X_2)^2 | X_1\right] - P(y|X_1)^2\\
		&\le \sum_{y} \E_{X_2}\left[P(y|X_1,X_2) | X_1\right] - P(y|X_1)^2 = \E_{X_2}\left[\sum_{y} P(y|X_1,X_2) | X_1\right]  - \sum_{y} P(y|X_1)^2\\
		&= 1 - \sum_{y} P(y|X_1)^2 \le 1 - \max_{y} P(y|X_1)^2 \le 2(1 - \max_{y} P(y|X_1))\\
	\end{align*}
	where the first inequality follows because $P(y|X_1,X_2)\in[0,1]$ and second follows trivially and third follows from $1-x^2 \le 2(1-x)$ for $x\in[0,1]$.
	Combining everything, we get $\E_{X_1,X_2} \left[\|\E[Y|X_1] - \E[Y|X_1,X_2]\|^2\right] \le 2k\E_{X_1}\left[1 - \max_{y} P(y|X_1)\right] = 2k~\text{Bayes-error}(P_{X_1,Y})$, thus proving the result.
\end{proof}

\section{Theoretical analysis for classification tasks}
\label{sec:classification}

\subsection{Classification tasks}\label{subsec:clf_tasks}
We now consider the benefit of learning $\psi$ from a class $\cH_1$ on linear classification task for label set $\cY=[k]$.
%Denote $\Delta_\cY$ to be the set of distributions over $\cY$.
The performance of a classifier is measured using the standard logistic loss
\begin{definition}
	For a task with $\cY=[k]$, classification loss for a predictor $f:\cX_1\rightarrow\R^k$ is
	\begin{align*}
	\ell_{\text{clf}}(f) = \E [\ell_{\text{log}}(f(X_1), Y)]\text{ , where } \ell_{\text{log}}(\hat{y}, y)=\left[-\log\left(\frac{e^{\hat{y}_y}}{\sum_{y'}e^{\hat{y}_{y'}}}\right)\right]
	%	\ell_{\text{logistic}}(\bW;\psi) = \E\limits_{x_1,y}\left[-\log\left(\frac{e^{(\bW\psi(x_1))_y}}{\sum_{y'}e^{\bW\psi(x_1))_{y'}}}\right)\right]
	\end{align*}
	The loss for representation $\psi:\cX_1\rightarrow\R^{d_1}$ and linear classifier $\bW\in\R^{k\times d_1}$ is denoted by $\ell_{\text{clf}}(\bW\psi)$.
\end{definition}
We note that the function $\ell_{\text{log}}$ is 1-Lipschitz in the first argument.
The result will also hold for the hinge loss $\ell_{\text{hinge}}(\hat{y}, y) = (1 - \hat{y}_y + \max_{y'\neq y}\hat{y}_{y'})_+$ which is also 1-Lipschitz, instead of $\ell_{\text{log}}$.

We assume that the optimal regressor $f^*_{\cH_1}$ for one-hot encoding also does well on linear classification.
\begin{assumption}\label{assmp:onehot_is_good}
	The best regressor for 1-hot encodings in $\cH_1$ does well on classification, i.e. $\ell_{\text{clf}}(\gamma f^*_{\cH_1})\le \epsilon_{\text{one-hot}}$ is small for some scalar $\gamma$.
\end{assumption}
\begin{remark}
	Note that if $\cH_1$ is universal, then $f^*_{\cH_1}(\bx_1) = \E[Y|X_1=\bx_1]$ and we know that $f^*_{\cH_1}$ is the Bayes-optimal predictor for binary classification.
	In general one can potentially predict the label by looking at $\argmax_{i\in[k]} f^*_{\cH_1}(\bx_1)_i$.
	The scalar $\gamma$ captures the margin in the predictor $f^*_{\cH_1}$.
\end{remark}
%The following assumption says that we can use this best regressor to also do well on linear classification (this holds in experiments).

We now show that using the classifier $\hat \bW$ obtained from linear regression on one-hot encoding with learned representations $\tilde{\psi}$ will also be good on linear classification.
The proof is in Section~\ref{sec:classification}
%Together with Theorem \ref{thm:main_result_approximate_CI}, we get that the logistic loss is also small for classification task:
\begin{theorem}
	\label{thm:main_classification_approximate_CI}
	For a fixed $\delta\in (0,1)$, under the same setting as Theorem \ref{thm:main_result_approximate_CI} and Assumption \ref{assmp:onehot_is_good}, we have:
	\begin{align*}
	\ell_{\text{clf}}\left(\gamma\hat \bW \tilde{\psi}\right) \le \tilde{\cO}\left(  \gamma \sqrt{\sigma^2\frac{d_2}{n_2} + \frac{\epsilon^2}{\beta^2} +\frac{\epsilon_{\pre}^2}{\beta^2 }}\right) + \epsilon_{\text{one-hot}},
	\end{align*}
	%	$$\ell_{\text{clf}}\left(\gamma\hat \bW \psi\right) \lesssim \sigma \frac{\sqrt{d_2+\log(d_2/\delta)}}{\beta \sqrt{n}} + \frac{\epsilon+\epsilon_{\pre}}{\beta^2} + \epsilon_{\text{one-hot}}, $$
	with probability $1-\delta$. 
\end{theorem}

\begin{proof}[Proof of Theorem~\ref{thm:main_classification_approximate_CI}]
	We simply follow the following sequence of steps
	\begin{align*}
		\ell_{\text{clf}}\left(\gamma \hat{\bW}\tilde{\psi}\right)
		&= \E [\ell_{\text{log}}\left(\gamma \hat{\bW}\tilde{\psi}(X_1), Y\right)]\\
		&\le^{(a)} \E \left[\ell_{\text{log}}\left(\gamma f^*_{\cH_1}(X_1), Y\right) + \gamma \|\hat{\bW}\tilde{\psi}(X_1) - f^*_{\cH_1}(X_1)\|\right]\\
		&\le^{(b)} \epsilon_{\text{one-hot}} + \gamma  \sqrt{\E\left[\|\hat{\bW}\tilde{\psi}(X_1) - f^*_{\cH_1}(X_1)\|^2\right]}\\
		&= \epsilon_{\text{one-hot}} + \gamma \sqrt{\ER_{\tilde{\psi}}[{\hat \bW}]}
	\end{align*}
	where $(a)$ follows because $\ell_{\text{log}}$ is 1-Lipschitz and $(b)$ follows from Assumption \ref{assmp:onehot_is_good} and Jensen's inequality.
	Plugging in Theorem \ref{thm:main_result_approximate_CI} completes the proof.
\end{proof}

\section{Four Different Ways to Use CI}
In this section we propose four different ways to use conditional independence to prove zero approximation error, i.e.,
\begin{claim}[informal]
When conditional independence is satisfied: $X_1\bot X_2|Y$, and some non-degeneracy is satisfied, there exists some matrix $\bW$ such that $\E[Y|X_1] = \bW \E[X_2|X_1]$.  	
\end{claim}

We note that for simplicity, most of the results are presented for the jointly Gaussian case, where everything could be captured by linear conditional expectation $\E^L[Y|X_1]$ or the covariance matrices. When generalizing the results for other random variables, we note just replace $X_1,X_2,Y$ by $\phi_1(X_1),\phi_2(X_2),\phi_y(Y)$ will suffice the same arguments. 

\subsection{Inverse Covariance Matrix} 
Write $\bSigma$ as the covariance matrix for the joint distribution $P_{X_1X_2Y}$. 
\begin{align*}
\bSigma =
\begin{bmatrix}
\bSigma_{XX} & \bSigma_{XY}\\
\bSigma^\top_{YY} & \bSigma_{YY}
\end{bmatrix},~~~
\bSigma^{-1} =
\begin{bmatrix}
\bA & \rho\\
\rho^\top & \bB
\end{bmatrix}
\end{align*}
where $\bA\in\Rr^{(d_1+d_2)\times (d_1+d_2)}, \rho\in\Rr^{(d_1+d_2)\times k}, \bB\in\Rr^{k\times k}$.
Furthermore
\begin{align*}
\rho = \begin{bmatrix}
\rho_1\\
\rho_2
\end{bmatrix};~~~
\bA = \begin{bmatrix}
\bA_{11} & \bA_{12}\\
\bA_{21} & \bA_{22}
\end{bmatrix}
\end{align*}
for $\rho_i\in\Rr^{d_i\times k},i=1,2$ and $\bA_{ij}\in\Rr^{d_i\times d_j}$ for $i,j\in\{1,2\}$.

\begin{claim}
	When conditional independence is satisfied, $\bA$ is block diagonal matrix, i.e., $\bA_{12}$ and $\bA_{21}$ are zero matrices.
\end{claim}
\begin{lemma}
	We have the following
	\begin{align}
	\E[X_1|X_2] &= (\bA_{11}-\brho_1\brho_1^\top)^{-1}(\brho_1\bar{\rho_2}^\top - \bA_{12})X_2\label{eq:1given2}\\
	\E[X_2|X_1] &= (\bA_{22}-\brho_2\brho_2^\top)^{-1}(\brho_2\bar{\rho_1}^\top - \bA_{21})X_1\label{eq:2given1}\\
	\E[Y|X] &= -B^{-\frac{1}{2}}(\brho_1^\top X_1 + \brho_2^\top X_2)\label{eq:ygiven12}
	\end{align}
	where $\brho_i=\rho_i \bB^{-\frac{1}{2}}$ for $i\in\{1,2\}$.
	Also,
	\begin{align*}
	(\bA_{11}-\brho_1\brho_1^\top)^{-1}\brho_1\brho_2^\top &= \frac{1}{1-\brho_1^\top \bA_{11}^{-1}\brho_1} \bA_{11}^{-1}\brho_1\brho_2^\top\\
	(\bA_{22}-\brho_2\brho_2^\top)^{-1}\brho_2\brho_1^\top &= \frac{1}{1-\brho_2^\top \bA_{22}^{-1}\brho_2} \bA_{22}^{-1}\brho_2\brho_1^\top
	\end{align*}
\end{lemma}
\begin{proof}
	We know that $\E[X_1|X_2]=\bSigma_{12}\bSigma_{22}^{-1}X_2$ and $\E[X_2|X_1]=\bSigma_{21}\bSigma_{11}^{-1}x_1$, where 
	\begin{align*}
	\bSigma_{XX} = \begin{bmatrix}
	\bSigma_{11} & \bSigma_{12}\\
	\bSigma_{21} & \bSigma_{22}\\
	\end{bmatrix}
	\end{align*}
	First using $\bSigma\bSigma^{-1}=I$, we get the following identities
	\begin{align}
	\bSigma_{XX}\bA+\bSigma_{XY}\rho^\top = \bI\label{eq:id1}\\
	\bSigma_{XY}^\top \bA+\bSigma_{YY}\rho^\top = 0\label{eq:id2}\\
	\bSigma_{XX} \rho+\bSigma_{XY}\bB = 0\label{eq:id3}\\
	\bSigma_{XY}^\top \rho+\bSigma_{YY}\bB = \bI\label{eq:id4}
	\end{align}
	From \Eqref{eq:id3} we get that $\bSigma_{XY}=-\bSigma_{XX}\rho \bB^{-1}$ and plugging this into \Eqref{eq:id1} we get
	\begin{align*}
	\bSigma_{XX}\bA-&\bSigma_{XX}\rho \bB^{-1}\rho^\top = \bI\\
	\implies\bSigma_{XX} &= (\bA-\rho \bB^{-1}\rho^\top)^{-1} = (\bA-\brho\brho^\top)^{-1}\\
	\implies\begin{bmatrix}
	\bSigma_{11} & \bSigma_{12}\\
	\bSigma_{21} & \bSigma_{22}\\
	\end{bmatrix} &= \left(\begin{bmatrix}
	\bA_{11}-\brho_1\brho_1^\top &
	\bA_{12}-\brho_1\brho_2^\top\\
	\bA_{21}-\brho_2\brho_1^\top &
	\bA_{22}-\brho_2\brho_2^\top\\
	\end{bmatrix}\right)^{-1}
	\end{align*}
	We now make use of the following expression for inverse of a matrix that uses Schur complement: $\bM/\alpha= \delta-\gamma \alpha^{-1} \beta$ is the Schur complement of $\alpha$ for $\bM$ defined below
	\begin{align*}
	\text{If } \bM = \begin{bmatrix}
	\alpha & \beta\\
	\gamma & \delta
	\end{bmatrix}, \text{ then, }
	\bM^{-1} = \begin{bmatrix}
	\alpha^{-1}+\alpha^{-1}\beta(\bM/\alpha)^{-1}\gamma\alpha^{-1} & -\alpha^{-1}\beta(\bM/\alpha)^{-1}\\
	-(\bM/\alpha)^{-1}\gamma\alpha^{-1} & (\bM/\alpha)^{-1}
	\end{bmatrix}
	\end{align*}
	For $\bM=(\bA-\brho\brho^\top)$, we have that $\bSigma_{XX}=\bM^{-1}$ and thus 
	\begin{align*}
	\bSigma_{12}\bSigma_{22}^{-1}
	&= -\alpha^{-1}\beta(\bM/\alpha)^{-1}((\bM/\alpha)^{-1})^{-1}\\
	&= -\alpha^{-1}\beta\\
	&= (\bA_{11}-\brho_1\brho_1^\top)^{-1}(\brho_1\brho_2^\top-\bA_{12})
	\end{align*}
	This proves \Eqref{eq:1given2} and similarly \Eqref{eq:2given1} can be proved.
	
	For \Eqref{eq:ygiven12}, we know that $\E[Y|X=(X_1,X_2)] = \bSigma_{YX}\bSigma_{XX}^{-1}X = \bSigma_{XY}^\top\bSigma_{XX}^{-1}X$.
	By using \Eqref{eq:id3} we get $\bSigma_{XY}=-\bSigma_{XX}\rho \bB^{-1}$ and thus
	\begin{align*}
	\E[Y|X=(X_1,X_2)]
	&= -\bB^{-1}\rho^\top\bSigma_{XX}\bSigma_{XX}^{-1}X\\
	&= -\bB^{-1}\rho^\top X = \bB^{-1}(\rho_1^\top X_1 + \rho_2^\top X_2)\\
	&= -\bB^{-\frac{1}{2}}(\brho_1^\top X_1 + \brho_2^\top X_2)
	\end{align*}
	For the second part, we will use the fact that $(\bI-\ba\bb^\top)^{-1} = \bI + \frac{1}{1-\ba^\top \bb}\ba\bb^\top$. Thus
	\begin{align*}
	(\bA_{11}-\brho_1\brho_1^\top)^{-1}\brho_1\brho_2
	&= (\bI-\bA_{11}^{-1}\brho_1\brho_1^\top)\bA_{11}^{-1}\brho_1\brho_2^\top\\
	&= (\bI+\frac{1}{1-\brho_1^\top \bA_{11}^{-1}\brho_1}\bA_{11}^{-1}\brho_1\brho_1)\bA_{11}^{-1}\brho_1\brho_2^\top\\
	&= \bA_{11}^{-1}(I+\frac{1}{1-\brho_1^\top \bA_{11}^{-1}\brho_1}\brho_1\brho_1\bA_{11}^{-1})\brho_1\brho_2^\top\\
	&= \bA_{11}^{-1}(\brho_1\brho_2^\top+\frac{\brho_1\bA_{11}^{-1}\brho_1}{1-\brho_1^\top \bA_{11}^{-1}\brho_1}\brho_1\brho_2^\top)\\
	&= \bA_{11}^{-1}\brho_1\brho_2^\top(1+\frac{\brho_1\bA_{11}^{-1}\brho_1}{1-\brho_1^\top \bA_{11}^{-1}\brho_1})\\
	&= \frac{1}{1-\brho_1^\top \bA_{11}^{-1}\brho_1}A_{11}^{-1}\brho_1\brho_2^\top
	\end{align*}
	The other statement can be proved similarly.
\end{proof}

\begin{claim}
\begin{align*}
\E[X_2|X_1] = (\bA_{22} -\bar\rho_2\bar\rho_2^\top )^{-1}\bar \rho_2\bar \rho_1^\top X_1. 
\E[Y|X_1] = -\bB^{-1/2}\bar \rho_1^\top X_1 - \bB^{-1/2}\brho_2^\top \E[X_2|X_1]
\end{align*}
Therefore $\E[Y|X_1]$ is in the same direction as $\E[X_2|X_1]$. 
\end{claim}

\subsection{Closed form of Linear Conditional Expectation} 
Refer to Claim \ref{claim:linear_conditional_expectation} and proof of Lemma \ref{lemma:linear_CI}. As this is the simplest proof we used in our paper.

\subsection{From Law of Iterated Expectation}

\begin{align*}
\E^L[X_2|X_1] = &\E^L[\E^L[ X_2|X_1,Y] |X_1 ]\\
= & \E\left[ [\bSigma_{X_2X_1} , \bSigma_{X_2 Y}  ]   \begin{bmatrix}
\bSigma_{X_1X_1} & \bSigma_{X_1Y} \\
\bSigma_{YX_1} & \bSigma_{YY} 
\end{bmatrix}^{-1}
\begin{bmatrix}
X_1 \\
Y 
\end{bmatrix}\mid X_1 \right] \\
= & \bA X_1 + \bB \E^L[Y|X_1].
\end{align*}
Using block matrix inverse,
\begin{align*}
\bA&=(\bSigma_{X_2X_1} -\bSigma_{X_2Y}\bSigma_{YY}^{-1}\bSigma_{YX_1} )   (\bSigma_{X_1X_1} - \bSigma_{X_1Y}\bSigma_{YY}^{-1}\bSigma_{YX_1} )^{-1} \in \mathbb{R}^{d_2 \times d_1}\\
&= \bSigma_{X_1X_2| Y}(\bSigma_{X_1X_1| Y})^{-1}\\
\bB &= \bSigma_{X_2 Y| X_1}(\bSigma_{YY| X_1})^{-1} \in \mathbb{R}^{d_2 \times \mathcal{Y}}.
\end{align*}

Therefore in general (without conditional independence assumption) our learned representation will be $\psi( x_1) = \bA x_1 + \bB f^* (x_1)$, where $f^*(\cdot):=\E^L[Y|X_1]$. 

It's easy to see that to learn $f^*$ from representation $\psi$, we need $A$ to have some good property, such as light tail in eigenspace, and $B$ needs to be full rank in its column space.

Notice in the case of conditional independence, $\bSigma_{X_1X_2| Y}=0$, and $A=0$. Therefore we could easily learn $f^*$ from $\psi$ if $X_2$ has enough information of $Y$ such that $\bSigma_{X_2 Y| X_1}$ is of the same rank as dimension of $Y$.

\subsection{From $\E[X_2|X_1,Y]=\E[X_2|Y]$}

\begin{proof}
	%This is direct result from the covariance matrix of $\phi_1(X_1),\phi_2(X_2),Y$. 
	Let the representation function $\psi$ be defined as follows, and let we use law of iterated expectation:
	\begin{align*}
	\psi(\cdot):= \E[X_2|X_1 ] = & \E[\E[X_2|X_1,Y]|X_1] \\
	= &\E[\E[X_2|Y]|X_1]\tag{uses CI}\\
	= & \sum_{y}  P(Y=y|X_1)\E[X_2|Y=y]\\
	=: & f(X_1)^\top A,
	\end{align*}
	where $f: \R^{d_1}\rightarrow \Delta_{\cY} $ satisfies $f(x_1)_y = P(Y=y|X_1=x_1)$, and $\bA \in \R^{\cY\times d_2} $ satisfies $\bA_{y,:} = \E[X_2|Y=y]$. Here $\Delta_d$ denotes simplex of dimension $d$, which represents the discrete probability density over support of size $d$. 
	
	Let $\bB=\bA^{\dagger}\in \R^{\cY\times d_2} $ be the pseudoinverse of matrix $\bA$, and we get $\bB\bA=\bI$ from our assumption that $A$ is of rank $|\cY|$. Therefore $f(\bx_1) = \bB \psi(\bx_1),\forall x_1$. Next we have:
	\begin{align*}
	\E[Y|X_1=\bx_1] = & \sum_{y}   P(Y=y|X_1=\bx_1)\times y \\
	= &  \hat{\bY} f(\bx_1)  \\
	= & (\hat \bY \bB)\cdot \psi(X_1).
	\end{align*}
	Here we denote by $\hat \bY\in \R^{k\times \cY}, \hat \bY_{:,y}=y$ that spans the whole support $\cY$.
	Therefore let $\bW^*=\hat \bY \bB$ will finish the proof. 
	
\end{proof}

% !TEX root = neurips21-ssl.tex

\newpage
\section{More on the experiments}
\label{sec:more_exp}
In this section, we include more experiment setup and results.

\begin{figure}[!tb]
   \centering     %%% not \center
   \hspace{-3pt}
   \subfigure{\includegraphics[width=0.46\textwidth]{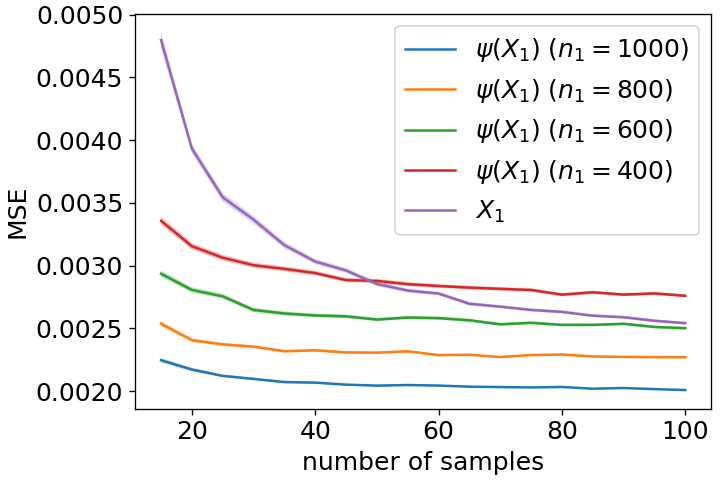}}
   \hspace{-5pt}
   \subfigure{\includegraphics[width=0.46\textwidth]{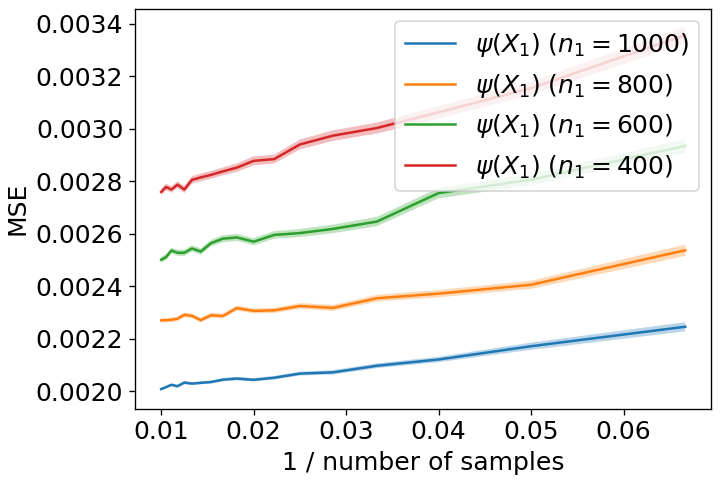}}
   \hspace{-10pt}
   \caption{\textbf{Left}: MSE of using $\psi$ to predict $Y$ versus using $X_1$ directly to predict $Y$. Using $\psi$ consistently outperforms using $X_1$. \textbf{Right}: MSE of $\psi$ learned with different $n_1$. The MSE scale with $1 / n_2$ as indicated by our analysis. Simulations are repeated 100 times, with the mean shown in solid line and one standard error shown in shadow.}
   \label{fig:simulation-appendix}
\end{figure}

\paragraph{Simulations. }

All the experiments are performed on a desktop computer with Intel i7-8700K, 16GB RAM.

Following Theorem \ref{thm:main_result_approximate_CI}, we know that the Excessive Risk (ER) is also controlled by (1) the number of samples for the pretext task ($n_1$), and (2) the number of samples for the downstream task ($n_2$), besides $k$ and $\epsilon_{CI}$ as discussed in the main text. In this simulation, we enforce strict conditional independence, and explore how ER varies with $n_1$ and $n_2$.
We generate the data the same way as in the main text, and keep $\alpha=0, k=2$, $d_1=50$ and $d_2=40$
We restrict the function class to linear model. 
Hence $\psi$ is the linear model to predict $X_2$ from $X_1$ given the pretext dataset.
We use Mean Squared Error (MSE) as the metric, since it is the empirical version of the ER. 
As shown in Figure \ref{fig:simulation-appendix}, $\psi$ consistently outperforms  $X_1$ in predicting $Y$ using a linear model learnt from the given downstream dataset, and ER does scale linearly with $1/n_2$, as indicated by our analysis.

\begin{figure}[!tb]
   \centering     %%% not \center
   \subfigure{\includegraphics[height=110pt]{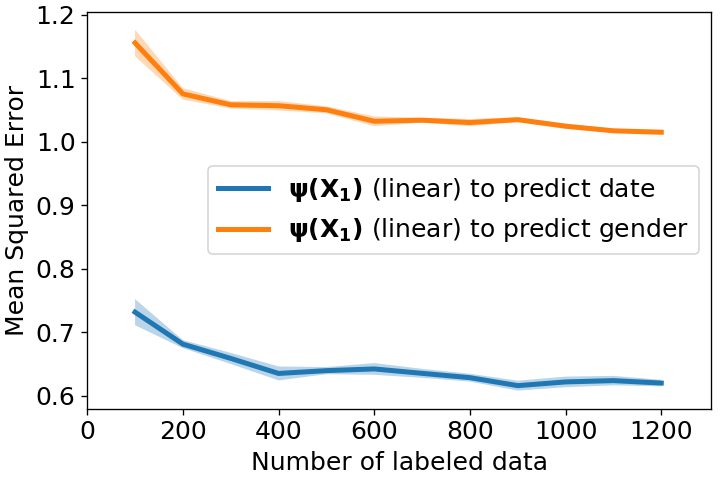}}
   \subfigure{\includegraphics[height=110pt]{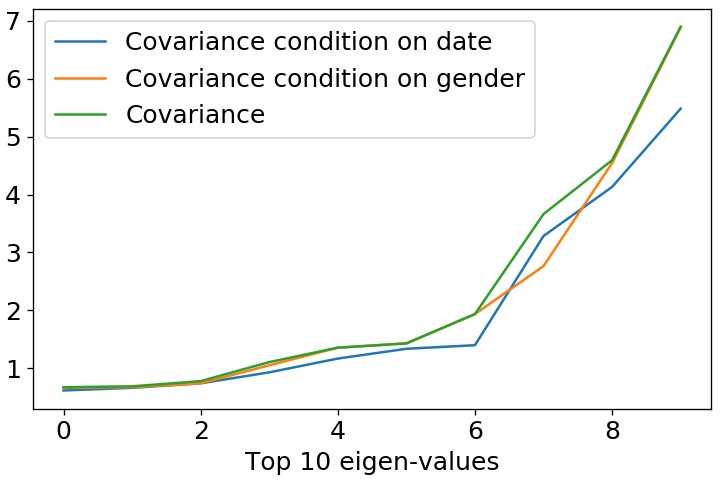}}
   \caption{\textbf{Left}: Mean Squared Error comparison of predicting gender and predicting date. \textbf{Right}: the spectrum comparison of covariance condition on gender and condition on date.}
   \label{fig:yb-gender-appendix}
\end{figure}

\paragraph{Computer Vision Task. }
For the context encoder part, we use all the recommended hyperparameter as in the provided source codes.
For the downstream resnet18 regression, we perform grid search over the hyperparameters to achieve best performance. Specifically, we set the batch size to be $24$, and traing the resnet18 for $50$ epoches. One pass of training (loops over all the settings with different number of labeled data) is finished within $6$ hours.
All the experiments are performed on a desktop computer with Intel i7-8700K, 16GB RAM, and NVIDIA Geforce 1080.
Training of the context encoder is finished within $12$ hours.
The yearbook dataset is distributed under BSD license.

Following the same procedure, we try to predict the gender $Y_G$. We normalize the label ($Y_G, Y_D$) to unit variance, and confine ourself to linear function class. That is, instead of using a context encoder to impaint $X_2$ from $X_1$, we confine $\psi$ to be a linear function. As shown on the left of Figure \ref{fig:yb-gender-appendix}, the MSE of predicting gender is higher than predicting dates. We find that $\|\bSigma_{\X_1\X_1}^{-1/2} \bSigma_{\X_1 X_2|Y_G } \|_F=9.32$, while $\|\bSigma_{\X_1\X_1}^{-1/2} \bSigma_{\X_1 X_2|Y_D } \|_F=8.15$. Moreover, as shown on the right of Figure \ref{fig:yb-gender-appendix}, conditioning on $Y_D$ cancels out more spectrum than conditioning on $Y_G$. In this case, we conjecture that, unlike $Y_D$, $Y_G$ does not capture much dependence between $X_1$ and $X_2$. And as a result, $\epsilon_{CI}$ is larger, and the downstream performance is worse, as we expected.
% $\|\bSigma_{\X_1\X_1}^{-1/2} \bSigma_{\X_1 X_2} \|_F=9.68$

%\begin{figure}
%\centering     %%% not \center
%\label{fig:exp-yearbook}
%\subfigure[]{\label{fig:yb-sample}\includegraphics[width=0.10\textwidth]{figures/yb_sample.png}}
%% \hspace{-1pt}
%\subfigure[]{\label{fig:yb-mse}\includegraphics[width=0.45\textwidth]{figures/yb_mse.png}}
%\hspace{-20pt}
%\subfigure[]{\label{fig:yb-mae}\includegraphics[width=0.45\textwidth]{figures/yb_l1.png}}
%\hspace{-10pt}
%\caption{(a) Example of the original images (first row),  the $X_2$ (content within the red block in the first row), the $X_1$ (content outside of the red block in the first row), the input to the inpainting task (the second row), the inpainting results ($\psi(X_1)$) (content within the red block the third row), and in this example $Y=1967$. (b) MSE (c) For better interpretation, we also show the result of mean absolute error. }
%\end{figure}

\paragraph{NLP Task. }

\begin{figure}[!tb]
   \centering     %%% not \center
   \subfigure{\includegraphics[width=0.40\textwidth]{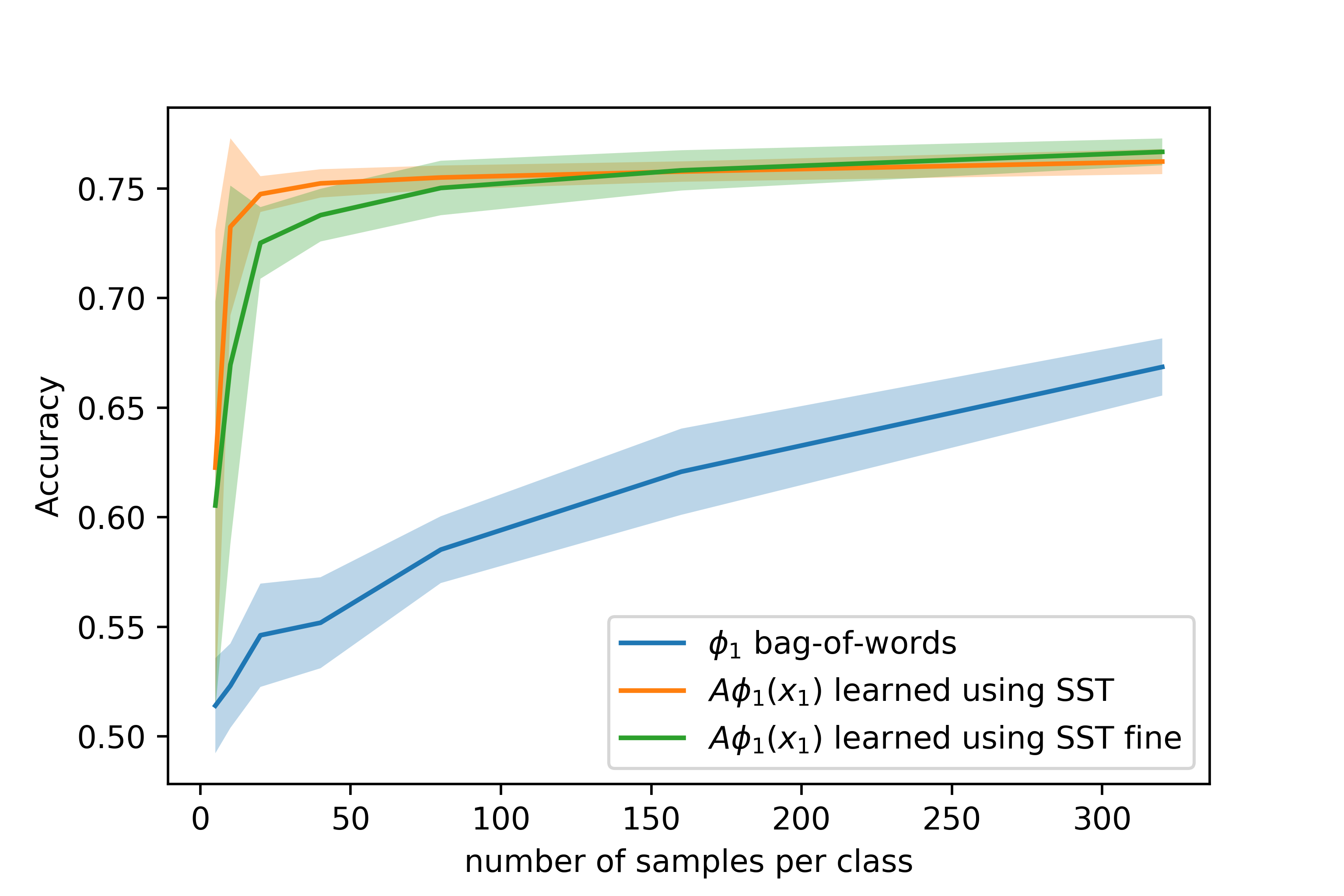}}
   \hspace{5pt}
   \subfigure{\includegraphics[width=0.40\textwidth]{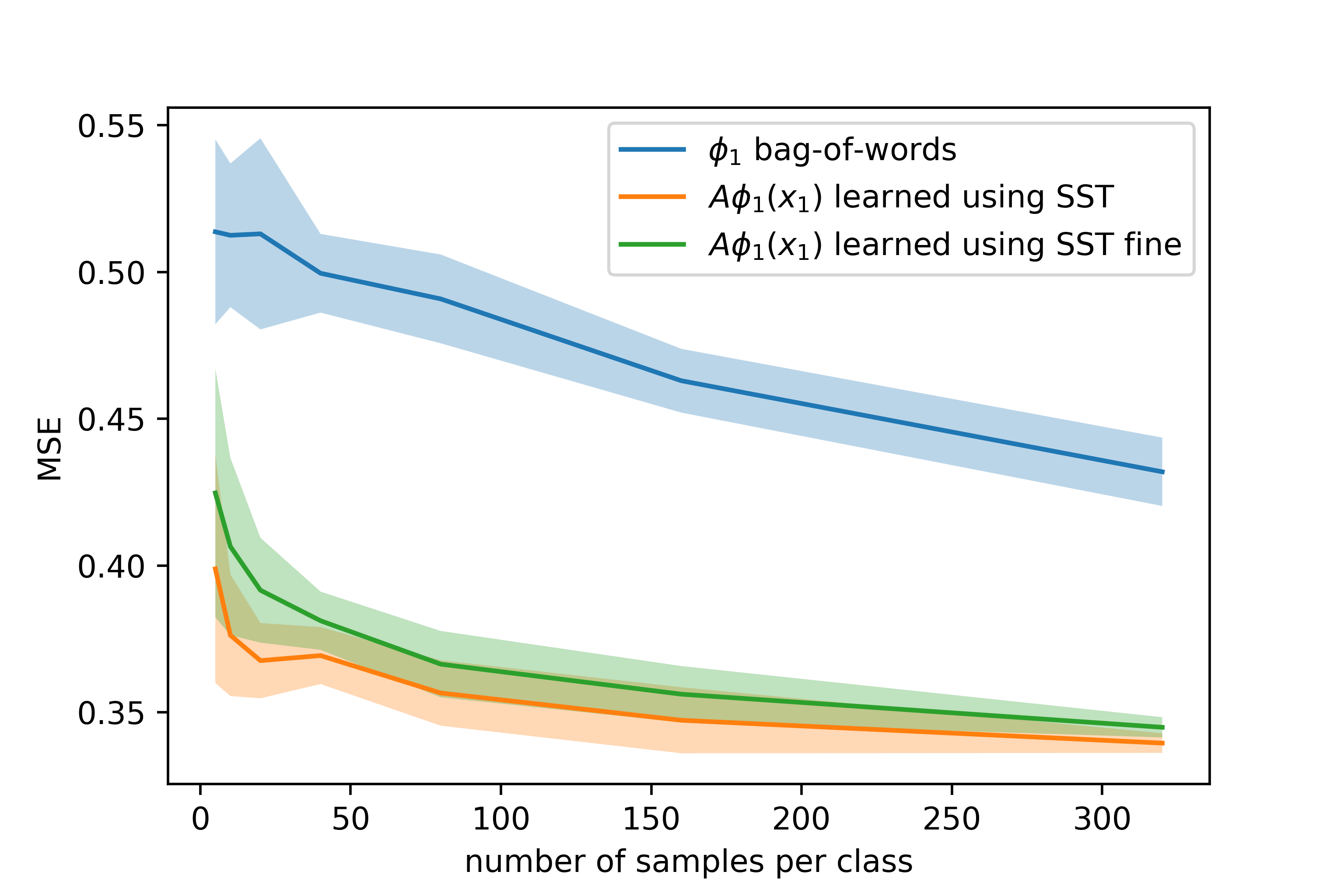}}
   \caption{Performance on SST of baseline $\phi_1(\bx_1)$, i.e. bag-of-words, and learned $\psi(\bx_1)$ for the two settings. \textbf{Left:} Classification accuracy, \textbf{Right:} Regression MSE.}
   \label{fig:exp-sst}
\end{figure}

We look at the setting where both $\cX_1$ and $\cX_2$ are the set of sentences and perform experiments by enforcing CI with and without latent variables.
The downstream task is sentiment classification with the Stanford Sentiment Treebank (SST) dataset \citep{socher2013recursive}, where inputs are movie reviews and the label set $\cY$ is $\{\pm1\}$.
We learn a linear representation $\psi(X_1) = \bB \phi(X_1)$ in the SSL phase as defined in Section~\ref{sec:beyondCI}.
Here we $X_1$, we pick $\phi(X_1)$ to be the bag-of-words representations of the movie review $X_1$, which has a vocabulary size of 13848
For $X_2$ we use a $d_2=300$ dimensional embedding of the sentence, that is the mean of word vectors (random Gaussians) for the words in the review $X_2$.
For SSL data we consider 2 settings, (a) enforce CI with the labels $\cY$, (b) enforce CI with extra latent variables, for which we use fine-grained version of SST with label set $\bar{\cY}=\{1,2,3,4,5\}$\footnote{Ratings $\{1,2\}$ correspond to $y=-1$ and $\{4,5\}$ correspond to $y=1$}..
In this setting, for every label $y\in\cY$ (or $\bar{y}\in\bar{\cY}$), we independently sample movie reviews $X_1$ and $X_2$ from the class $y$ (or $\bar{y}$), thus simulating the CI (or approximate CI) condition.
We test the learned $\psi$ on SST binary task with linear regression and linear classification; results are presented in Figure~\ref{fig:exp-sst}.
We observe that in both settings $\psi$ outperforms $\phi_1$, especially in the small-sample-size regime.
Exact CI is better than CI with latent variables, as suggested by theory.

The function $\psi$ (or equivalently matrix $\bB\in\R^{300\times 13848}$) is learnt by minimizing $\|X_2 - \bB \phi(X_1)\|^2$ averaged over the SSL train data with an $\|\cdot\|^2_F$ penalty on the matrix $\bB$. 
We use the scikit-learn RidgeRegressionCV\footnote{\url{https://scikit-learn.org/stable/modules/generated/sklearn.linear_model.RidgeCV.html}} solver for this with regularizer parameters in the list $[0.001, 0.1, 10, 1000]$.
Plotting Figure~\ref{fig:exp-sst} took less than an hour when using 8 Intel(R) Xeon(R) Silver 4214 CPUs on a cluster.

\end{document}